\documentclass{article}

\usepackage{microtype}
\usepackage{graphicx}
\usepackage{subfigure} 
\usepackage{booktabs} 
\usepackage{multicol}

\usepackage[pagebackref=true]{hyperref}

\usepackage[noend]{algorithmic}

\usepackage[accepted]{icml2023}

\usepackage{amsmath}
\usepackage{amssymb}
\usepackage{mathtools}
\usepackage{bbold}
\usepackage{amsthm}

 \newtheorem{innercustomprop}{Proposition}
\newenvironment{customprop}[1]
  {\renewcommand\theinnercustomprop{#1}\innercustomprop}
  {\endinnercustomprop}
  
 \newtheorem{innercustomassumption}{Assumption}
\newenvironment{customassumption}[1]
  {\renewcommand\theinnercustomassumption{#1}\innercustomassumption}
  {\endinnercustomassumption}

\DeclareMathOperator*{\argmin}{arg\,min}

\usepackage[capitalize,noabbrev]{cleveref}

\theoremstyle{plain}
\newtheorem{theorem}{Theorem}[section]
\newtheorem{proposition}[theorem]{Proposition}
\newtheorem{lemma}[theorem]{Lemma}

\theoremstyle{definition}

\newtheorem{assumption}[theorem]{Assumption}
\theoremstyle{remark}
\newtheorem*{remark}{Remark}

\usepackage[textsize=tiny]{todonotes}

\icmltitlerunning{Coin Sampling: Gradient-Based Bayesian Inference without Learning Rates}

\begin{document}

\twocolumn[
\icmltitle{Coin Sampling: Gradient-Based Bayesian Inference without Learning Rates}



\icmlsetsymbol{equal}{*}

\begin{icmlauthorlist}
\icmlauthor{Louis Sharrock}{yyy}
\icmlauthor{Christopher Nemeth}{yyy}
\end{icmlauthorlist}

\icmlaffiliation{yyy}{Department of Mathematics, Lancaster University, UK}

\icmlcorrespondingauthor{Louis Sharrock}{l.sharrock@lancaster.ac.uk}

\icmlkeywords{Machine Learning, ICML}

\vskip 0.3in
]



\printAffiliationsAndNotice{}  

\begin{abstract}
In recent years, particle-based variational inference (ParVI) methods such as Stein variational gradient descent (SVGD) have grown in popularity as scalable methods for Bayesian inference. Unfortunately, the properties of such methods invariably depend on hyperparameters such as the learning rate, which must be carefully tuned by the practitioner in order to ensure convergence to the target measure at a suitable rate.  In this paper, we introduce a suite of new particle-based methods for scalable Bayesian inference based on coin betting, which are entirely learning-rate free. We illustrate the performance of our approach on a range of numerical examples, including several high-dimensional models and datasets, demonstrating comparable performance to other ParVI algorithms with no need to tune a learning rate.
\end{abstract}

\section{Introduction}
\label{sec:intro}
The task of sampling from complex, high-dimensional probability distributions is of fundamental importance to Bayesian inference \cite{Robert2004,Gelman2013}, machine learning \cite{Neal1996,Andrieu2003a,Wilson2020}, molecular dynamics \cite{Krauth2006,Lelievre2016}, and scientific computing \cite{MacKay2003,Liu2009}. In this paper, we consider the canonical task of sampling from a target probability distribution $\pi(\mathrm{d}x)$ on $\mathbb{R}^d$ with density $\pi(x)$ with respect to the Lebesgue measure of the form\footnote{In a slight abuse of notation, we use $\pi$ to denote both the target distribution and its density.}
\begin{equation}
\pi(x) := \frac{e^{-U(x)}}{Z},
\end{equation}
where $U:\mathbb{R}^d\rightarrow\mathbb{R}$ is a continuously differentiable function known as the potential, and $Z = \int_{\mathbb{R}^d} e^{-U(x)}\mathrm{d}x$ is an unknown normalising constant.
 
 Recently, there has been growing interest in hybrid sampling methods which combine the non-parametric nature of Markov chain Monte Carlo (MCMC) sampling with the parametric approach used in variational inference (VI). In particular, {particle based variational inference} (ParVI) methods \cite{Liu2016a,Chen2018,Liu2019} approximate the target distribution using an ensemble of interacting particles, which are deterministically updated by iteratively minimising a metric such as the Kullback-Leibler (KL) divergence.
 
\begin{figure*}[t!]
\centering
\subfigure[Gaussian.]{\includegraphics[trim=0 16mm 0 16mm, clip, width=.32\textwidth]{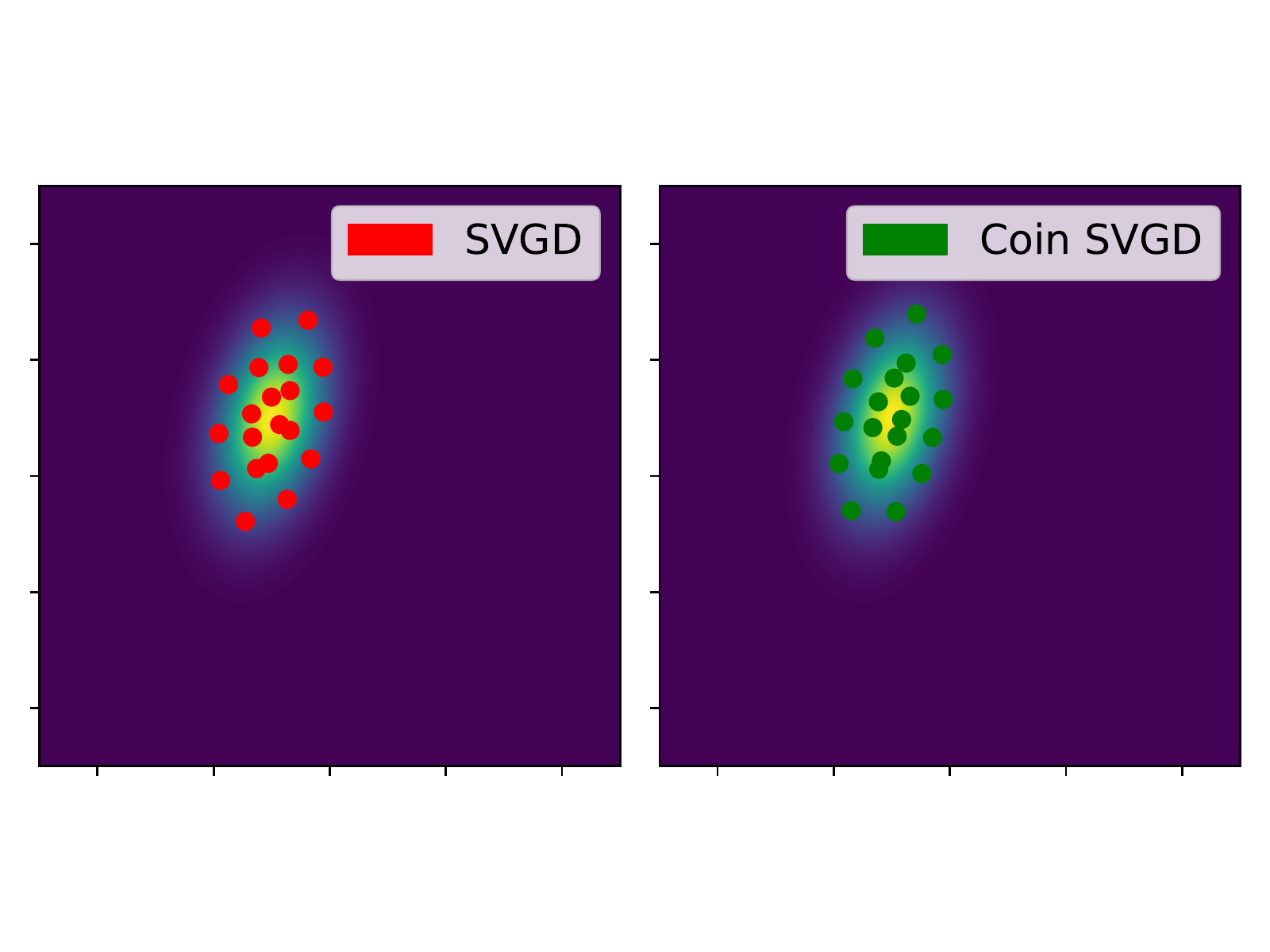}}
\subfigure[Mixture of Gaussians.]{\includegraphics[trim=0 16mm 0 16mm, clip, width=.32\textwidth]{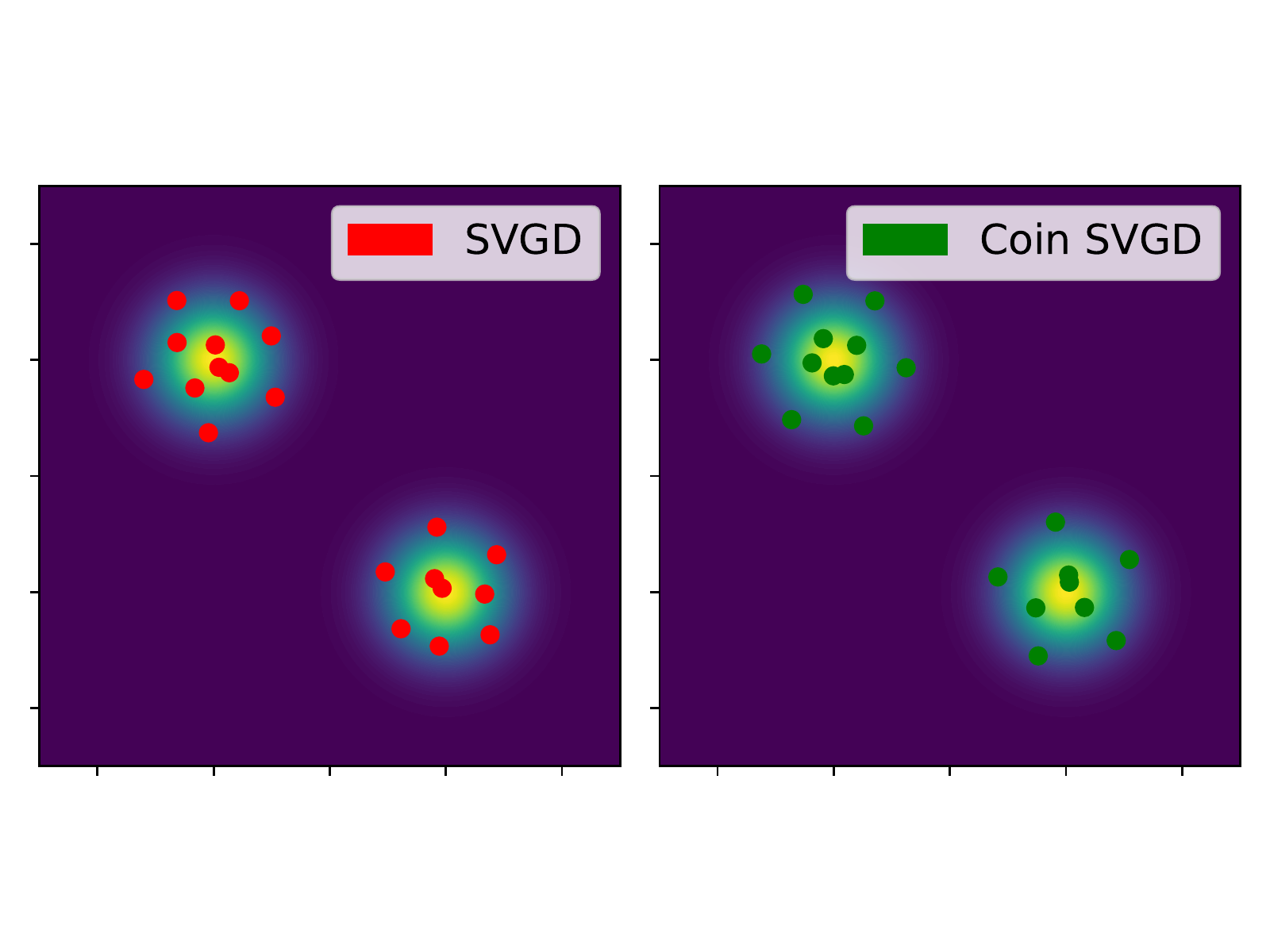}}
\subfigure[`Donut'.]{\includegraphics[trim=0 16mm 0 16mm, clip, width=.32\textwidth]{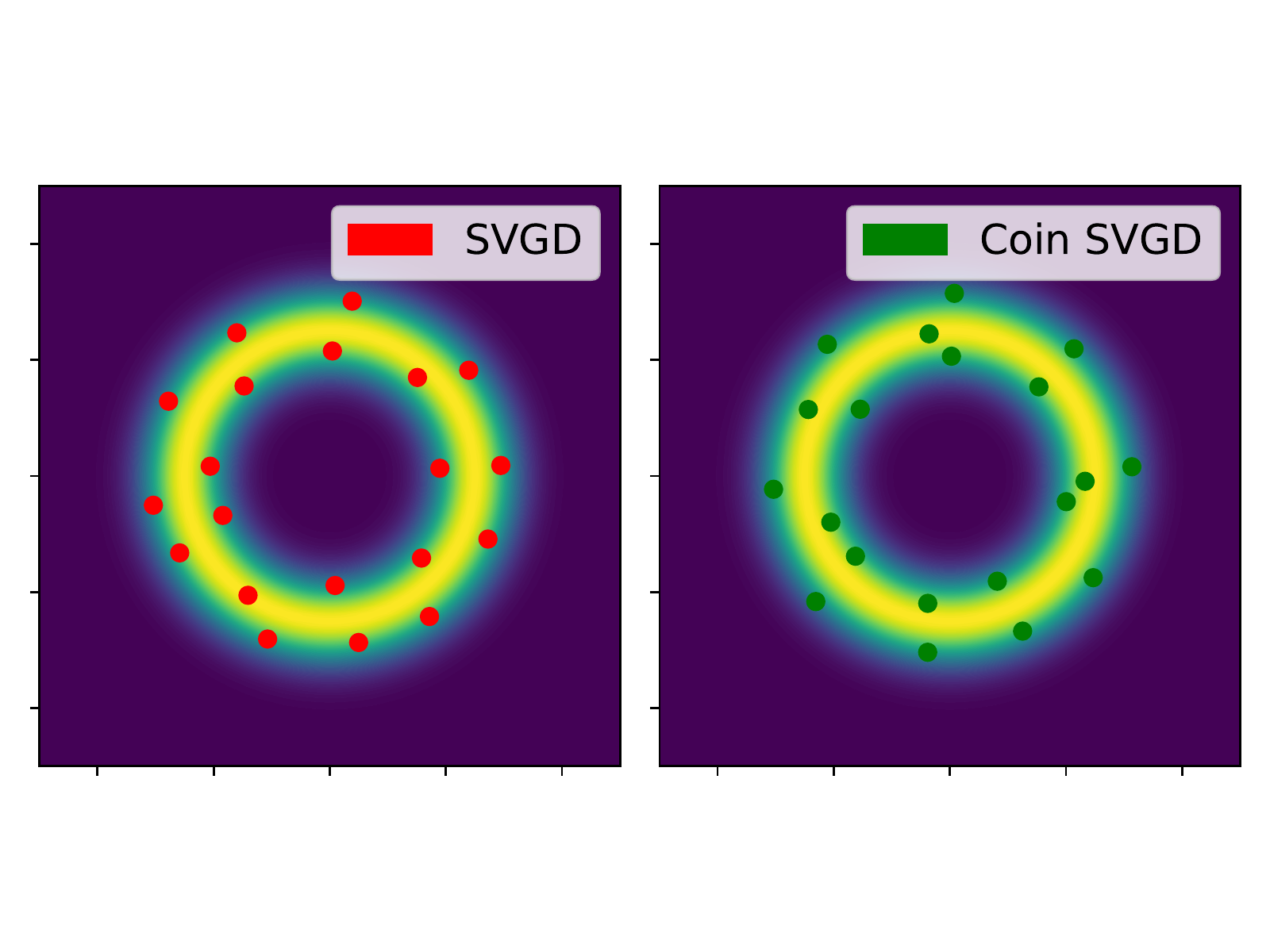}}
\subfigure[Rosenbrock Banana.]{\includegraphics[trim=0 16mm 0 16mm, clip, width=.32\textwidth]{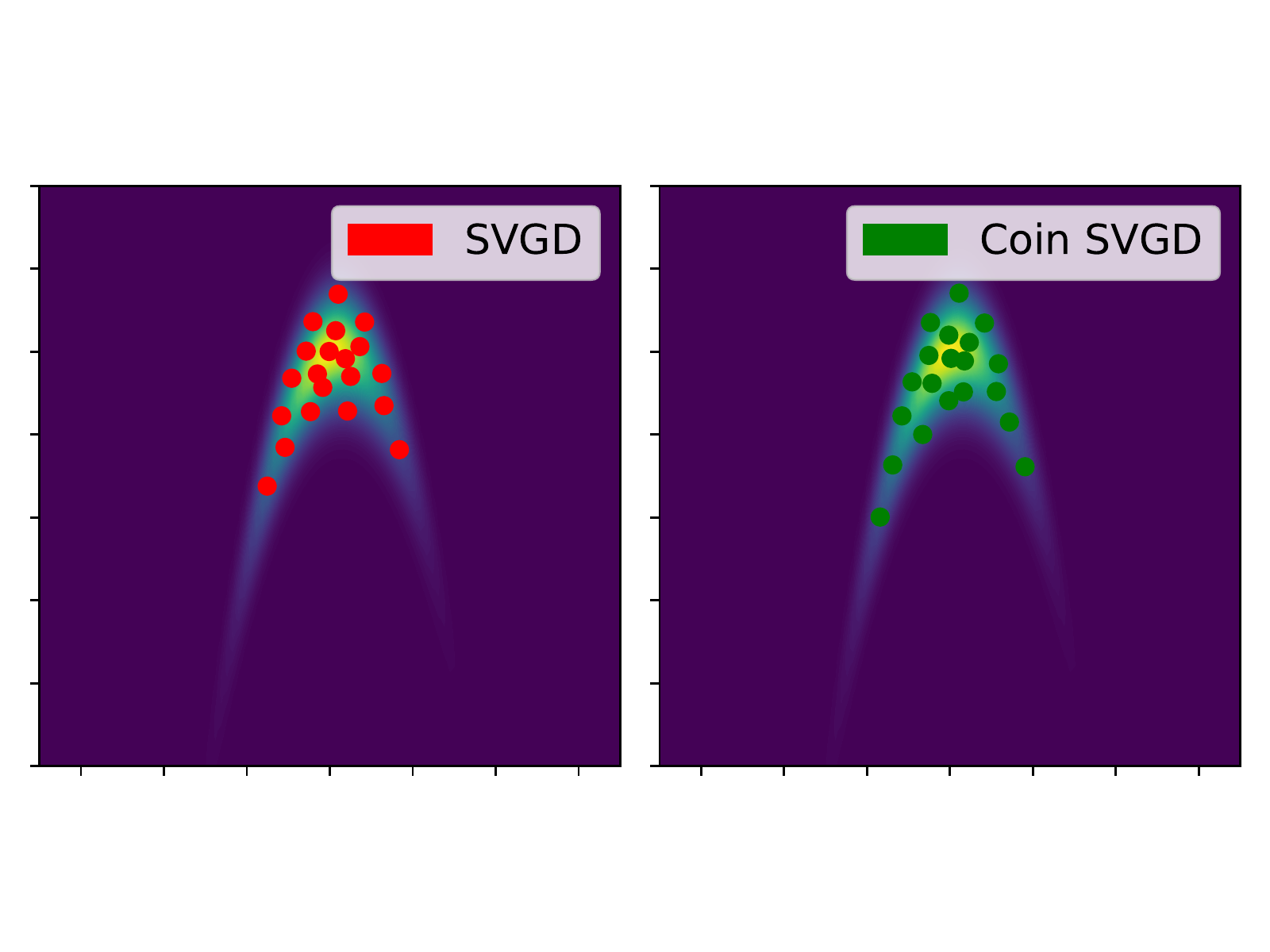}}
\subfigure[`Squiggle'.]{\includegraphics[trim=0 16mm 0 16mm, clip, width=.32\textwidth]{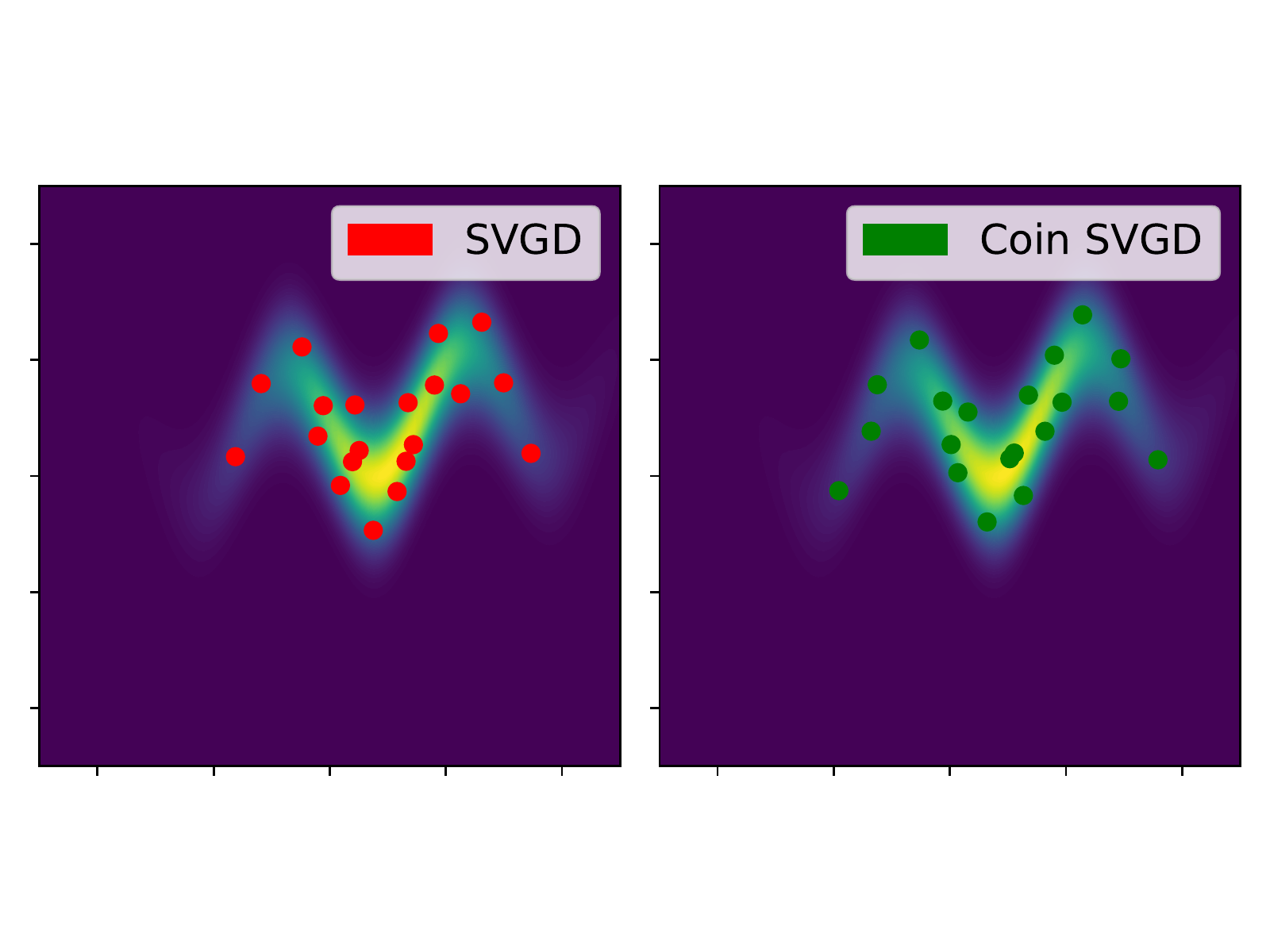}}
\subfigure[`Funnel'.]{\includegraphics[trim=0 16mm 0 16mm, clip, width=.32\textwidth]{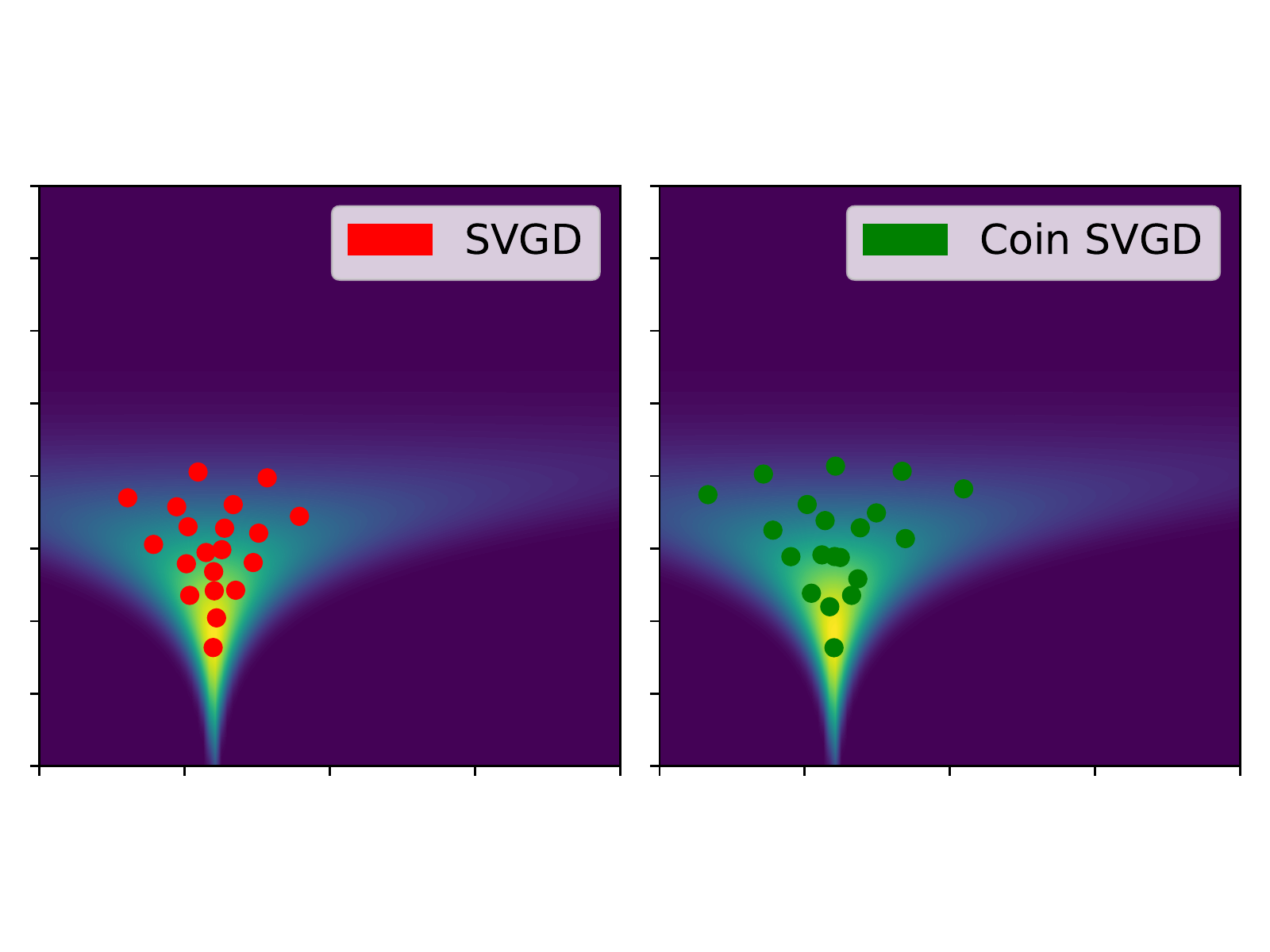}}
\caption{\textbf{A comparison between SVGD \cite{Liu2016a} and its learning-rate free analogue, Coin SVGD (Alg. \ref{alg:param_free_svgd})}. We plot the samples generated by both methods for several two-dimensional target distributions. Further details are provided in Sec. \ref{sec:numerics} and App. \ref{sec:toy-details}.}
\label{fig:figure1}
\vspace{-0.3cm}
\end{figure*}

 Perhaps the most well known of these methods is Stein variational gradient descent (SVGD) \cite{Liu2016a}, which iteratively updates the particles according to a form of gradient descent on the KL divergence, with the descent direction restricted to belong to a unit ball in a reproducing kernel Hilbert space (RKHS). This approach has since given rise to several variants \cite{Liu2017,Han2018a,Liu2018a,Zhuo2018,Chen2018a,Detommaso2018,Futami2019,Futami2019a,Wang2019a,Chen2020,Ye2020,Liu2022,Sun2022}; and found success in a range of problems, including uncertainty quantification \cite{Zhu2018}, reinforcement learning \cite{Haarnoja2017,Liu2017a,Zhang2018b}, learning deep probabilistic models \cite{Pu2017,Wang2017a}, and Bayesian meta-learning \cite{Feng2017,Yoon2018}.

In order to construct and analyse sampling algorithms of this type, one popular approach is to reformulate the sampling problem as an optimisation problem in the space of measures \cite{Jordan1998,Wibisono2018,Cheng2018a,Durmus2019}. In this setting, one views the target $\pi$ as the solution of an optimisation problem 
\begin{equation}
 \pi = \argmin_{\mu\in\mathcal{P}_2(\mathbb{R}^d)} \mathcal{F}(\mu),
\end{equation} 
where $\mathcal{P}_2(\mathbb{R}^d)$ denotes the set of probability measures $\{\mu:\int_{\mathbb{R}^d}||x||^2\mu(\mathrm{d}x)<\infty\}$, and $\mathcal{F}:\mathcal{P}(\mathbb{R}^d)\rightarrow\mathbb{R}$ is a functional which is uniquely minimised at $\pi$. A general strategy for solving this problem is then to simulate a time-discretisation of 
the gradient flow of $\mathcal{F}$ over $\mathcal{P}_2(\mathbb{R}^d)$, having equipped this space with a suitable metric \cite{Ambrosio2008}. 

Many popular sampling algorithms can be understood from this perspective. For example, Langevin Monte Carlo (LMC), a popular MCMC algorithm, corresponds to the so-called forward-flow discretisation of the gradient flow of the KL divergence with respect to the quadratic Wasserstein metric \cite{Wibisono2018,Durmus2019}.\footnote{We note that the connection between the law of the overdamped Langevin diffusion (i.e., the continuous-time dynamics of LMC) and the gradient flow of the KL divergence dates back to Otto et al. \cite{Jordan1998,Otto2001,Otto2005}.} 
Meanwhile, SVGD can be viewed as the explicit Euler discretisation of the gradient flow of the KL divergence with respect to a kernelised Wasserstein metric \cite{Liu2017,Duncan2019}. Other more recent examples, designed with this perspective in mind, include maximum mean discrepancy (MMD) gradient descent \cite{Arbel2019}, the Wasserstein proximal gradient algorithm \cite{Salim2020}, 
kernel Stein discrepancy descent (KSDD) \cite{Korba2021}, Laplacian adjusted Wasserstein gradient descent (LAWGD) \cite{Chewi2020}, mollified interaction energy descent (MIED) \cite{Li2023}, and the various other ParVI methods described in \citet{Chen2018,Liu2019,Liu2019a}.

One feature common to all of these approaches is the need to specify an appropriate learning rate (i.e., step size) $\gamma$, or a learning rate schedule $(\gamma_t)_{t\geq 1}$. This learning rate must be sufficiently small to ensure convergence to the target measure, or a close approximation thereof, but also large enough to ensure convergence within a reasonable time period. In theory, for a given target $\pi$, existing convergence rates allow one to derive an optimal learning rate (see, e.g., \citealp{Korba2020,Salim2022,Sun2022} for SVGD; \citealp{Dalalyan2017,Dalalyan2017a,Durmus2017,Dalalyan2019,Durmus2019a} for LMC). Invariably, however, the optimal learning rate is a function of the unknown target measure (e.g., Corollary 6 in \citealp{Korba2020};  Theorem 9 in \citealp{Durmus2019}) and thus, in practice, cannot be computed. 

With these considerations in mind, a natural question is whether one can obtain a gradient-based sampling method which does not require a learning rate. In this paper, we answer this question in the affirmative. In particular, inspired by the parameter-free optimisation methods developed by Orabona and coworkers \cite{Orabona2014,Orabona2016,Orabona2017,Cutkosky2018,Jun2019,Chen2022}, and leveraging the view of sampling as an optimisation problem in the space of measures \cite{Wibisono2018}, we obtain a new suite of particle-based algorithms for scalable Bayesian inference which are entirely learning rate free. Similar to other ParVIs, our algorithms deterministically update an ensemble of interacting particles in order to approximate the target distribution. However, unlike other ParVIs, our algorithms do not correspond to the time-discretisation of any gradient flow, and thus bear little resemblance to existing methods. 

Under the assumption of log-concavity, we outline how to establish convergence to the target measure in the infinite-particle regime and obtain a non-asymptotic convergence rate. We then illustrate the performance of our approach on a range of numerical examples, including both convex and non-convex targets. Our results indicate that the proposed methodology achieves comparable performance to existing particle-based sampling algorithms in a range of tasks, with no need to tune a learning rate.

\section{Preliminaries}
\label{sec:prelims}

\subsection{Optimisation in Euclidean Space}
We begin by reviewing optimisation in Euclidean spaces, focusing on the learning-rate free stochastic optimisation method introduced by \citet{Orabona2016}. This will later provide the foundation for our learning-rate free sampling algorithms. 

\subsubsection{Notation}
Let $\mathcal{X}\subseteq\mathbb{R}^d$, and write $||\cdot||$ and $\langle \cdot,\cdot\rangle$ for the Euclidean norm and inner product in $\mathbb{R}^d$. Let $f:\mathcal{X}\rightarrow\mathbb{R}\cup\{-\infty,\infty\}$, and let $f^{*}:\mathcal{X}^{*}\rightarrow\mathbb{R}\cup\{-\infty,\infty\}$ denote the Fenchel conjugate of $f$, so that $f^{*}(u) = \sup_{x\in\mathcal{X}}\left[\langle u,x\rangle - f(x)\right]$.

Suppose that $f$ is $m$-strongly convex, for some $m\geq 0$. Let $x\in\mathcal{X}$. We say that $g\in\mathcal{X}$ is a subgradient of $f$ at $x$, and write $g\in\partial f(x)$ if, for any $z\in\mathcal{X}$, 
\begin{equation}
f(z) - f(x) \geq \langle g,z-x\rangle + \frac{m}{2}||z-x||^2.
\end{equation}
If $f$ is differentiable at $x$, then the differential set $\partial f(x)$ contains a single element, $\partial f(x)= \{\nabla f(x)\}$, where $\nabla f(x)$ denotes the gradient of $f$ at $x$. 

\subsubsection{Euclidean Gradient Flows}
We begin by considering the optimisation problem
\begin{equation}
x^{*} = \argmin_{x\in\mathcal{X}} f(x), \label{eq:minimisation}
\end{equation}
where $f:\mathcal{X}\rightarrow\mathbb{R}$ is $m$-strongly convex. We can solve this problem using the gradient flow of $f$, defined as the solution $x:[0,\infty)\rightarrow\mathbb{R}^d$ of the following differential inclusion
\begin{equation}
    \dot{x}_t \in -\partial f(x_t), \label{eq:diff_eq}
\end{equation}
initialised at $x_0\in\mathcal{X}$. This inclusion admits a unique, absolutely continuous solution for almost all $t\geq 0$ (e.g., Theorem 3.1 in \citealp{Brezis1973}, Theorem 2.7 in \citealp{Peypouquet2010}; Proposition 2.1 in \citealp{Santambrogio2017}).
Moreover, the function $t\mapsto f(x_t)$ is decreasing, with $\lim_{t\rightarrow\infty} f(x_t) = \inf_{x\in\mathcal{X}}f(x)$ \citep[Proposition 3.1]{Peypouquet2010}.

In practice, it is necessary to use a time-discretisation of this gradient flow. One standard choice is a backward Euler discretisation, which results in the proximal point algorithm \cite{Guler1991,DeGiorgi1993}. Alternatively, one can utilise a forward Euler discretisation, which results in the standard subgradient descent algorithm \cite{Shor1985}
\begin{equation}
    x_{t+1} = x_{t} - \gamma g_t~,~~~g_t\in\partial f(x_t). \label{eq:GD}
\end{equation}
The properties of this algorithm depend, necessarily, on the choice of learning rate $\gamma>0$. For example, given an $L$-Lipschitz function, it is well known that the average of the algorithm iterates $\bar{x}_T = \frac{1}{T}\sum_{t=1}^{T} x_t$ satisfies \citep[e.g.,][]{Zinkevich2003} 
\begin{equation}
    f\left(\bar{x}_T\right) - f(x^{*}) \leq \frac{1}{T}\left[ \frac{||x_1-x^{*}||^2}{2\gamma} + \frac{L^2T\gamma}{2}\right]. \label{eq:sub_GD_bound} 
\end{equation}
Using this expression, one can obtain the `ideal' learning rate as $\gamma_{\text{ideal}} = \frac{||x_1-x^{*}||}{L\sqrt{T}}$, which implies the optimal error bound
\begin{equation}
    f(\bar{x}_T) - f(x^{*}) \leq \frac{L||x_1-x^{*}||}{\sqrt{T}}. \label{eq:optimal_bound}
\end{equation}
In practice, however, it is not possible to achieve this bound. Indeed, even in hindsight, one cannot compute the ideal learning rate $\gamma_{\text{ideal}}$, since it depends on the unknown $||x_1-x^{*}||$.

\subsubsection{Learning-Rate Free Gradient Descent}
\label{sec:param_free_grad_descent}
Following \citet{Orabona2016}, we now outline an alternative approach for solving the stochastic optimisation problem in \eqref{eq:minimisation} which is entirely learning-rate free. Consider a gambler who bets on the outcomes of a series of adversarial coin flips. Suppose that the gambler starts with an initial wealth $w_0 = \varepsilon>0$. In the $t^{\text{th}}$ round, the gambler bets on the outcome of a coin flip $c_t\in\{-1,1\}$, where $+1$ denotes heads and $-1$ denotes tails. For now, we make no assumptions on how $c_t$ is generated.

We will encode the gambler's bet in the $t^{\text{th}}$ round by $x_t\in\mathbb{R}$. In particular, $\mathrm{sign}(x_t)\in\{-1,1\}$ will denote whether the bet is on heads or tails, and $|x_t|\in\mathbb{R}$ will denote the size of the bet. Thus, in the $t^{\text{th}}$ round, the gambler wins $x_tc_t$ if $\mathrm{sign}(c_t) = \mathrm{sign}(x_t)$; and loses $x_tc_t$ otherwise. Finally, we will write $w_t$ for the wealth of the gambler at the end of the $t^{\text{th}}$ round. Clearly, we then have that
\begin{equation}
    w_t = \varepsilon + \sum_{i=1}^t c_ix_i.
\end{equation}
We will restrict our attention to the case in which the gambler's bets satisfy $x_t = \beta_tw_{t-1}$, for some betting fraction $\beta_t\in[-1,1]$. This is equivalent to the assumption that the gambler cannot borrow any money.

We will now outline how to solve the convex optimisation problem $x^{*}=\argmin_{x\in\mathbb{R}}f(x)$ using a coin-betting algorithm. For simplicity, we will restrict our attention to the simple one-dimensional function $f(x) = |x-10|$. We note, however, that this approach can easily be extended to any convex function $f:\mathbb{R}^d\rightarrow\mathbb{R}$ \cite{Orabona2016}. 
Suppose we define the outcome of a coin flip $\smash{c_t\in\{-1,1\}}$ to be equal to $\smash{-g_t \in -\partial [f(x_t)]}$, the negative subgradient of $\smash{f(x_t)}$. In this case, under a certain assumption on the betting strategy $(\beta_t)_{t=1}^T$, \citet{Orabona2016} show that the average of bets $f(\bar{x}_T)$ converges to $f(x^{*})$, with a rate which depends on the quality of the betting strategy.
\begin{lemma}
Suppose that the betting strategy $\smash{(\beta_t)_{t=1}^T}$ guarantees that, for any sequence of coin flips $\smash{(c_t)_{t=1}^T} \in\{-1,1\}$, there exists a function $\smash{h:\mathbb{R}\rightarrow\mathbb{R}}$ such that the wealth after $T$ rounds satisfies $\smash{w_T \geq h(\sum_{t=1}^T c_t)}$. Then
\label{orabona_lemma}
\begin{align}
    f\left(\frac{1}{T}\sum_{t=1}^T x_t\right) - f(x^{*})  \leq \frac{h^{*}(x^{*}) + \varepsilon}{T}.
\end{align}
\end{lemma}
\begin{proof}
See App. \ref{sec:theoretica_results}.
\end{proof}
We can thus use any suitable coin-betting algorithm to obtain $x^{*} = \argmin_{x\in\mathbb{R}}f(x)$, given access to the subgradients of $f$. Moreover, any such algorithm will be entirely learning-rate free. There are various betting strategies which satisfy the requirement that $\smash{w_T \geq h(\sum_{t=1}^T c_t)}$ \citep[e.g.,][]{Orabona2016,Orabona2017,Chen2022}. Perhaps the simplest such strategy is one based on the Krichevsky-Trofimov (KT) estimator \cite{Krichevsky1981}, which defines the betting strategy to be equal to $\beta_t = \sum_{i=1}^{t-1}c_i/t$. This results in the coin betting algorithm 
\begin{equation}
    x_{t} = -\frac{\sum_{i=1}^{t-1}g_i}{t}\left(\varepsilon - \sum_{i=1}^{t-1} g_i x_i \right). \label{eq:KT_algorithm}
\end{equation}
In this case, it is possible to show \citep[][Lemma 14]{Orabona2016} that the wealth is lower bounded by
\begin{equation}
    h\left(\sum_{t=1}^T c_t\right) = \frac{\varepsilon}{K\sqrt{T}} \exp\bigg( \frac{\left(\sum_{t=1}^T c_t\right)^2}{2T}\bigg),
\end{equation}
where $K$ is a universal constant. Thus, using Lemma \ref{orabona_lemma} and an appropriate bound on the convex conjugate of $h$, one obtains \citep[][Corollary 5]{Orabona2016}
\begin{equation}
    f(\bar{x}_T) - f(x^{*}) \leq K \frac{||x^{*}|| \sqrt{\log (1 + \frac{24T^2||x^{*}||^2}{\varepsilon^2})} + \varepsilon}{\sqrt{T}}. \label{eq:KT_bound}
\end{equation}
It is instructive to compare this bound with \eqref{eq:optimal_bound}, the corresponding bound for subgradient descent with an optimally chosen learning rate. Although the coin-betting approach does not quite achieve the optimal bound in \eqref{eq:optimal_bound}, it comes close, containing only an additional log-factor. This can be viewed as the trade-off for the fact that the algorithm is now learning-rate free.

\section{Coin Sampling for Bayesian Inference}
\label{sec:svgd}
Our approach, summarised in Alg. \ref{alg:param_free_grad_descent_v2}, can be viewed as a natural extension of the learning-rate free optimisation methods introduced in Sec. \ref{sec:param_free_grad_descent} to the Wasserstein space. In particular, coin sampling replaces Euclidean gradients with Wasserstein gradients in the coin-betting framework, and can thus be used to solve optimisation problems on the space of probability measures, i.e., for Bayesian inference.

\subsection{Optimisation in Wasserstein Space}
\label{sec:opt_wasserstein_space}
To extend coin betting to our setting, we will require some basic concepts from optimal transport, including the definition of the Wasserstein space, and of a Wasserstein gradient flow. We provide additional details on geodesic convexity and subdifferential calculus in App. \ref{sec:wasserstein_details}; see also the books of \citet{Ambrosio2008} and \citet{Villani2008}.

\subsubsection{The Wasserstein Space}
\label{sec:wasserstein-space}
Let $\mathcal{P}_2(\mathbb{R}^d)$ denote the set of probability measures on $\mathbb{R}^d$ with finite $2^{\text{nd}}$ moment: $\int_{\mathbb{R}^d} ||x||^2 \mu(\mathrm{d}x)<\infty$. For any $\mu\in\mathcal{P}_{2}(\mathbb{R}^d)$, let $L^{2}(\mu)$ denote the set of measurable functions $f:\mathbb{R}^d\rightarrow\mathbb{R}^d$ such that $\int_{\mathbb{R}^d}||f(x)||^2\mu(\mathrm{d}x)<\infty$. We will write $||\cdot||^2_{L^2(\mu)}$ and $\langle\cdot,\cdot\rangle_{L^2(\mu)}$ to denote, respectively, the norm and the inner product of this space.

Given a probability measure $\mu\in\mathcal{P}_2(\mathbb{R}^d)$ and a measurable function $T:\mathbb{R}^d\rightarrow\mathbb{R}^d$, we write $T_{\#}\mu$ for the pushforward measure of $\mu$ under $T$, that is, the measure such that $T_{\#}\mu(B)=\mu(T^{-1}(B))$ for all Borel measurable $B\in\mathcal{B}(\mathbb{R}^d)$. For every $\mu,\nu\in\mathcal{P}_2(\mathbb{R}^d)$, let $\Gamma(\mu,\nu)$ be the set of couplings (or transport plans) between $\mu$ and $\nu$, defined as  
$\Gamma(\mu,\nu) = \{\gamma\in\mathcal{P}_2(\mathbb{R}^d): Q^{1}_{\#}\gamma = \mu, Q^{2}_{\#}\gamma = \nu \}$, 
where $Q^{1}$ and $Q^{2}$ denote the projections onto the first and second components of $\mathbb{R}^d\times\mathbb{R}^d$. The Wasserstein $2$-distance between $\mu$ and $\nu$ is then defined according to
\begin{equation}
W_2^2(\mu,\nu) = \inf_{\gamma \in\Gamma(\mu,\nu)}  \int_{\mathbb{R}^d\times\mathbb{R}^d}||x-y||^2 \gamma(\mathrm{d}x,\mathrm{d}y). \label{eq:wasserstein}
\end{equation}
The Wasserstein distance $W_2$ is a distance over $\mathcal{P}_2(\mathbb{R}^d)$. Thus $(\mathcal{P}_2(\mathbb{R}^d), W_2)$ is a metric space of probability measures, known as the Wasserstein space.

\subsubsection{Wasserstein Gradient Flows}
Recall the optimisation problem from Sec. \ref{sec:intro},
\begin{equation}
    \pi = \argmin_{\mu\in\mathcal{P}_2(\mathbb{R}^d)}\mathcal{F}(\mu), \label{eq:wasserstein_optimisation}
\end{equation}
where $\mathcal{F}:\mathcal{P}_2(\mathbb{R}^d)\rightarrow(-\infty,\infty]$ is a proper, lower semi-continuous 
functional uniquely minimised at $\pi$. There are various possible choices for the dissimilarity functional $\mathcal{F}$ \citep[see, e.g., ][]{Simon-Gabriel2018}. In the context of Bayesian inference, perhaps the most common choice is $\mathrm{KL}(\mu|\pi)$, the Kullback-Leibler (KL) divergence of $\mu$ with respect to $\pi$. Other possibilities include the chi-squared divergence $\mathcal{X}^2(\mu|\pi)$ \cite{Chewi2020}, and the maximum mean discrepancy $\mathrm{MMD}(\mu|\pi)$ \cite{Arbel2019}, of which the kernel Stein discrepancy $\mathrm{KSD}(\mu|\pi)$ \cite{Korba2021} is a special case.

Similarly to the Euclidean case, typical solutions to \eqref{eq:wasserstein_optimisation} are based on the use of a gradient flow. In particular, one can now consider the {Wasserstein gradient flow} of $\mathcal{F}$, defined as the weak solution $\smash{\mu:[0,\infty)\rightarrow\mathcal{P}_2(\mathbb{R}^d)}$ of the continuity equation \citep[][Chapter 11]{Ambrosio2008}
\begin{equation}
    \partial_{t}\mu_t  +\nabla \cdot \left(v_t\mu_t\right)=0~,~~~v_t \in -\partial \mathcal{F}(\mu_t), \label{eq:wasserstein_gf}
\end{equation}
where $\partial\mathcal{F}(\mu)$ denotes the Fr\'{e}chet subdifferential $\partial\mathcal{F}(\mu)$ of $\mathcal{F}$ at $\mu$ (see App. \ref{sec:wasserstein_details} for a precise definition).
Under mild conditions, this equation admits a unique solution for any initial condition (e.g., Theorem 11.1.4 and Theorem 11.2.1 in \citealp{Ambrosio2008}; Proposition 4.13 in \citealp{Santambrogio2017}). In addition, the function $t\mapsto \mathcal{F}(\mu_t)$ is decreasing, so that $\lim_{t\rightarrow\infty} \mathcal{F}(\mu_t) = \inf_{\mu\in\mathcal{P}_2(\mathbb{R}^d)}\mathcal{F}(\mu)$ \citep[][Chapter 11]{Ambrosio2008}.

\subsubsection{Discretised Wasserstein Gradient Flows}
For practical purposes, it is once more necessary to discretise the gradient flow in \eqref{eq:wasserstein_gf}. One option is the backward Euler discretisation, which corresponds to the minimising movement \citep[Definition 2.0.6]{Ambrosio2008} or JKO \cite{Jordan1998} scheme. Another natural choice is the forward Euler scheme, which yields the Wasserstein (sub)gradient descent algorithm \citep[e.g.,][]{Guo2022}
\begin{equation}
    \mu_{t+1} = \left(\boldsymbol{\mathrm{id}} - \gamma \xi_t\right)_{\#} \mu_t~,~~~\xi_t \in \partial \mathcal{F}(\mu_t). \label{eq:wasserstein_subgradient_descent}
\end{equation}
For different choices of the functional $\mathcal{F}$, this discretisation yields the population limit of several existing particle-based algorithms. These include MMD gradient descent \cite{Arbel2019}, KSDD \cite{Korba2021}, and, replacing the Wasserstein gradient \eqref{eq:wasserstein_subgradient_descent} by a kernel approximation, SVGD \cite{Liu2016a} and LAWGD \cite{Chewi2020}. 

Regardless of the choice of numerical discretisation, the properties of the resulting algorithm depend, necessarily, on the choice of learning rate $\gamma>0$. To illustrate this point, we recall the following bound for the Wasserstein subgradient descent algorithm \citep[Theorem 8]{Guo2022}
\begin{align}
    \mathcal{F}\left(\bar{\mu}_T\right) - \mathcal{F}(\pi) \leq 
\frac{1}{T}&\left[\frac{W_2^2(\mu_1,\pi)}{2\gamma} +  \frac{L^2 T \gamma}{2} \right], \label{eq:wasserstein_gradient_descent_bound}
\end{align}
where $\bar{\mu}_T = \frac{1}{T}\sum_{t=1}^T \mu_t$, which holds under the assumption that the Wasserstein subgradients $||\xi_t||_{L^2(\mu_t)}\leq L$. We note that a similar bound also holds for the Langevin Monte Carlo (LMC) algorithm \citep[Sec.  3]{Durmus2019}. 

Based on \eqref{eq:wasserstein_gradient_descent_bound}, one can obtain the optimal worst case learning rate as $\gamma_{\text{ideal}} = \frac{W_2(\mu_1,\pi)}{L\sqrt{T}}$, and thus the optimal error bound is given by
\begin{equation}
    \mathcal{F}\left(\bar{\mu}_T\right) - \mathcal{F}(\pi) \leq \frac{L W_2(\mu_1,\pi)}{\sqrt{T}}. 
\end{equation}
Similar to the Euclidean case, however, this rate cannot be achieved in practice. In particular, computing $\gamma_{\text{ideal}}$ now depends on the unknown Wasserstein distance $W_2(\mu_1,\pi)$.

\subsection{Coin Wasserstein Gradient Descent}
\label{sec:learning-rate-freee-wasserstein}

We now introduce an alternative approach to solving \eqref{eq:wasserstein_optimisation} which is entirely learning rate free. Consider a {gambler}, indexed by some initial bet $x_0\in\mathbb{R}^d$, who bets on a series of outcomes $c_t = c_t(x_0)\in[-1,1]^d$. Similar to before, we assume that this gambler has initial wealth $w_0>0$. In the $t^{\text{th}}$ round, we now suppose that this gambler bets $x_t - x_0 \in\mathbb{R}^d$ on the outcome $c_t\in\mathbb{R}^d$. The wealth $w_t = w_t(x_0)$ of the gambler thus accumulates as 
\begin{equation}
    w_t = w_0 + \sum_{s=1}^t \langle c_s, x_s - x_0 \rangle .
\end{equation}
We will assume, similar to before, that the bets $x_t - x_0$ satisfy $x_t - x_0 = \beta_t w_{t-1}$, for some vector-valued betting fraction $\beta_t = \beta_t(x_0)\in[-1,1]^{d}$. In fact, henceforth we will always assume that $\smash{\beta_t = \frac{1}{t}{\sum_{s=1}^{t-1}c_s}}$, which corresponds to the KT betting strategy. The sequence of bets is thus given by 
\begin{equation}
    x_t = x_0 + \frac{\sum_{s=1}^{t-1}c_s}{t}\bigg(w_0 + \sum_{s=1}^{t-1} \langle c_s, x_s - x_0 \rangle \bigg). \label{eq:wasserstein_bets}
\end{equation}
Suppose, now, that in fact $x_0\sim \mu_0$, for some `initial betting distribution' $\mu_0\in\mathcal{P}_2(\mathbb{R}^d)$. In addition, suppose that we write $\varphi_{t}:\mathbb{R}^d\rightarrow\mathbb{R}^d$ for the function which maps $x_0\mapsto x_t$. We can then define a sequence of `betting distributions' $\mu_t\in\mathcal{P}_2(\mathbb{R}^d)$ as the push-forwards of $\mu_0$, under $\varphi_t:\mathbb{R}^d\rightarrow\mathbb{R}^d$, viz, 
\begin{equation}
    \mu_t = (\varphi_t)_{\#}\mu_0.
\end{equation}
This implies, in particular, that given $x_0\sim\mu_0$, the random variable $x_t:=\varphi_t(x_0)$ is distributed according to $\mu_t$. 

 We propose to use this framework to solve the minimisation problem in \eqref{eq:wasserstein_optimisation}. In particular, taking inspiration from \citet{Orabona2016}, we will consider a betting game in which the bets are given by \eqref{eq:wasserstein_bets}, and the outcomes are given by ${c}_t =-\frac{1}{L}\nabla_{W_2}\mathcal{F}(\mu_t)(x_t)$, where $L$ is an upper bound on the Wasserstein gradients (see Assumption \ref{assumption2}). We will refer to this betting game, summarised in Alg. \ref{alg:param_free_grad_descent_v2}, as \emph{coin Wasserstein gradient descent} or \emph{coin sampling}.

\begin{algorithm*}[t]
   \caption{Coin Wasserstein Gradient Descent}
   \label{alg:param_free_grad_descent_v2}
\begin{algorithmic}
   \STATE {\bfseries Input:} initial measure $\mu_0\in\mathcal{P}_2(\mathbb{R}_{d})$, initial parameter $x_0\sim \mu_0$, initial wealth $w_0\in\mathbb{R}_{+}$, dissimilarity functional $\mathcal{F}:\mathcal{P}_2(\mathbb{R}^d)\rightarrow(-\infty,\infty]$, gradient upper bound $L$.
   \FOR{$t=1$ {\bfseries to} $T$}
   \STATE Compute 
\begin{equation}
x_{t} = x_0 - \frac{\sum_{s=1}^{t-1} \nabla_{W_2}\mathcal{F}(\mu_s)(x_s)}{Lt} \bigg( w_0 - \sum_{s=1}^{t-1} \langle \frac{1}{L}\nabla_{W_2}\mathcal{F}(\mu_s)(x_s), x_s - x_0 \rangle \bigg)~~~,~~~\mu_t = (\varphi_{t})_{\#}\mu_{0}. \label{eq:x_transform}
\end{equation}
   \ENDFOR
   \STATE {\bfseries Output:} $\mu_T$ or $\frac{1}{T}\sum_{t=1}^T \mu_t$.
\end{algorithmic}
\end{algorithm*}

\subsection{Theoretical Results}
In this case, under a rather strong sufficient condition (see App. \ref{sec:proofs}), we can show that the average of the betting distributions $\frac{1}{T}\sum_{t=1}^T \mu_t$ converges to the target $\pi$, at a rate determined by the betting strategy. For this result to hold, we will also require the following assumptions.

\begin{assumption} \label{assumption1}
The functional $\mathcal{F}:\mathcal{P}_2(\mathbb{R}^d) \rightarrow (-\infty,\infty]$ is (i) proper and lower semi-continuous, and (ii) geodesically convex. 
\end{assumption}

\begin{assumption} \label{assumption2}
There exists $L>0$ such that, for all $t\in[T]$,  $||\nabla_{W_2}\mathcal{F}(\mu_t)(x_t)||\leq L$.
\end{assumption}

Assumption \ref{assumption1}(i) is a general technical condition satisfied in all relevant cases \citep[e.g.,][Sec. 10]{Ambrosio2008}. Assumption \ref{assumption1}(ii) is a standard condition used in the analysis of existing algorithms such as LMC \cite{Wibisono2018,Durmus2019a}. This assumptions holds, for example, if $\mathcal{F}(\mu) = \mathrm{KL}(\mu|\pi)$, and the potential $U:\mathbb{R}^d\rightarrow\mathbb{R}$ is convex \citep[Sec. 9.4]{Ambrosio2008}. 

To our knowledge, Assumption \ref{assumption2} has also only explicitly appeared in the analysis of the Wasserstein subgradient descent algorithm in \citet{Guo2022}. However, similar conditions have also been used to analyse the convergence of SVGD to its population limit (\citealp{Liu2017a}, Theorem 3.2; \citealp{Korba2020}, Proposition 7). Meanwhile, convergence rates for SVGD (in the infinite particle regime) can be established under boundedness assumptions for the kernel function, as well as bounds on either the KSD \cite{Liu2017a}, the Stein Fisher information \cite{Korba2020}, or the Hessian of the potential \cite{Salim2022,Shi2022a} at each iteration.

\begin{proposition} \label{theorem:convergence_rate}
Let Assumptions \ref{assumption1} - \ref{assumption2} and Assumption \ref{sufficient_condition_2} (see App. \ref{sec:proofs}) hold. Then
    \begin{align}
    &\mathcal{F}\left(\frac{1}{T}\sum_{t=1}^T \mu_t\right) - \mathcal{F}(\pi)
    \leq \frac{L}{T} \bigg[ ~w_0 \label{eq:theorem1_bound}\\
    &+\int_{\mathbb{R}^d} \left|\left|x\right|\right| \sqrt{T \ln \left(1+ \frac{96K^2T^2\left|\left|x\right|\right|^2}{w_0^2}\right)} \pi(\mathrm{d}x) \nonumber \\
    &+\int_{\mathbb{R}^d} \left|\left|x\right|\right| \sqrt{T \ln \left(1+ \frac{96T^2\left|\left|x\right|\right|^2}{w_0^2}\right)} \mu_0(\mathrm{d}x) \bigg]. \nonumber 
\end{align}
\end{proposition}
\vspace{-7mm}
\begin{proof} 
See App. \ref{sec:proofs}.
\end{proof}
\vspace{-3mm}

The proof of Proposition \ref{theorem:convergence_rate} closely follows the proof used to establish the convergence rate of the parameter-free optimisation algorithm in \citet{Orabona2016}. In our case, however, it is no longer evident how to convert a lower bound on the wealth into an upper bound on the regret (see Lemma \ref{orabona_lemma}). In App. \ref{sec:proofs}, we provide a technical sufficient condition (Assumption \ref{sufficient_condition_2}) which allow us to obtain the rate in Proposition \ref{theorem:convergence_rate}. It is unclear, however, how to verify this condition in practice. We leave as an open question whether it is possible to obtain more easily verifiable conditions under which this result still holds.

\subsection{Practical Implementation}
In principle, Alg. \ref{alg:param_free_grad_descent_v2} requires knowledge of a bound on the Wasserstein gradients (see Assumption \ref{assumption2}). If such a constant is unknown in advance, then it can be adaptively estimated using a similar approach to the one proposed in \citet{Orabona2017}. We provide full details of this adaptive approach, which in practice we use in all of our numerical experiments, in App. \ref{sec:coinWGD_adaptive}. 

Alg. \ref{alg:param_free_grad_descent_v2} also assumes that it is possible to observe the sequence of vector fields $(\nabla_{W_2}\mathcal{F}(\mu_t)_{t\in[T]}$. In practice, this is unrealistic: these quantities depend on knowledge of the measures $(\mu_t)_{t\in[T]}$, which typically we cannot compute in closed form. Following existing ParVIs, a standard approach is to approximate these quantities using a set of interacting particles. In particular, suppose we initialise $\smash{
(x_0^{i})_{i=1}^N\stackrel{\text{i.i.d.}}{\sim} \mu_0(\mathrm{d}x)}$, with empirical law $\smash{\mu_0^N = \frac{1}{N}\sum_{i=1}^n \delta_{x_0^{i}}}$. We can then update the particles according to an empirical version of \eqref{eq:x_transform}. This yields, after each iteration, particles $(x_t^{i})_{i=1}^N$, with empirical distribution $\smash{\mu_t^{N} = \frac{1}{N}\sum_{i=1}^N\delta_{x_t^{i}}}$. 

This approach relies, crucially, on being able to compute or approximate $\smash{(\nabla_{W_2}\mathcal{F}(\mu_t^{N}))_{t\in[T]}}$, the Wasserstein gradients of $\mathcal{F}$ evaluated at $\smash{(\mu_t^{N})_{t\in[T]}}$. Fortunately, this is also central to existing particle-based sampling algorithms, including SVGD \cite{Liu2016a}, KSDD \cite{Korba2020}, and LAWGD \cite{Chewi2020}. We can thus take inspiration from these methods to compute or to approximate the required terms. In fact, for different choices of $\mathcal{F}$, and different approximations of $\nabla_{W_2}\mathcal{F}(\mu_t^{N})$, we obtain learning-rate free versions of SVGD (Alf. \ref{alg:param_free_svgd}), LAWGD (App. \ref{sec:coin-lawgd}), and KSDD (App. \ref{sec:coin-ksd}). We refer to these algorithms as Coin SVGD, Coin LAWGD, and Coin KSDD, respectively.

\textbf{Coin Stein Variational Gradient Descent}. 
\label{sec:coin-svgd}
\begin{algorithm*}[t]
   \caption{Coin Stein Variational Gradient Descent (Coin SVGD)}
   \label{alg:param_free_svgd}
\begin{algorithmic}
   \STATE {\bfseries Input:} initial measure $\mu_0\in\mathcal{P}_2(\mathbb{R}_{d})$, initial particles $(x_0^{i})_{i=1}^N \stackrel{\mathrm{i.i.d.}}{\sim}\mu_0$, initial wealth of particles $(w_0^{i})_{i=1}^N\in\mathbb{R}_{+}$, kernel $k$, gradient upper bound $L$. 
   \FOR{$t=1$ {\bfseries to} $T$}
   \FOR{$i=1$ {\bfseries to} $N$} 
   \STATE Compute, with $P_{\mu_s^N,k}\nabla_{W_2}\mathcal{F}(\mu_s^N)(\cdot)$ defined as in \eqref{eq:SVGD_grad}, \vspace{-1mm}
\begin{align}
\hspace{-7mm} x_{t}^{i} &= x_0^{i} -\frac{{\sum_{s=1}^{t-1} P_{\mu_s^N,k}\nabla_{W_2}\mathcal{F}(\mu_s^N)(x_s^{i})}}{Lt} \bigg( w_0^{i} - \sum_{s=1}^{t-1} \bigg\langle \frac{1}{L}P_{\mu_s^N,k}\nabla_{W_2}\mathcal{F}(\mu_s^N)(x_s^{i}), x_s^{i} - x_0^{i} \bigg\rangle \bigg),\quad 
\vspace{-2mm}
\end{align}
   \STATE Define $\smash{\mu_{t}^{N} = \frac{1}{N}\sum_{i=1}^N \delta_{x_t^{i}}}$. \vspace{2mm}
   \ENDFOR
   \ENDFOR
   \STATE {\bfseries Output:} $\mu_T^{N}$ or $\frac{1}{T}\sum_{t=1}^T \mu_t^{N}$.
\end{algorithmic}
\end{algorithm*}
We now provide further details on Coin SVGD. Let $\mathcal{F}(\mu) = \mathrm{KL}(\mu|\pi)$, with $\nabla_{W_2}\mathcal{F}(\mu) = \nabla \ln \frac{\mu}{\pi}$. Let $\smash{k:\mathbb{R}^d\times\mathbb{R}^d\rightarrow\mathbb{R}}$ denote a positive semi-definite kernel, and $\mathcal{H}_k$ the associated reproducing kernel Hilbert space (RKHS). Finally, let $P_{\mu,k}: L^2(\mu)\rightarrow L^2(\mu)$ denote the integral operator defined according to $P_{\mu,k} f(\cdot) = \int_{\mathbb{R}^d} k(x,\cdot) f(x)\mu(\mathrm{d}x)$.

Following \citet{Liu2016a}, suppose that we replace $\nabla_{W_2}\mathcal{F}(\mu)$ by $P_{\mu,k}\nabla_{W_2}\mathcal{F}(\mu)$ in Alg. \ref{alg:param_free_grad_descent_v2}, its image under the integral operator $P_{\mu,k}$. This essentially plays the role of the Wasserstein gradient in $\mathcal{H}_k$. Using integration by parts, and recalling that $\pi\propto e^{-U}$, it holds that   $P_{\mu,k}\nabla_{W_2}\mathcal{F}(\mu)=\mathbb{E}_{x\sim \mu}\left[k(x,\cdot) \nabla U(x) - \nabla_{1}k(x,\cdot)\right]$. Thus, in particular,
\begin{align}
    &P_{\mu_t^N,k}\nabla_{W_2}\mathcal{F}(\mu_t^N)(x_t^{i}) \nonumber \\
    &= \tfrac{1}{N}\textstyle\sum_{j=1}^N[k(x_t^{j},x_t^{i}) \nabla U(x_t^{j}) -  \nabla_{1}k(x_t^{j},x_t^{i})].  \label{eq:SVGD_grad}
\end{align}
Substituting this expression into Alg. \ref{alg:param_free_grad_descent_v2}, we arrive at a learning-rate free analogue of SVGD. This algorithm is summarised in Alg. \ref{alg:param_free_svgd}. We note this algorithm is not entirely tuning free, since it requires a choice of bandwidth for the kernel. In practice, however, this parameter can be tuned automatically using the median rule \cite{Liu2016a}.

\section{Numerical Results}
\label{sec:numerics}
In this section, we evaluate the numerical performance of Coin SVGD (Alg. \ref{alg:param_free_svgd}). We provide additional results for Coin LAWGD (Alg.  \ref{alg:param_free_lawgd}) and Coin KSDD (Alg.  \ref{alg:param_free_ksd}) in App. \ref{sec:numerics_LAWGD} and \ref{sec:numerics_KSD}. In all experiments, we implement the adaptive version of Coin SVGD (see App. \ref{sec:coinWGD_adaptive}). For both Coin SVGD and SVGD, we use the RBF kernel $k(x,x') = \exp(-\frac{1}{h}||x-x'||_2^2)$, with bandwidth chosen using the median heuristic in \citet{Liu2016a}. Additional implementation details and results are provided in App. \ref{sec:experimental-details}. Code to reproduce our numerical results can be found at \url{https://github.com/louissharrock/Coin-SVGD}.

\subsection{Toy Examples}
\label{sec:toy_examples}

We first illustrate the performance of Coin SVGD on a series of toy examples (see App. \ref{sec:toy-details} for full details). In Fig. \ref{fig:figure1} (see Sec. \ref{sec:intro}), we plot the samples generated by Coin SVGD and SVGD after $T=1000$ iterations, using $N=20$ particles. In all examples, Coin SVGD qualitatively appears to converge to the correct target distribution. 

In Fig. \ref{fig:figure1_KSD} - \ref{fig:KSD_v_N} (App. \ref{sec:toy-details}), we provide a more quantitative assessment of our method, plotting the KSD and the energy distance \cite{Szekely2013} between the targets in Fig. \ref{fig:figure1}, and the approximations obtained by Coin SVGD and SVGD. Our results indicate that the performance of Coin SVGD is competitive with the best performance of SVGD (i.e., the performance of SVGD using the optimal but {a priori} unknown learning rate) (Fig. \ref{fig:figure1_KSD}) and that Coin SVGD often converges more rapidly to the target distribution (Fig. \ref{fig:figure1_KSD_vs_t}). They also confirm, as expected, that Coin SVGD generates increasingly accurate posterior approximations as the number of particles $N$ increases (Fig. \ref{fig:figure1_energy} and Fig. \ref{fig:KSD_v_N}).

\subsection{Bayesian Independent Component Analysis}
\label{sec:bayes_ica}
We next consider a Bayesian independent component analysis (ICA) model \citep[e.g.,][]{Comon1994}. Suppose we observe $\boldsymbol{x} = (x_1,\dots,x_p)\in\mathbb{R}^p$. The task of ICA is to infer the `unmixing matrix' $\mathbf{W}\in\mathbb{R}^{p \times p}$ such that $\boldsymbol{x}= \mathbf{W}^{-1} \boldsymbol{s}$, where $\boldsymbol{s}=(s_1,\dots,s_p)\in\mathbb{R}^p$ denote the latent independent sources. We will assume each $s_i$ has the same density: $s_i\sim p_s$. The log-likelihood of this model is then given by $\log p(\boldsymbol{x}|\mathbf{W}) = \log |\mathbf{W}| + \sum_{i=1}^p p_s([\mathbf{W}\boldsymbol{x}]_{i})$. For the prior, we assume that the entries of $\mathbf{W}$ are i.i.d., with law $\mathcal{N}(0,1)$. The posterior is then $p(\mathbf{W}|\boldsymbol{x}) \propto p(\boldsymbol{x}|\mathbf{W}) p(\mathbf{W})$, with
\begin{equation}
\nabla_{\mathbf{W}} \log p(\mathbf{W}|\boldsymbol{x}) = (\mathbf{W}^{-1})^{\top} - \frac{p'_s(\mathbf{W}\boldsymbol{x})}{p_s(\mathbf{W}\boldsymbol{x})} \boldsymbol{x}^{\top} - \mathbf{W}. 
\end{equation}
Following \citet{Korba2021}, we choose $p_s$ such that $\smash{{p_s'(\cdot)}/{p_s(\cdot)} = \tanh(\cdot)}$. We are interested in sampling from $p(\mathbf{W}|\boldsymbol{x})$. In our experiments, we generate 1000 samples of $\boldsymbol{x}$ from the ICA model, for $p\in\{2,4, 8, 16\}$. We use $N=10$ particles, so that each algorithm returns 10 estimated unmixing matrices $(\bar{\boldsymbol{W}}_i)_{i=1}^{10}$. We then repeat each experiment 50 times, thus obtaining 500 estimates for each method. To assess convergence, we compute the Amari distance \cite{Amari1995} between the true $\mathbf{W}$ and the estimates 
generated by each algorithm. 
This is equal to zero if and only if the two matrices are the same up to scale and permutation. 
We run SVGD for three learning rates: an `optimal' rate, which we determine by running SVGD over a range of six candidate values $\gamma\in[1\times 10^{-5},1\times 10^{0}]$, and selecting the one which returns the lowest average Amari distance, a smaller rate, and a larger rate. We also include the results of a random output, where the estimated matrices have entries which are generated i.i.d. $\mathcal{N}(0,1)$. 

\begin{figure}[t!]
\centering
\subfigure[$p=2$. \label{fig:BayesICAp2}]{\includegraphics[trim=0 0mm 5 0mm, clip, width=.467\columnwidth]{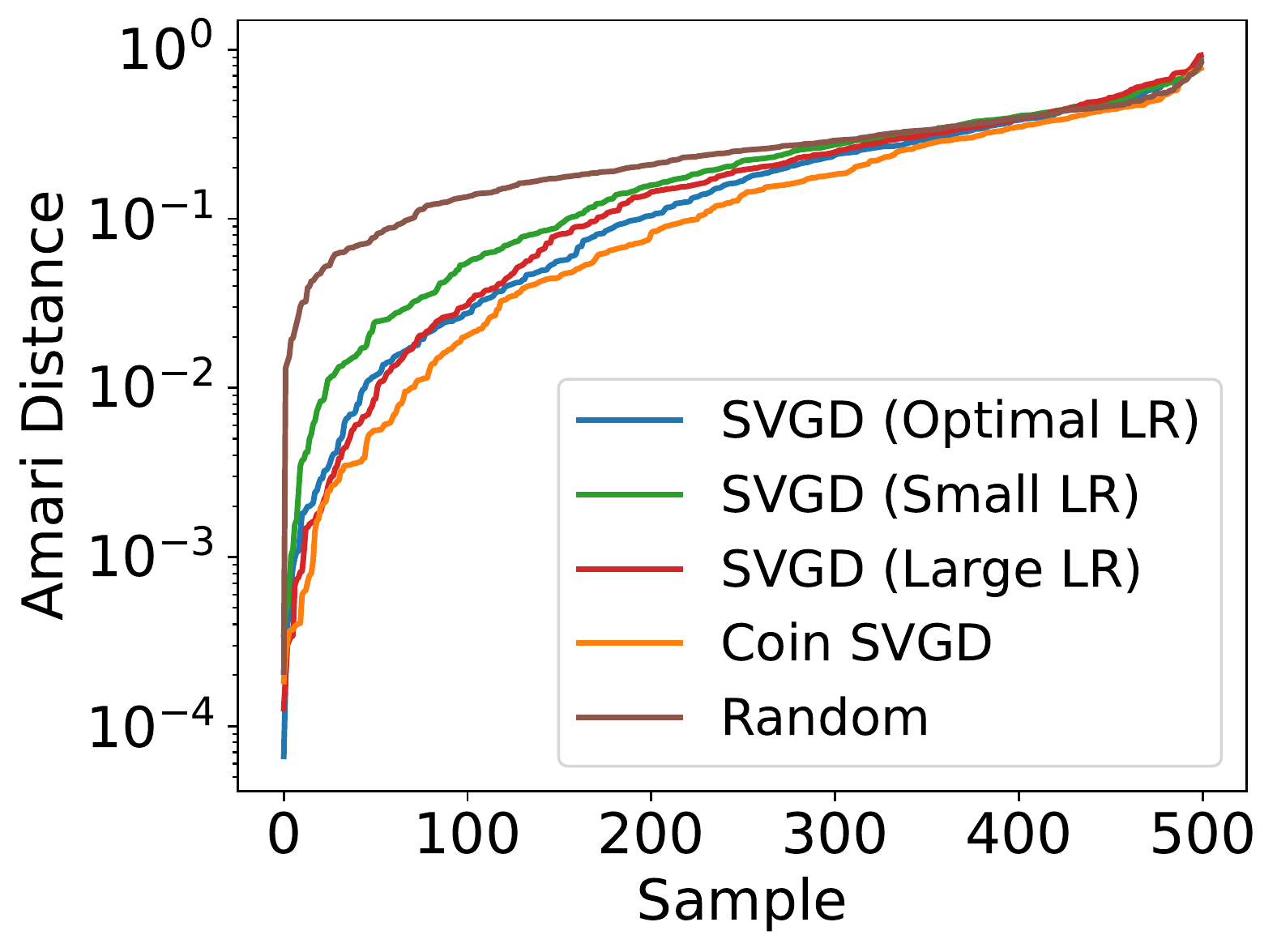}} \hfill
\subfigure[$p=4$.]{\includegraphics[trim=0 0mm 5 0mm, clip, width=.467\columnwidth]{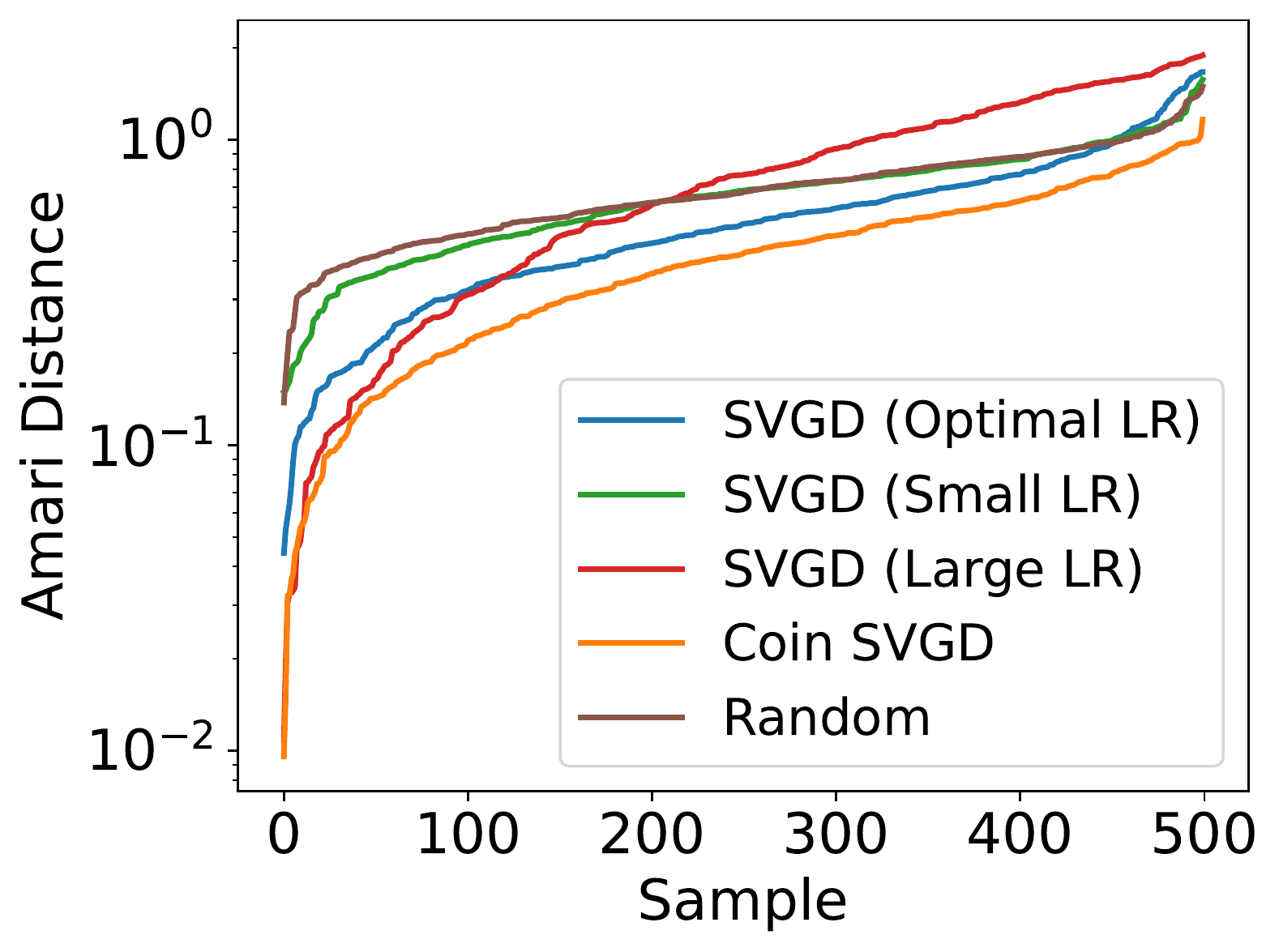}}
\subfigure[$p=8$.]{\includegraphics[trim=0 0mm 5 0mm, clip, width=.467\columnwidth]{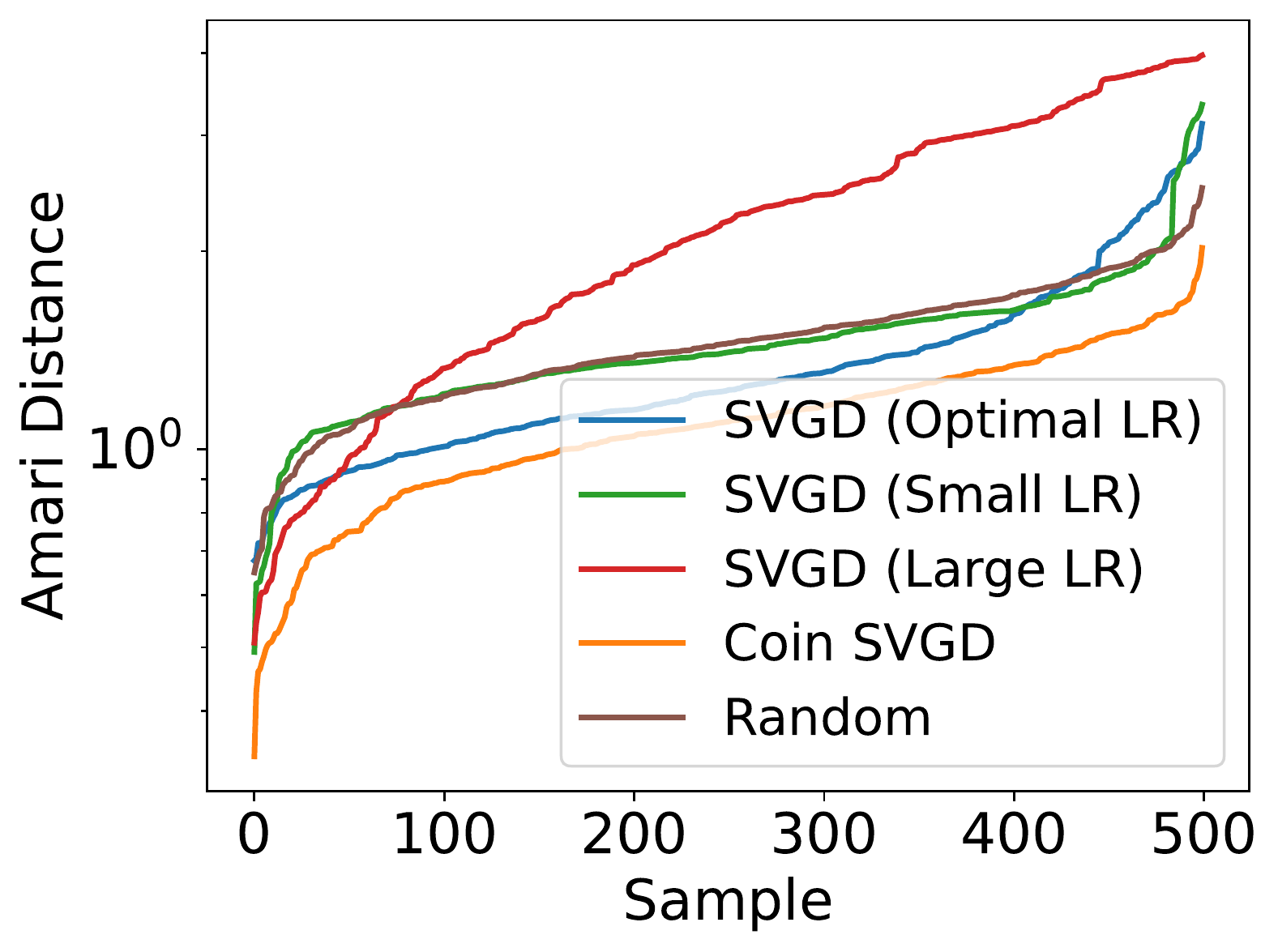}} \hfill
\subfigure[$p=16$.]{\includegraphics[trim=0 0mm 5 0mm, clip, width=.467\columnwidth]{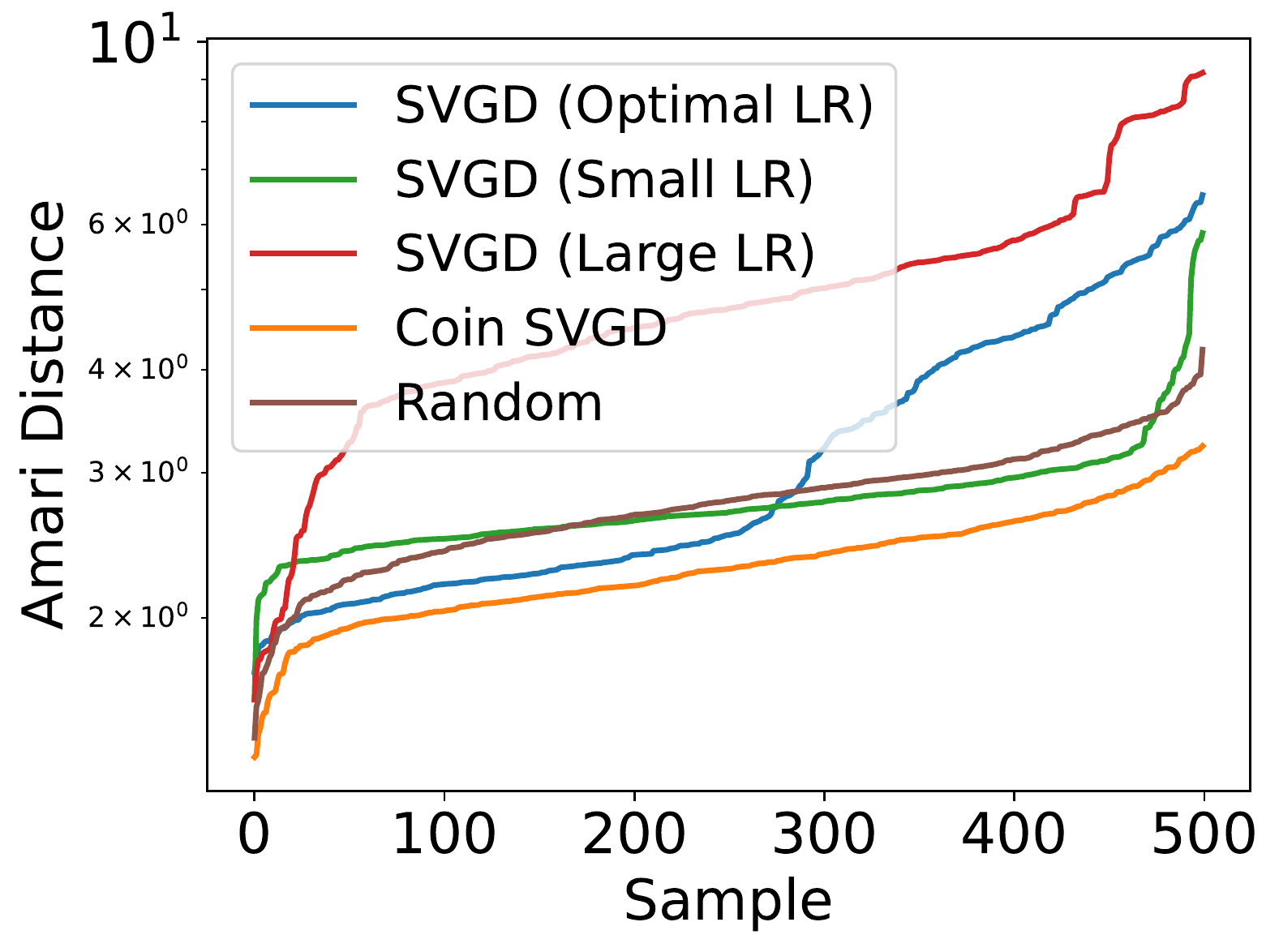}}
\vspace{-2mm}
\caption{\textbf{Results for the Bayesian ICA model}. Amari distances between the true unmixing matrix $\boldsymbol{W}$, and the 500 approximate unmixing matrices output by Coin SVGD and SVGD, for three different values of the learning rate.}
\label{fig:BayesICA}
\vspace{-5mm}
\end{figure}

Our results are plotted in Fig. \ref{fig:BayesICA}. For lower dimensional data, Coin SVGD performs similarly to SVGD with the optimal learning rate (see Fig. \ref{fig:BayesICAp2}). In fact, in this case, using a smaller or larger learning rate does not have a significant effect on the performance of SVGD. On the other hand, for higher dimensions, the gap between the two algorithms increases, as does the importance of choosing a good learning rate for SVGD. In particular, for $p\in\{4,8,16\}$, Coin SVGD increasingly {outperforms} SVGD, for any choice of the learning rate. These results perfectly illustrate the robustness of Coin SVGD: our algorithm performs consistently well across these experiments, even as the dimension varies.

\subsection{Bayesian Logistic Regression}
\label{sec:bayes-lr}
We next consider the Bayesian logistic regression model for binary classification, as described in \citet{Gershman2012}. Let $\mathcal{D} = (\boldsymbol{x}_i,y_i)_{i=1}^N$ be a dataset with feature vectors $\boldsymbol{x}_i\in\mathbb{R}^p$, and binary labels $y_i\in\{-1,1\}$. We assume that $p(y_i=1 |\boldsymbol{x}_i, \boldsymbol{w}) = (1+\exp(-\boldsymbol{w}^T \boldsymbol{x}_i))^{-1}$, for some $\boldsymbol{w}\in\mathbb{R}^p$. We place a Gaussian prior $p(\boldsymbol{w}|\alpha) = \mathcal{N}(\boldsymbol{w}|0,\alpha^{-1})$ on the regression weights $\boldsymbol{w}$, and a Gamma prior $p(\alpha) = \mathrm{Gamma}(\alpha | 1, 0.01)$ on $\alpha \in\mathbb{R}_{+}$. We would like to sample from $p(\boldsymbol{\theta}|\mathcal{D})$, where the parameter of interest is $\boldsymbol{\theta} = [\boldsymbol{w},\log \alpha]^T\in\mathbb{R}^{p+1}$. We test our algorithm using the Covertype dataset, which consists of 581,012 data points and 54 features. We randomly partition the data into a training dataset (70\%), validation dataset (10\%), and testing dataset (20\%). We run each algorithm with $N=50$ particles for $T=5000$ iterations, and compute stochastic gradients using mini-batches of size 100. The results are averaged over 20 random train-test splits.

In Fig. \ref{fig:BayesLRa} - \ref{fig:BayesLRb}, we plot the test accuracy and the negative log-likelihood for Coin SVGD, and for SVGD as a function of the step size. Meanwhile, in Fig. \ref{fig:BayesLRc} - \ref{fig:BayesLRd}, we plot the test accuracy and the negative log-likelihood against the number of iterations. Once again, we consider three learning rates for SVGD: the optimal learning rate as determined by the results in Fig. \ref{fig:BayesLRa}, a smaller learning rate, and a larger learning rate. Similar to before, the performance of Coin SVGD is similar to the best performance of SVGD. On the other hand, when the learning rate is too small or too large, SVGD either converges slowly (green lines, Fig. \ref{fig:BayesLRc} - \ref{fig:BayesLRd}) or is unstable (red lines, Fig. \ref{fig:BayesLRc} - \ref{fig:BayesLRd}).

\begin{figure}[t!]
\centering
\subfigure[Test Accuracy. \label{fig:BayesLRa}]{\includegraphics[trim=0 0mm 15mm 13mm, clip, width=.485\columnwidth]{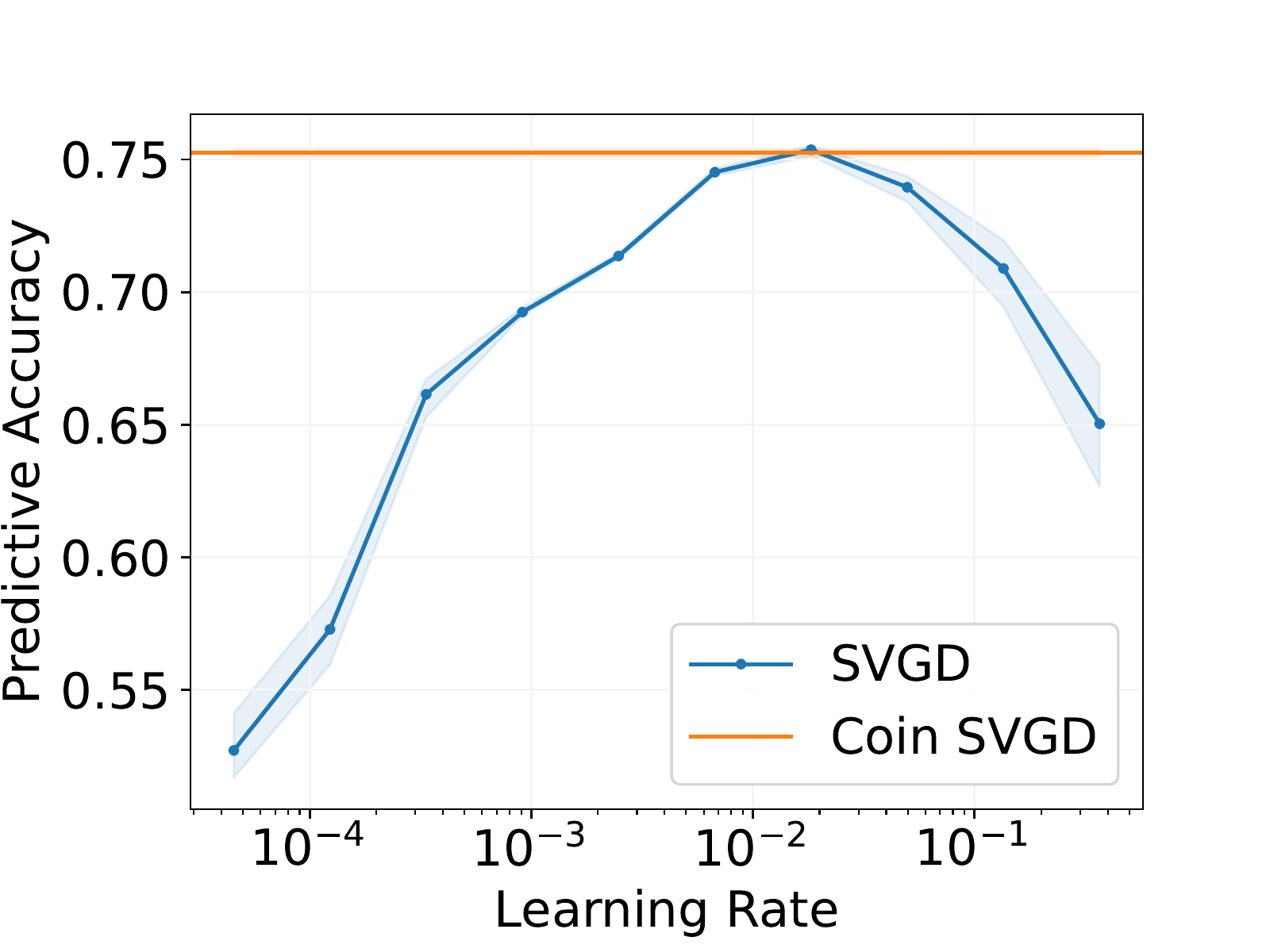}} \hfill
\subfigure[Negative Log-Likelihood. \label{fig:BayesLRb}]{\includegraphics[trim=0 0 15mm 13mm, clip, width=.485\columnwidth]{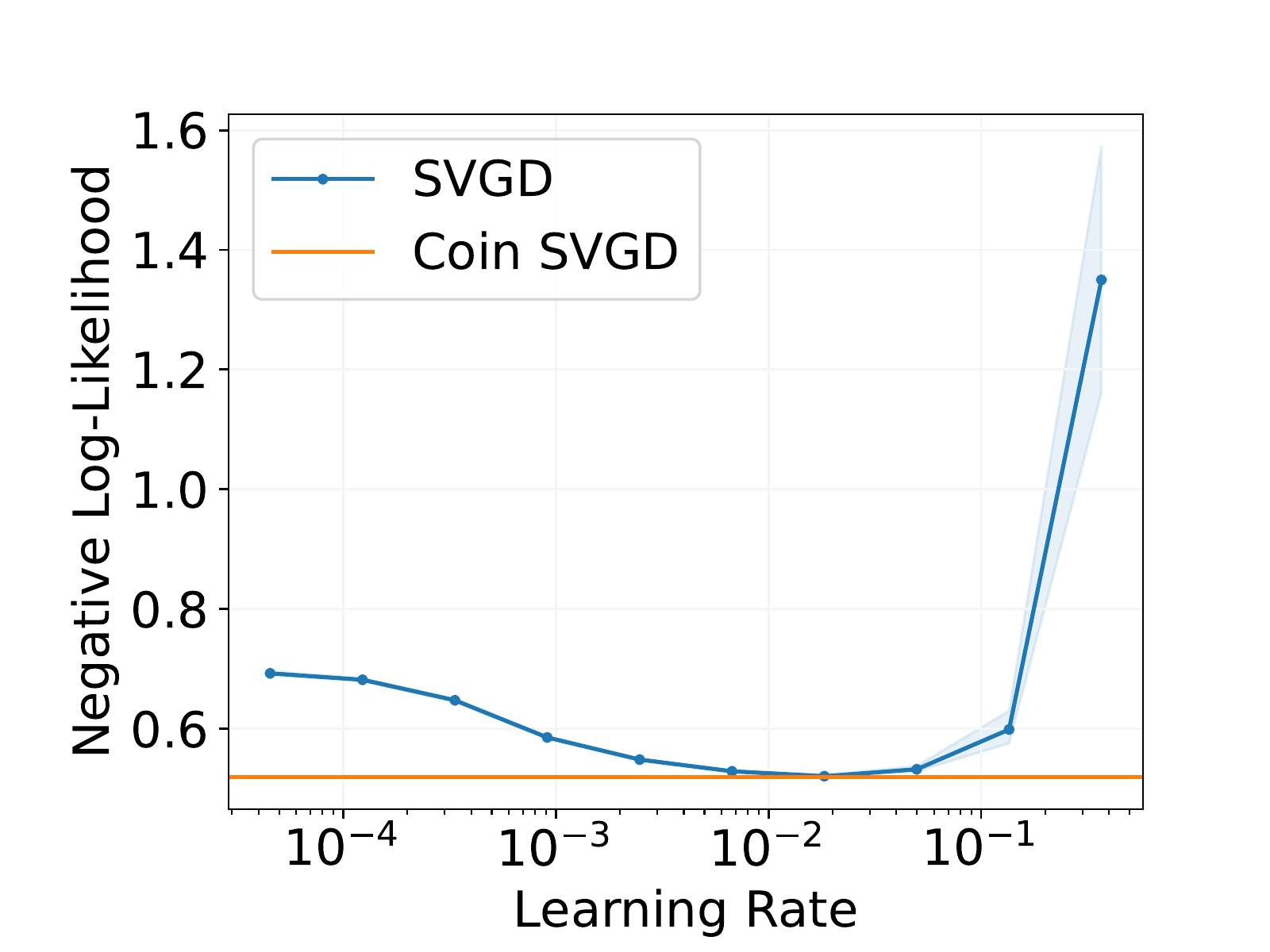}} \\
\subfigure[Test Accuracy. \label{fig:BayesLRc}]{\includegraphics[trim=0 0mm 0mm 0mm, clip, width=.485\columnwidth]{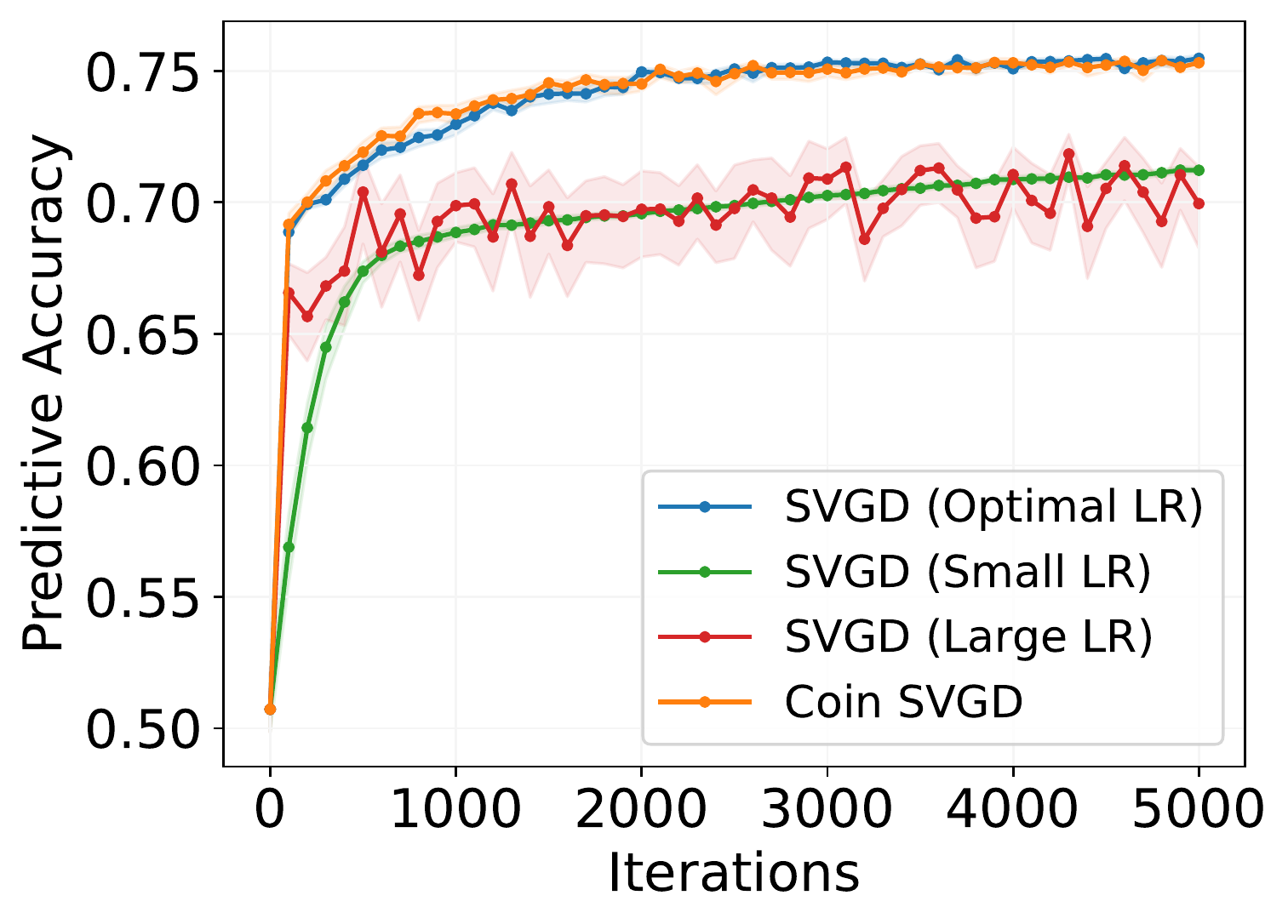}} \hfill
\subfigure[Negative Log-Likelihood. \label{fig:BayesLRd}]{\includegraphics[trim=0 0 0mm 0mm, clip, width=.495\columnwidth]{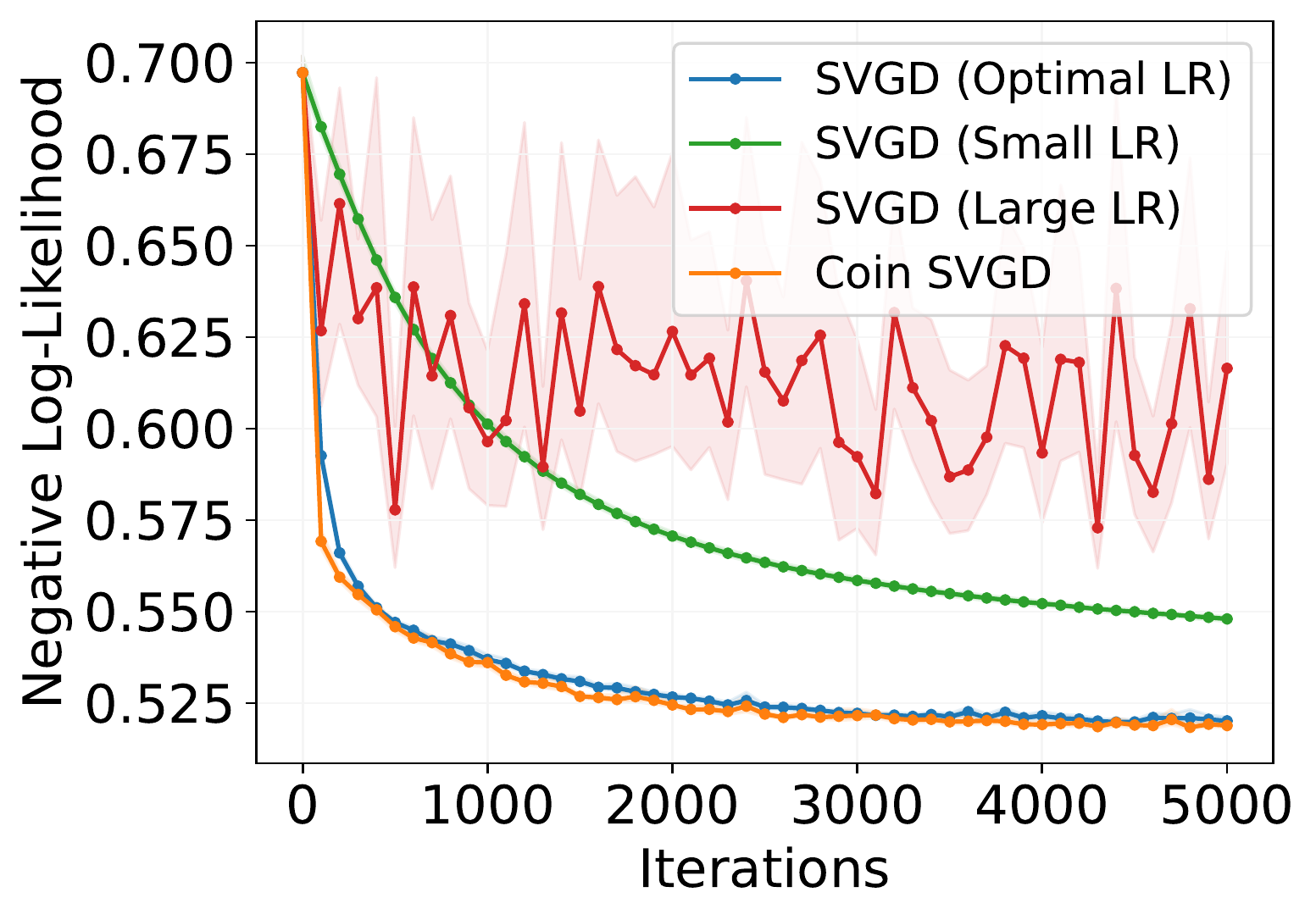}}
\vspace{-2mm}
\caption{\textbf{Results for the Bayesian logistic regression model}. (a)-(b). Test accuracy and negative log-likelihood for Coin SVGD and SVGD, as a function of the learning rate. (c)-(d).  Test accuracy and negative log-likelihood for Coin SVGD and SVGD (three learning rates) as a function of the number of iterations.}
\label{fig:BayesLR2}
\vspace{-5mm}
\end{figure}

\subsection{Bayesian Neural Network}
\label{sec:bnn}

We next consider a Bayesian neural network model. Our settings are identical to those given in \citet{Liu2016a}; see also \citet{Hernandez-Lobato2015}. In particular, we use a two-layer neural network with 50 hidden units with $\mathrm{RELU}(x) = \max(0,x)$ as the activation function. We assume the output is normal, and place a $\mathrm{Gamma}(1,0.1)$ prior on the inverse covariance. We then assign an isotropic Gaussian prior to the neural network weights. We test the performance of our algorithms on several UCI datasets. The datasets are partitioned into 90\% for training and 10\% for testing, and our results are averaged over 20 random train-test splits. Finally, we use $N=20$ particles, and consider a snapshot of the performance after $T=2000$ iterations.

Our results, shown in Fig. \ref{fig:BayesNN} (see also Fig. \ref{fig:BayesNN-additional} in App. \ref{sec:bnn-extra}), indicate that SVGD slightly outperforms Coin SVGD for well chosen learning rates, but significantly under-performs Coin SVGD when the learning rate is too small or too large. For certain datasets, the performance of Coin SVGD is close to the optimal performance of SVGD, while for others, there remains a reasonable performance gap. We expect that this could be reduced using recent advancements in parameter-free stochastic optimisation \citep[e.g.][]{Chen2022,Chen2022a}.

\begin{figure}[t!]
\centering
\subfigure[Boston.]{\includegraphics[trim=0 0mm 15mm 13mm, clip, width=.49\columnwidth]{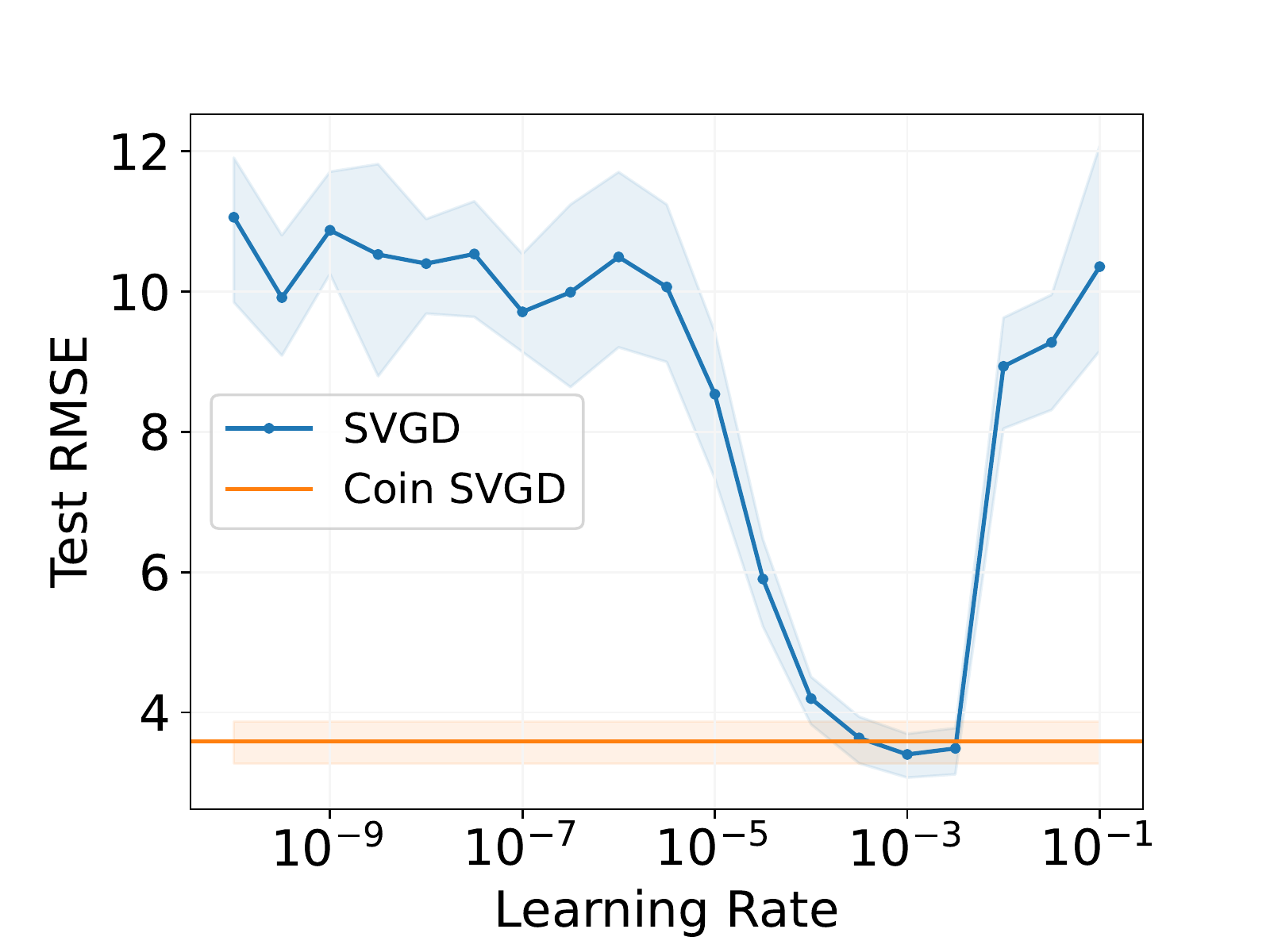}} \hfill
\subfigure[Concrete.]{\includegraphics[trim=0 0mm 15mm 13mm, clip, width=.49\columnwidth]{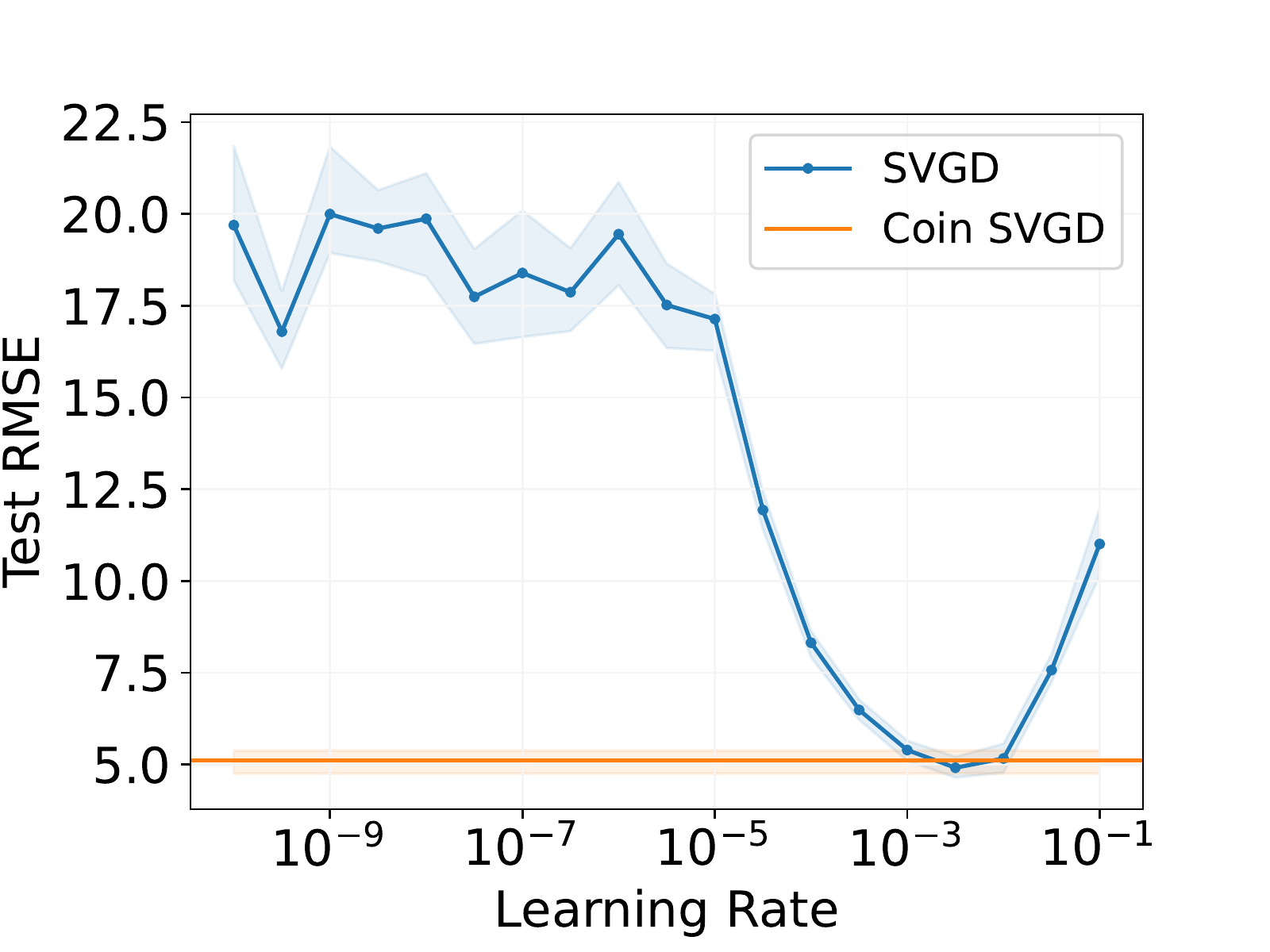}} \hfill
\subfigure[Energy.]{\includegraphics[trim=0 0mm 15mm 13mm, clip, width=.49\columnwidth]{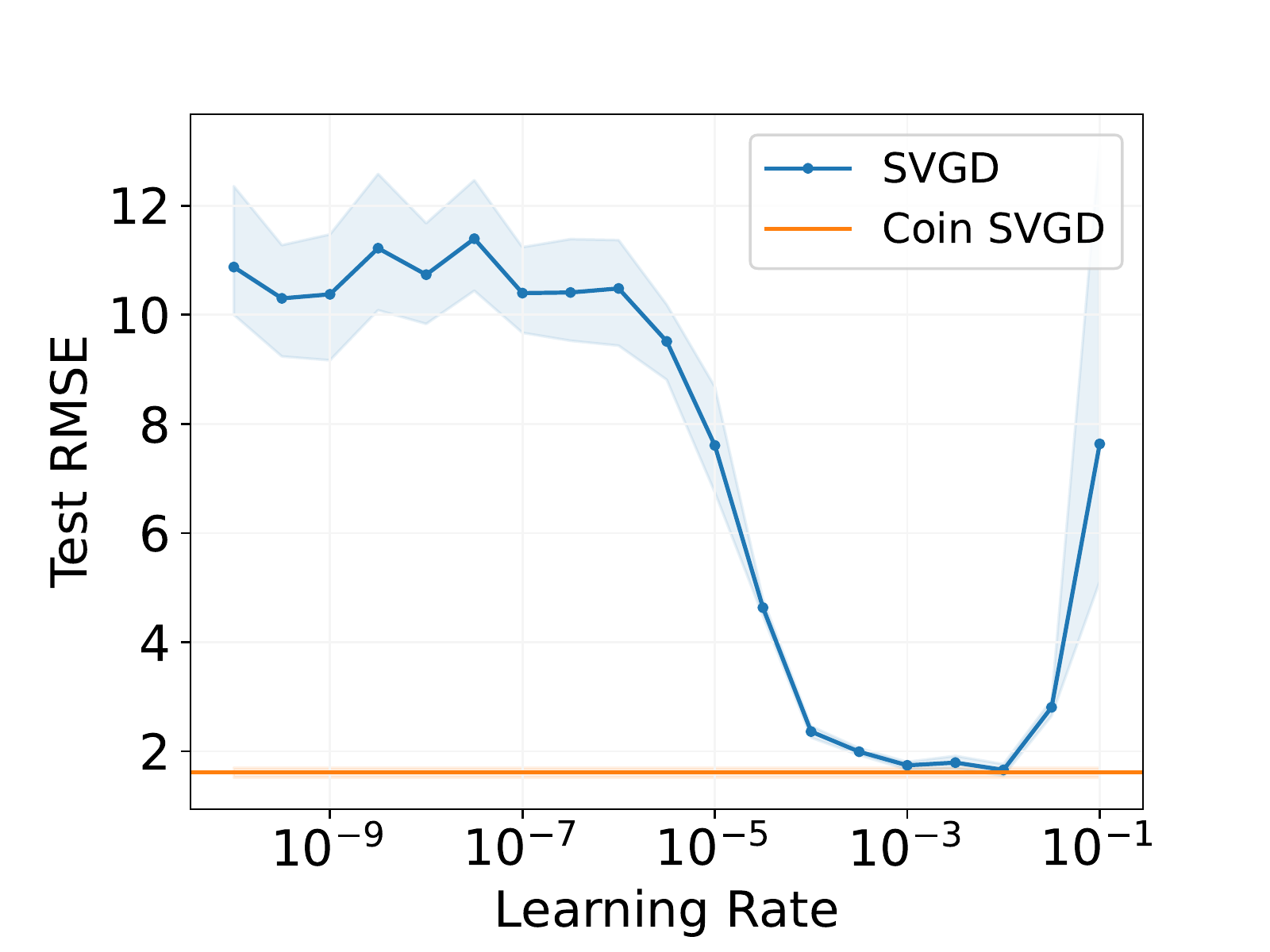}} \hfill
\subfigure[Kin8nm.]{\includegraphics[trim=0 0mm 15mm 13mm, clip, width=.49\columnwidth]{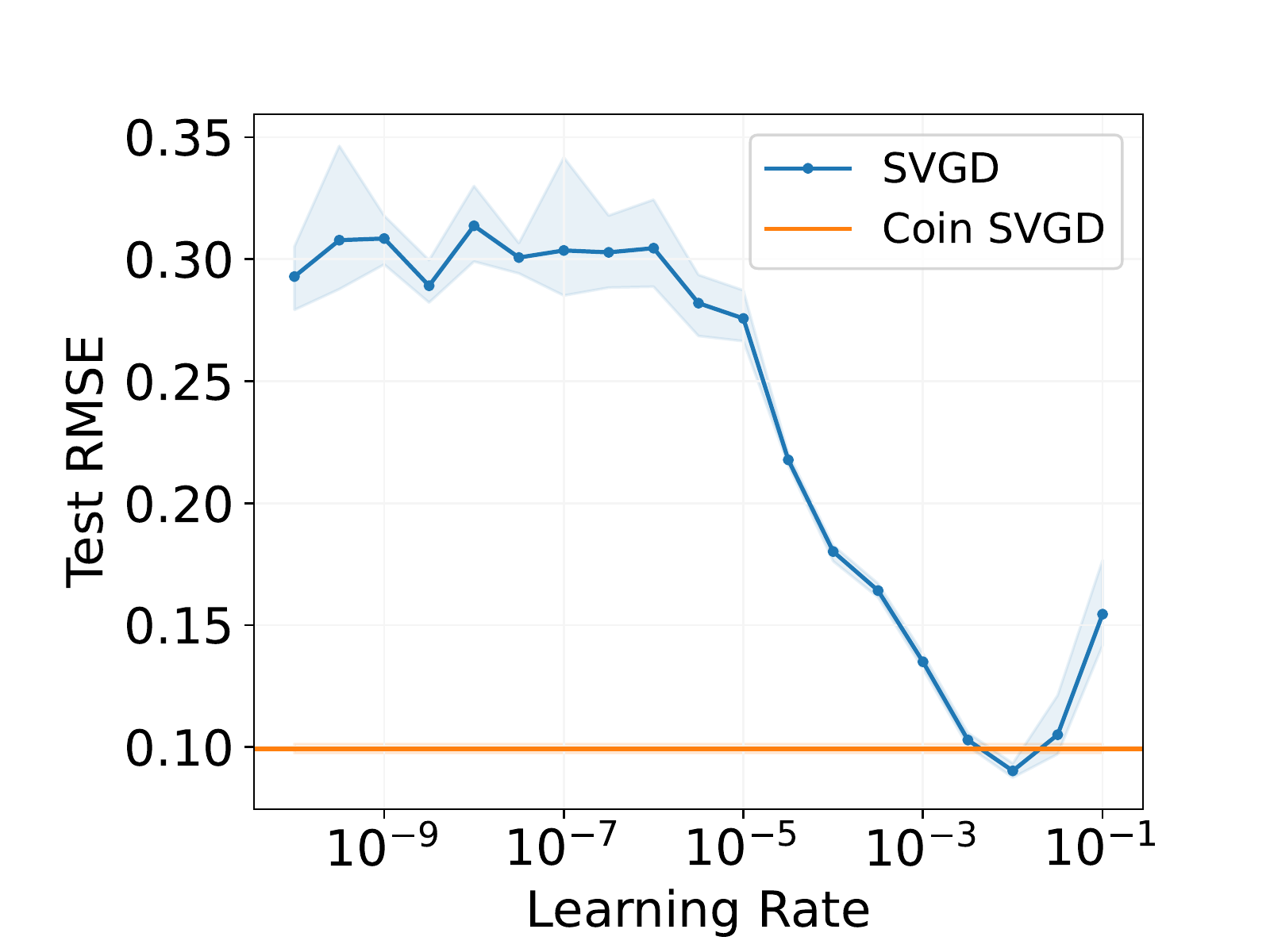}} \hfill
\vspace{-4mm}
\caption{\textbf{Results for the Bayesian neural network}. Average test RMSE for Coin SVGD and SVGD, as a function of the learning rate, after $T=2000$ iterations, for several UCI datasets.}
\label{fig:BayesNN}
\vspace{-4mm}
\end{figure}

\subsection{Bayesian Probabilistic Matrix Factorisation}
\label{sec:PCA}
Finally, we consider a Bayesian probabilistic matrix factorisation (PMF) model \citep{Salakhutdinov2008}. This model is defined as follows. Let $\mathbf{R}\in\mathbb{R}^{N\times M}$ be a matrix of ratings for $N$ users and $M$ movies, where $R_{ij}$ is the rating user $i$ gave to movie $j$. Define matrices $\mathbf{U}$ and $\mathbf{V}$ for users and movies, respectively, where $\mathbf{U}_i\in\mathbb{R}^d$ and $\mathbf{V}_j\in\mathbb{R}^d$ are $d$-dimensional latent feature vectors for user $i$ and movie $j$. The likelihood for the rating matrix is given by
\begin{equation}
    p(\mathbf{R}|\mathbf{U},\mathbf{V},\alpha) = \prod_{i=1}^N \prod_{j=1}^M \left[ \mathcal{N}(R_{ij}|\mathbf{U}_i^T \mathbf{V}_j, \alpha^{-1}) \right]^{I_{ij}},
\end{equation}
where $I_{ij}$ denotes an indicator variable which equals $1$ if users $i$ gave a rating for movie $j$. The priors for the users and movies are $p(\mathbf{U}|\mu_{\mathbf{U}},\Lambda_{\mathbf{U}}) = \prod_{i=1}^N \mathcal{N}(\mathbf{U}_i|\mu_{\mathbf{U}},\Lambda_{\mathbf{U}}^{-1})$ and $p(\mathbf{V}|\mu_{\mathbf{V}},\Lambda_{\mathbf{V}}) = \prod_{j=1}^M \mathcal{N}(\mathbf{V}_j|\mu_{\mathbf{U}},\Lambda_{\mathbf{U}}^{-1})$, with prior distributions on the hyper-parameters, for $\mathbf{W}=\mathbf{U}$ or $\mathbf{V}$, given by $\mu_{\mathbf{W}}\sim \mathcal{N}(\mu_{\mathbf{W}}|\mu_0,\Lambda_{\mathbf{W}})$ and $\Lambda_{\mathbf{W}}\sim \Gamma(a_0,b_0)$. The parameters of interest are then $\theta = (\mathbf{U},\mu_{\mathbf{U}}, \Lambda_{\mathbf{U}}, \mathbf{V}, \mu_{\mathbf{V}}, \Lambda_{\mathbf{V}})$. In our experiments, we use hyper-parameters $(\alpha, \mu_0, a_0, b_0) = (3,0,4,5)$, and set the latent dimension $d=20$.

We test our algorithm on the MovieLens dataset \cite{Harper2015}, which consists of 100,000 ratings, taking values \{1,2,3,4,5\}, for 1,682 movies from 943 users. The data are split into 80\% for training and 20\% for testing. We use $N=50$ particles; a batch size of 1000 for stochastic gradients; and average the results over $10$ random seeds. Our results are shown in Fig. \ref{fig:BayesPMF}, where we plot the RMSE for SVGD and Coin SVGD, as a function of the learning rate, after $T=1000$ and $T=2000$ iterations. We also compare against the stochastic gradient Langevin dynamics (SGLD) algorithm \cite{Welling2011}. In this case, Coin SVGD outperforms SVGD for almost all learning rates, and significantly outperforms SGLD.

\begin{figure}[t!]
\centering
\subfigure[$T=1000$.]{\includegraphics[trim=0 0mm 15mm 12mm, clip, width=.485\columnwidth]{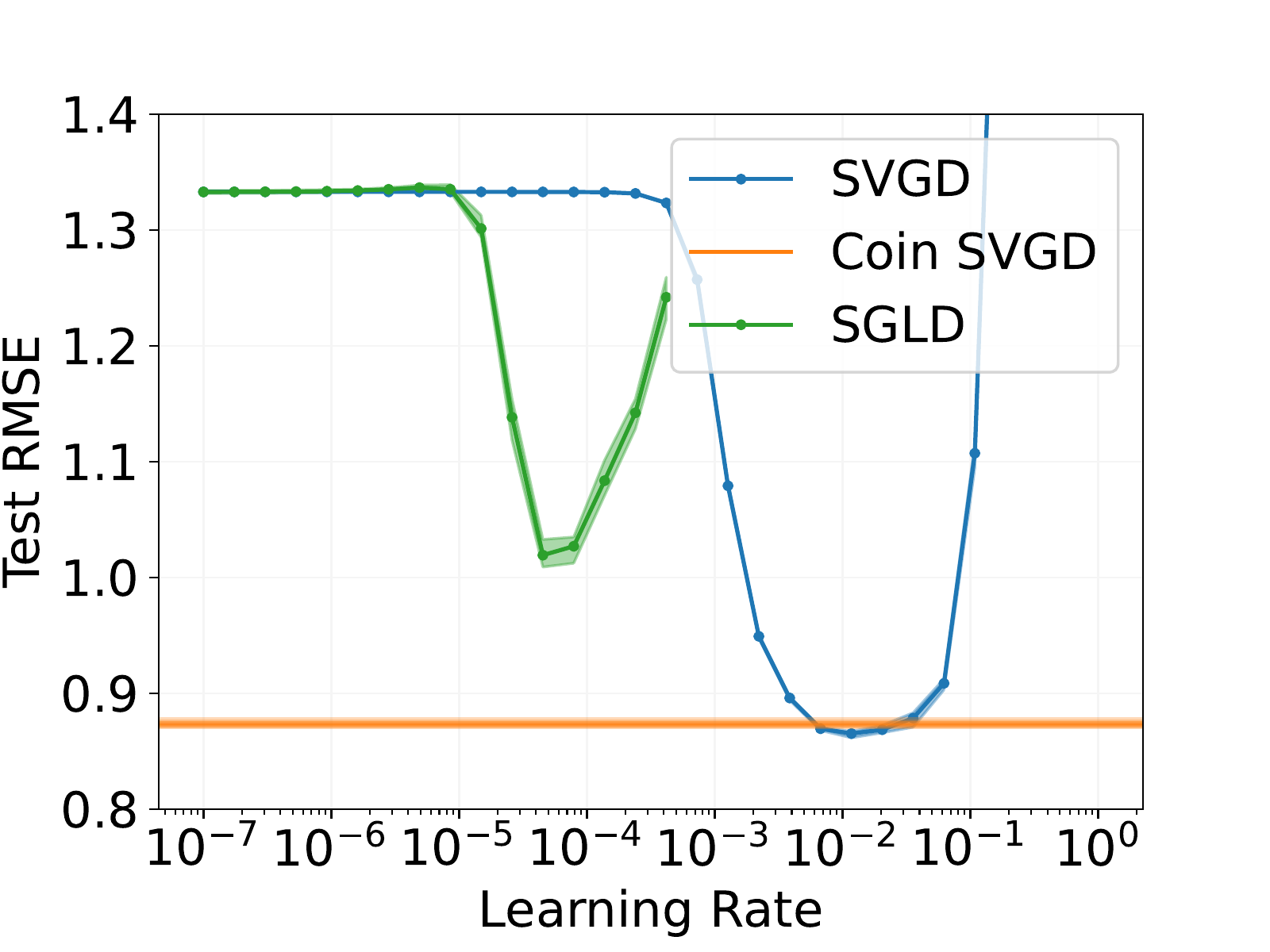}} \hfill
\subfigure[$T=2000$.]{\includegraphics[trim=0 0 15mm 12mm, clip, width=.485\columnwidth]{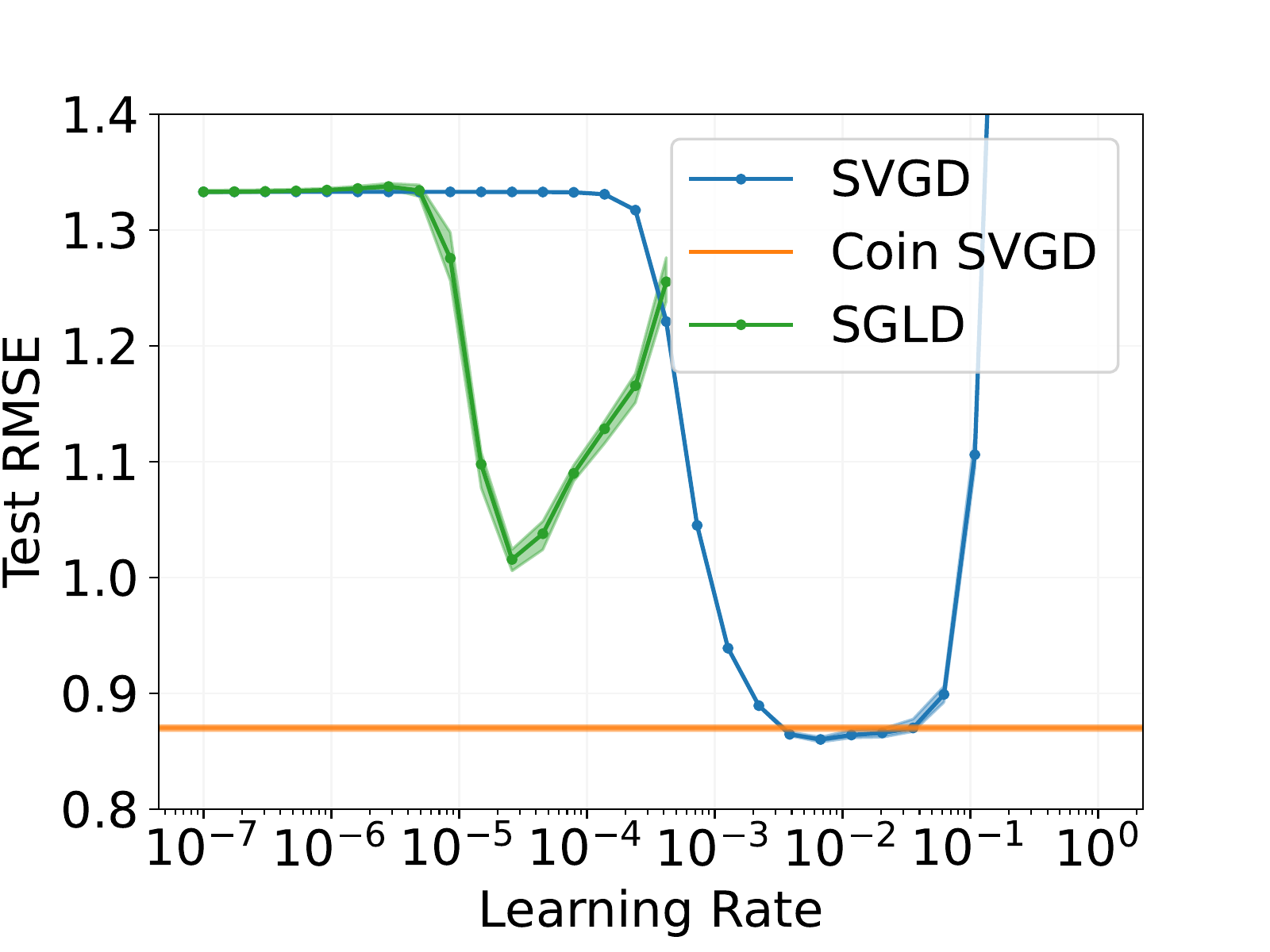}}
\vspace{-3.5mm}
\caption{\textbf{Results for the Bayesian probabilistic matrix factorisation model}. Test RMSE for Coin SVGD, SVGD, and SGLD, as a function of the learning rate, after $T\in\{1000,2000\}$ iterations.}
\label{fig:BayesPMF}
\vspace{-5.5mm}
\end{figure}

\section{Discussion}
In this paper, we introduced a suite of new algorithms for Bayesian inference which are entirely learning-rate free, inspired by coin betting techniques from convex optimisation. In empirical experiments, our coin sampling algorithms - most notably Coin SVGD - demonstrated comparable performance to their learning-rate dependent counterparts, with no need for any hyperparameter tuning.

We highlight several opportunities for future work. In terms of theory, the main open challenge is to establish the convergence of our algorithms under more easily verifiable assumptions. In this paper, we were able to obtain a rather technical condition which was sufficient for convergence. However, it remains unclear how to verify this condition in practice, even for relatively simple target distributions. In terms of methodology, a natural extension of this work is to apply a similar treatment to the many recent variants of SVGD \citep[e.g.][]{Detommaso2018,Wang2019a,Chen2020,Gong2021}.

\bibliographystyle{icml2023}
\bibliography{references}

\newpage
\appendix
\onecolumn

\section{Background}
\label{sec:wasserstein_details}

\subsection{Additional Properties of The Wasserstein Space}
One important property of the Wasserstein space $(\mathcal{P}_2(\mathbb{R}^d), W_2)$ is that, under appropriate regularity conditions, there exists a unique optimal coupling $\gamma_{*}\in\Gamma(\mu,\nu)$ which minimises the transport cost $\smash{\int_{\mathbb{R}^d\times\mathbb{R}^d}||x-y||^p \gamma_{*}(\mathrm{d}x,\mathrm{d}y)}$. This optimal coupling is of the form
\begin{equation}
\gamma = \smash{(\boldsymbol{\mathrm{id}}\times \boldsymbol{t}_{\mu}^{\nu})_{\#}\mu},
\end{equation}
where $\boldsymbol{\mathrm{id}}:\mathbb{R}^d\rightarrow \mathbb{R}^d$ is the identity map, and $\boldsymbol{t}_{\mu}^{\nu}$ is known as the optimal transport map \cite{Brenier1991,Gigli2011}. It follows that $(\boldsymbol{t}_{\mu}^{\nu})_{\#}\mu = \nu$ and \begin{equation}
    W_{2}^2(\mu,\nu) = \int_{\mathbb{R}^d}||x-y||^2\gamma_{*}(\mathrm{d}x,\mathrm{d}y)= \int_{\mathbb{R}^d} ||x-\boldsymbol{t}_{\mu}^{\nu}(x)||^2\mathrm{d}x.
\end{equation}

\subsection{Geodesic Convexity}
Let $\mu,\nu\in\mathcal{P}_2(\mathbb{R}^d)$. We define a constant speed geodesic between $\mu$ and $\nu$ as a curve $(\lambda^{\mu\rightarrow\nu}_\eta)_{\eta\in[0,1]}$ such that $\lambda_0 = \mu$, $\lambda_1 = \nu$, and $W_2(\lambda_{\iota},\lambda_{\eta}) = (\eta-\iota)W_2(\mu,\nu)$ for all $\iota,\eta\in[0,1]$. If $\boldsymbol{t}_{\mu}^{\nu}$ is the optimal transport map between $\mu$ and $\nu$, then a constant speed geodesic is given by \citep[e.g.,][Sec. 7.2]{Ambrosio2008}
\begin{equation}
\lambda_{\eta}^{\mu\rightarrow\nu} = \left((1-\eta)\boldsymbol{\mathrm{id}} + \eta \boldsymbol{t}_{\mu}^{\nu}\right)_{\#}\mu.
\end{equation}
Let $\mathcal{F}:\mathcal{P}_2(\mathbb{R}^d)\rightarrow(-\infty,\infty]$. The functional $\mathcal{F}$ is said to be lower semi-continuous if, for all $M\in\mathbb{R}$, $\{\mathcal{F}\leq M\}$ is a closed subset of $\mathcal{P}_2(\mathbb{R}^d)$. 
For $m\geq 0$, we say that $\mathcal{F}$ is $m$-geodesically convex if, for any $\mu,\nu\in\mathcal{P}_2(\mathbb{R}^d)$, there exists a constant speed geodesic $(\lambda_{\eta}^{\mu\rightarrow\nu})_{\eta\in[0,1]}$ between $\mu$ and $\nu$ such that, for all $\eta\in[0,1]$, 
\begin{equation}
\mathcal{F}(\lambda_{\eta}^{\mu\rightarrow\nu}) \leq (1-\eta)\mathcal{F}(\mu) + \eta\mathcal{F}(\nu) - \frac{m}{2}\eta(1-\eta)W_2^2(\mu,\nu).
\end{equation}
In the case that this inequality holds for $m=0$, we will simply say that $\mathcal{F}$ is geodesically convex.

\subsection{Subdifferential Calculus in the Wasserstein Space}
Let $\mu\in\mathcal{P}_2(\mathbb{R}^d)$, and let $\xi\in L^2(\mu)$. Let $\mathcal{F}$ be a proper and lower semi-continuous functional on $\mathcal{P}_2(\mathbb{R}^d)$. We say that $\xi\in L^2(\mu)$ belongs to the Fr\'{e}chet subdifferential of $\mathcal{F}$ at $\mu$, and write $\xi\in\partial\mathcal{F}(\mu)$ if, for any $\nu\in\mathcal{P}_2(\mathbb{R}^d)$, 
\begin{equation}
\liminf_{\nu\rightarrow\mu}\frac{\mathcal{F}(\nu) - \mathcal{F}(\mu) - \int_{\mathbb{R}^d}\langle \xi(x),\boldsymbol{t}_{\mu}^{\nu}(x) - x\rangle \mu(\mathrm{d}x)}{W_2(\nu,\mu)}\geq 0.
\end{equation}
Suppose, in addition, that $\mathcal{F}$ is $m$-geodesically convex. Then $\xi\in L^2(\mu)$ belongs to the Fr\'{e}chet subdifferential $\partial\mathcal{F}(\mu)$ if and only if, for all $\nu\in\mathcal{P}_2(\mathbb{R}^d)$, 
\begin{equation}
\mathcal{F}(\nu) - \mathcal{F}(\mu) \geq \int_{\mathbb{R}^d}\langle \xi(x), \boldsymbol{t}_{\mu}^{\nu}(x) -x \rangle\mu(\mathrm{d}x) + \frac{m}{2}W_2^2(\mu,\nu).
\end{equation}
For certain functionals $\mathcal{F}$, and under mild regularity conditions, \citep[see Lemma 10.4.13 in][]{Ambrosio2008}, one has that $\partial \mathcal{F}(\mu) = \{\nabla_{W_2}\mathcal{F}(\mu)\}$, where $\nabla_{W_2}\mathcal{F}(\mu)\in L^2(\mu)$ is given by
\begin{equation}
\nabla_{W_2}\mathcal{F}(\mu) = \nabla \frac{\partial \mathcal{F}(\mu)}{\partial \mu}(x)~~~\text{for $\mu$-a.e. $x\in\mathbb{R}^d$},
\end{equation}
and $\frac{\partial\mathcal{F}(\mu)}{\partial\mu}:\mathbb{R}^d\rightarrow\mathbb{R}$ denotes the first variation of $\mathcal{F}$ at $\mu$, that is, the unique function such that
\begin{equation}
\lim_{\varepsilon\rightarrow0}\frac{1}{\varepsilon}\left( \mathcal{F}(\mu + \varepsilon \zeta) - \mathcal{F}(\mu)\right) = \int_{\mathbb{R}^d} \frac{\partial \mathcal{F}(\mu)}{\partial \mu}(x)\zeta(\mathrm{d}x),
\end{equation}
where $\zeta = \nu - \mu$, and $\nu\in\mathcal{P}_2(\mathbb{R}^d)$. We will refer to $\nabla_{W_2}\mathcal{F}(\mu)$ as the Wasserstein gradient of $\mathcal{F}$ at $\mu$. 

\section{Theoretical Results}
\label{sec:proofs}

\label{sec:theoretica_results}

\subsection{Lemma \ref{orabona_lemma}}
\begin{proof}[Proof of Lemma \ref{orabona_lemma}] This result is well known; see, e.g., Lemma 1 in \citet{Orabona2016}; Theorem 9.6 in \citet{Orabona2022}; Sec. 4 in \citet{Orabona2017}; Part 2 in \citet{Orabona2020}. In particular, we have
\begin{align}
    f\left(\frac{1}{T}\sum_{t=1}^T x_t\right) - f(x^{*}) &\leq \frac{1}{T}\sum_{t=1}^T \bigg(f(x_t) - f(x^{*})\bigg) \tag{Jensen's inequality} \\
    &\leq \frac{1}{T} \left(\sum_{t=1}^T c_tx^{*} - \sum_{t=1}^T c_tx_t\right) \tag{convexity} \\
    &\leq \frac{1}{T}\left(\left(\sum_{t=1}^T c_t\right)x^{*} - h\left(\sum_{t=1}^T c_t\right) + \varepsilon\right) \tag{definition of $h(\cdot)$}  \\[2mm]
    &\leq \frac{1}{T}\left(\max_{v} \left[ vx^{*} - h\left(v\right)\right] + \varepsilon\right) \tag{maximum over $v=\sum_{t=1}^T c_t$}  \\[3mm]
    &= \frac{h^{*}(x^{*}) + \varepsilon}{T}. \tag{definition of $h^*(\cdot)$} \label{eq:opt_convergence_rate}
\end{align}
\end{proof}

\subsection{Proposition \ref{theorem:convergence_rate}}
In this section, we outline how to prove Proposition \ref{theorem:convergence_rate}. Our proof of this result will rely on a rather strong sufficient condition, which we provide below (Assumption \ref{sufficient_condition_2}). Before we state this assumption, we will first require some additional notation.

First, throughout this section, we will use $\smash{(\mu_t)_{t\in[T]}}$ to denote the sequence of betting measures defined by Alg. \ref{alg:param_free_grad_descent_v2}, and $(\varphi_t)_{t\in[T]}$ to denote the transport maps from $\mu_0$ to $(\mu_t)_{t\in[T]}$ defined by Alg. \ref{alg:param_free_grad_descent_v2}. In addition, we will write  $\smash{(t_{\mu_t}^{\pi})_{t\in[T]}}$ for the optimal transport maps from $(\mu_t)_{t\in[T]} \mapsto \pi$, and $\smash{(t^{\mu_t}_{\pi})_{t\in[T]}}$ for the optimal transport maps from $\smash{\pi \mapsto (\mu_t)_{t\in[T]}}$. Finally, we let $\smash{(\tilde{t}_{\pi,t}^{\mu_t})_{t\in[T]}}$ denote the transport maps from $\pi$ to $(\mu_t)_{t\in[T]}$ defined according to $\smash{\tilde{t}^{\mu_t}_{\pi,t} := \varphi_t \circ t_{\pi}^{\mu_0}}$, and $\smash{(\tilde{t}_{\mu_0,t}^{\pi})_{t\in[T]}}$ the transport maps from $\mu_0$ to $\pi$ defined according to $\smash{\tilde{t}_{\mu_0,t}^{\pi} = t_{\mu_t}^{\pi} \circ \varphi_t}$. With this notation at hand, we are now ready to introduce the sufficient condition required for our convergence result. 
\vspace{2mm}

\begin{assumption} \label{sufficient_condition_2}
  Define the functions $v:\mathbb{R}^d\rightarrow\mathbb{R}^d$ and $\tilde{v}:\mathbb{R}^d\rightarrow\mathbb{R}^d$ according to
\begin{align}
    v(x) &= \sum_{t=1}^T -\nabla_{W_2} \mathcal{F}(\mu_t)(t_{\pi}^{\mu_t}(x))~,\quad\tilde{v}(x)= \sum_{t=1}^T -\nabla_{W_2} \mathcal{F}(\mu_t)(\tilde{t}_{\pi,t}^{\mu_t}(x)).
\end{align}
Then there exists a constant $K>0$ such that, for all $x\in\mathbb{R}^d$,  
\begin{equation}
     \frac{1}{2L^2T} \left[||v(x)||^2 - ||\tilde{v}(x)||^2\right]  \leq \ln K .
\end{equation}
\end{assumption}

We can now proceed to the proof Proposition \ref{theorem:convergence_rate}. For convenience, we first recall the original statement of this result.
\vspace{2mm}

\begin{customprop}{3.3} \label{theorem:convergence_rate_app}
Let Assumptions \ref{assumption1} - \ref{assumption2} and Assumption \ref{sufficient_condition_2} hold. Then 
    \begin{align}
    \mathcal{F}\left(\frac{1}{T}\sum_{t=1}^T \mu_t\right) - \mathcal{F}(\pi)
    \leq \frac{L}{T} \bigg[ w_0 &+\int_{\mathbb{R}^d} \left|\left|x\right|\right| \sqrt{T \ln \left(1+ \frac{96K^2T^2\left|\left|x\right|\right|^2}{w_0^2}\right)} \pi(\mathrm{d}x) \\
    &+\int_{\mathbb{R}^d} \left|\left|x\right|\right| \sqrt{T \ln \left(1+ \frac{96T^2\left|\left|x\right|\right|^2}{w_0^2}\right)} \mu_0(\mathrm{d}x)  \bigg],  \label{eq:theorem1_bound_v2}
\end{align}
where $K>0$ is the constant defined in Assumption \ref{sufficient_condition_2}.
\end{customprop}

\begin{proof}[Proof of Proposition \ref{theorem:convergence_rate}]
Our proof begins in much the same fashion as the proof of Lemma \ref{orabona_lemma}. On this occasion, we consider
\begin{align}
\mathcal{F}\left(\frac{1}{T}\sum_{t=1}^T \mu_t\right) - \mathcal{F}(\pi) &\leq \frac{1}{T}\sum_{t=1}^T \mathcal{F}\left( \mu_t\right) - \mathcal{F}(\pi) \label{eq:38v2} \\
&\leq \frac{1}{T} \sum_{t=1}^T \int_{\mathbb{R}^d} \langle -\nabla_{W_2} \mathcal{F}(\mu_t)(x), t_{\mu_t}^{\pi}(x) - x\rangle\mu_t (\mathrm{d}x) \label{eq:line2} \\
&\leq \frac{L}{T} \sum_{t=1}^T \bigg[ \int_{\mathbb{R}^d} \big\langle -\nabla_{W_2} \hat{\mathcal{F}}(\mu_t)(t_{\pi}^{\mu_t}(x)), x \big\rangle  \pi(\mathrm{d}x) - \int_{\mathbb{R}^d} \big\langle -\nabla_{W_2} \hat{\mathcal{F}}(\mu_t)(\varphi_t(x)), \varphi_t(x) \big\rangle \mu_0(\mathrm{d}x)\bigg]  \label{eq:line3} \\
&= \frac{L}{T}\bigg[  \int_{\mathbb{R}^d} \big\langle \sum_{t=1}^T-\nabla_{W_2} \hat{\mathcal{F}}(\mu_t)(t_{\pi}^{\mu_t}(x)), x \big\rangle\pi(\mathrm{d}x) - \int_{\mathbb{R}^d} \sum_{t=1}^T \big\langle -\nabla_{W_2} \hat{\mathcal{F}}(\mu_t)(\varphi_t(x)), x \big\rangle \mu_0(\mathrm{d}x) \nonumber \\
&-  \int_{\mathbb{R}^d} \sum_{t=1}^T \big\langle -\nabla_{W_2} \hat{\mathcal{F}}(\mu_t)(\varphi_t(x)), \varphi_t(x) \mu_0(\mathrm{d}x) + \int_{\mathbb{R}^d} \sum_{t=1}^T \big\langle -\nabla_{W_2} \hat{\mathcal{F}}(\mu_t)(\varphi_t(x)), x \big\rangle \mu_0(\mathrm{d}x) \bigg] \label{eq:42v3} \\
&= \frac{L}{T}  \bigg[ \int_{\mathbb{R}^d} \big\langle \sum_{t=1}^T-\nabla_{W_2} \hat{\mathcal{F}}(\mu_t)(t_{\pi}^{\mu_t}(x)), x \big\rangle  - \big\langle \sum_{t=1}^T-\nabla_{W_2} \hat{\mathcal{F}}(\mu_t)(\tilde{t}_{\pi,t}^{\mu_t}(x)), t_{\pi}^{\mu_0}(x) \big\rangle \pi(\mathrm{d}x) \nonumber \\[-1mm]
&\hspace{5mm}-  \int_{\mathbb{R}^d} \sum_{t=1}^T \big\langle -\nabla_{W_2} \hat{\mathcal{F}}(\mu_t)(\varphi_t(x)), \varphi_t(x) - x\big\rangle \mu_0(\mathrm{d}x)\bigg],  \label{eq:42v4}
\end{align}
where in \eqref{eq:38v2} we have used Jensen's inequality, in \eqref{eq:line2} we have used the definition of geodesic convexity (see App. \ref{sec:wasserstein_details}), in \eqref{eq:line3} we have substituted $\smash{x\mapsto t_{\pi}^{\mu_t}(x)}$ and $\smash{x\mapsto \varphi_t(x)}$ in the first and second integrals, respectively, used the fact that, by definition,  $\smash{(t_{\pi}^{\mu_t})_{\#}\pi = \mu_t}$ and $\smash{(\varphi_t)_{\#}\mu_0 = \mu_t}$ (see Alg. \ref{alg:param_free_grad_descent_v2}), and introduced the notation $\smash{\hat{\mathcal{F}} = \frac{1}{L}\mathcal{F}}$; in \eqref{eq:42v3} we have added and subtracted the same integral; and in \eqref{eq:42v4} we have substituted $\smash{x\mapsto t_{\pi}^{\mu_0}(x)}$, used the fact that $\smash{\tilde{t}_{\pi,t}^{\mu_t}(x) = \varphi_t(t_{\pi}^{\mu_0}(x))}$, and the fact that $\smash{(t_{\pi}^{\mu_0})_{\#}\pi = \mu_0}$ in the second integral; and combined the third and fourth integrals.

By construction, the betting strategy in Alg. \ref{alg:param_free_grad_descent_v2} guarantees that, for any $x\in\mathbb{R}^{d}$, and any arbitrary sequence $c_1(x),\dots,c_T(x)\in\mathbb{R}^d$, such that $||c_t(x)||\leq 1$,  there exists an even, logarithmically convex function $h:\mathbb{R}\rightarrow\mathbb{R}_{+}$ such that the wealth is lower bounded as \citep[][Proof of Theorem 3, App. C]{Orabona2016}
\begin{align}
w_T(x) &= w_0 + \sum_{t=1}^T \langle c_t(x), \varphi_t(x) - x \rangle  \geq h
\left( \left|\left| \sum_{t=1}^T c_t(x) \right|\right|\right). \label{eq:wealth_inequalityv2}
\end{align} 
In particular, the betting strategy in Alg. \ref{alg:param_free_grad_descent_v2} guarantees that this inequality holds with \citep[][App. F.1, Proof of Corollary 5]{Orabona2016}
\begin{align}
h\left(u\right) &= w_0 \frac{2^T\Gamma(1)\Gamma(\frac{T+1}{2}+\frac{u}{2})\cdot \Gamma(\frac{T+1}{2}-\frac{u}{2})}{\Gamma^{2}(\frac{1}{2})\Gamma(T+1)}. \label{eq:h_funcv2}
\end{align}
Due to Lemma 16 in \citet{Orabona2016}, we also have that 
\begin{equation}
h(u) \geq i(u) := \frac{w_0}{K_1\sqrt{T}} \exp \left( \frac{ u^2}{2T}\right), \label{eq:i_funcv2}
\end{equation}
where $K_1=e\sqrt{\pi}$ is a universal constant. We will apply the inequality in \eqref{eq:wealth_inequalityv2} for the sequence $c_t(x) = -\nabla_{W_2}\hat{\mathcal{F}}(\mu_t)(\varphi_t(x))$. In particular, substituting this sequence into \eqref{eq:wealth_inequalityv2}, and using also the inequality in \eqref{eq:i_funcv2}, we have that 
 \begin{align}
     w_0 + \sum_{t=1}^T \left\langle -\nabla_{W_2}\hat{\mathcal{F}}(\mu_t)(\varphi_t(x)), \varphi_t(x) - x\right\rangle  \geq i\left(\left|\left|\sum_{t=1}^T  -\nabla_{W_2}\hat{\mathcal{F}}(\mu_t)(\varphi_t(x)) \right|\right|\right). \label{eq:45v2}
 \end{align}
Suppose now that we define the function $I:\mathbb{R}^d\rightarrow(-\infty,\infty]$ according to $I(u)=i(||u||)$. Thus, in particular, 
\begin{equation}
    I(u) = \frac{w_0}{K_1\sqrt{T}} \exp \left( \frac{ ||u||^2}{2T}\right). \label{eq:bigI_def}
\end{equation}
By using this definition in \eqref{eq:45v2}, and then substituting \eqref{eq:45v2} into \eqref{eq:38v2} - \eqref{eq:42v4}, we then have
\begin{align}
&\mathcal{F}\left(\frac{1}{T}\sum_{t=1}^T \mu_t\right) - \mathcal{F}(\pi) \leq \frac{L}{T} \left[ \sum_{t=1}^T  \int_{\mathbb{R}^d} \left \langle -\nabla_{W_2} \hat{\mathcal{F}}(\mu_t)(t_{\pi}^{\mu_t}(x)), x \right\rangle -  \left\langle - \nabla_{W_2} \hat{\mathcal{F}}(\mu_t)(\tilde{t}_{\pi,t}^{\mu_t}(x)), t_{\pi}^{\mu_0}(x) \right\rangle\pi (\mathrm{d}x)\right. \nonumber \\
&\left.\hspace{45mm}- \int_{\mathbb{R}^d} I \left(\sum_{t=1}^T -\nabla_{W_2} \hat{\mathcal{F}}(\mu_t)(\varphi_t(x))\right) \mu_0 (\mathrm{d}x) + w_0\right] \label{eq:32v2} \\
&\hphantom{\mathcal{F}\left(\frac{1}{T}\sum_{t=1}^T \mu_t\right) - \mathcal{F}(\pi)}= \frac{L}{T} \left[  \int_{\mathbb{R}^d} \left \langle \sum_{t=1}^T - \nabla_{W_2} \hat{\mathcal{F}}(\mu_t)(t_{\pi}^{\mu_t}(x)), x \right\rangle - \left\langle \sum_{t=1}^T - \nabla_{W_2} \hat{\mathcal{F}}(\mu_t)(\tilde{t}_{\pi,t}^{\mu_t}(x)), t_{\pi}^{\mu_0}(x) \right\rangle \pi (\mathrm{d}x)\right. \nonumber \\
&\left.\hspace{45mm}- \int_{\mathbb{R}^d} I \left(\sum_{t=1}^T -\nabla_{W_2} \mathcal{\hat{F}}(\mu_t)(\tilde{t}_{\pi,t}^{\mu_t}(x))\right) \pi (\mathrm{d}x) + w_0\right], \label{eq:new_line2_v0}
\end{align}
where, in \eqref{eq:new_line2_v0}, we have substituted $x\mapsto t_{\pi}^{\mu_0}(x)$ in the second integral, used the fact that $\smash{\tilde{t}_{\pi,t}^{\mu_t}(x) = \varphi_t(t_{\pi}^{\mu_0}(x))}$, and finally the fact that $(t_\pi^{\mu_0})_{\#}\pi = \mu_0$. Suppose we also now define
\begin{equation}
u(x) = \sum_{t=1}^T -\nabla_{W_2} \mathcal{\hat{F}}(\mu_t)({t}_{\pi}^{\mu_t}(x))~~~,~~~\tilde{u}(x) = \sum_{t=1}^T -\nabla_{W_2} \mathcal{\hat{F}}(\mu_t)(\tilde{t}_{\pi,t}^{\mu_t}(x))~~~,~~~A = \frac{\int_{\mathbb{R}^d} I\left(\tilde{u}(x)\right) \pi(\mathrm{d}x)}{\int_{\mathbb{R}^d} I\left(u(x))\right)\pi(\mathrm{d}x)}.  \label{eq:constantA}
\end{equation}
Using this notation, and re-ordering terms, we can now rewrite the previous inequality as 
\begin{align}
&\mathcal{F}\left(\frac{1}{T}\sum_{t=1}^T \mu_t\right) - \mathcal{F}(\pi) 
\leq  \frac{L}{T} \bigg[w_0 + \int_{\mathbb{R}^d} \left \langle u(x), x \right\rangle \pi (\mathrm{d}x) + \int_{\mathbb{R}^d} \left\langle \tilde{u}(x), - t_{\pi}^{\mu_0}(x)\right\rangle \pi(\mathrm{d}x)- \int_{\mathbb{R}^d} I\left(\tilde{u}(x)\right) \pi (\mathrm{d}x)\bigg] \\
&\hphantom{\mathcal{F}\left(\frac{1}{T}\sum_{t=1}^T \mu_t\right) - \mathcal{F}(\pi)}\leq  \frac{L}{T} \bigg[ w_0 +  \int_{\mathbb{R}^d} \left \langle u(x), x \right\rangle -\frac{1}{2} I(\tilde{u}(x)) \pi (\mathrm{d}x) + \int_{\mathbb{R}^d} \left\langle \tilde{u}(x), - t_{\pi}^{\mu_0}(x)\right\rangle - \frac{1}{2}I(\tilde{u}(x)) \pi(\mathrm{d}x)\bigg] \label{eq:46v2} \\
&\hphantom{\mathcal{F}\left(\frac{1}{T}\sum_{t=1}^T \mu_t\right) - \mathcal{F}(\pi)}=\frac{L}{T} \bigg[w_0+ \int_{\mathbb{R}^d} A \left( \langle u(x), \frac{x}{A}\rangle - \frac{1}{2} I(u(x)) \right) \pi(\mathrm{d}x) \nonumber \\
&\hspace{48mm}+ \int_{\mathbb{R}^d} \left(\langle \tilde{u}(x), - t_{\pi}^{\mu_0}(x)\rangle - \frac{1}{2} I(\tilde{u}(x)) \right) \pi(\mathrm{d}x) \bigg]. \label{eq:47v2}
\end{align}
Suppose that we now fix $x\in\mathbb{R}^d$, and write $\theta=u(x)$, $\eta = \tilde{u}(x)$, $F=\frac{1}{2}I$, $x^{*} = \frac{x}{A}$, and $x^{\dagger} = -t_{\pi}^{\mu_0}(x)$. Using this notation, for each $x\in\mathbb{R}^d$, we can now rewrite the first and second integrands in \eqref{eq:47v2} as 
\begin{align*}
\langle u(x), \frac{x}{A}\rangle - \frac{1}{2}I(u(x)) &:= \langle \theta, x^{*} \rangle - F(\theta), \\
\langle \tilde{u}(x), - t_{\pi}^{\mu_0}(x)\rangle - \frac{1}{2}I(\tilde{u}(x)) &:= \langle \eta, x^{\dagger} \rangle - F(\eta).
\end{align*}
Taking the supremum over $\theta\in\mathbb{R}^d$ and $\eta\in\mathbb{R}^d$, respectively, and using the definition of the convex conjugate, we can easily upper bound these expressions by
\begin{align}
\langle \theta, x^{*} \rangle - F(\theta) &\leq \sup_{\theta\in\mathbb{R}^d} \left(\langle \theta, x^{*} \rangle - F(\theta)\right) \leq F^{*}(x^{*}), \\
\langle \eta, x^{\dagger} \rangle - F(\eta) &\leq \sup_{\eta\in\mathbb{R}^d} \left(\langle \eta, x^{\dagger} \rangle - F(\eta)\right) \leq F^{*}(x^{\dagger}),
\end{align}
where, as elsewhere, $F^{*}$ denotes the Fenchel conjugate of $F$. Returning to our previous notation, and using the fact that $x\in\mathbb{R}^d$ was chosen arbitrarily, we thus have that
\begin{align}
    \langle u(x), \frac{x}{A}\rangle - \frac{1}{2}I(u(x)) &\leq (\frac{1}{2}I)^{*}\left(\frac{x}{A}\right) = \frac{1}{2} I^{*}\left(\frac{2x}{A}\right), \label{eq:fenchel_bound1} \\
    \langle \tilde{u}(x), -t_{\pi}^{\mu_0}(x)\rangle - \frac{1}{2} I(\tilde{u}(x)) &\leq (\frac{1}{2}I)^{*}\left(- t_{\pi}^{\mu_0}(x)\right)= \frac{1}{2}I^{*}\left(- 2t_{\pi}^{\mu_0}(x)\right), \label{eq:fenchel_bound2}
\end{align}
for all $x\in\mathbb{R}^d$, where in both lines we have used the fact that $(af)^{*}(x^{*})=af^{*}(\frac{x^{*}}{a})$, with $a=\frac{1}{2}$. Substituting \eqref{eq:fenchel_bound1} and \eqref{eq:fenchel_bound2} into \eqref{eq:47v2}, it now follows straightforwardly that
\begin{align}
\mathcal{F}\left(\frac{1}{T}\sum_{t=1}^T \mu_t\right) - \mathcal{F}(\pi) &\leq  \frac{L}{T} \left[ w_0 + \int_{\mathbb{R}^d} \frac{1}{2}AI^{*}\left(\frac{2x}{A}\right)\pi(\mathrm{d}x) + \int_{\mathbb{R}^d} \frac{1}{2}I^{*}\left(-2t_{\pi}^{\mu_0}(x)\right)\pi(\mathrm{d}x) \right] \\
&=\frac{L}{T} \left[ w_0 + \int_{\mathbb{R}^d} \frac{1}{2}AI^{*}\left(\frac{2 t_{\mu_0}^{\pi}(x)}{A}\right)\mu_0(\mathrm{d}x) + \int_{\mathbb{R}^d} \frac{1}{2}I^{*}\left(-\frac{1}{2}x\right)\mu_0(\mathrm{d}x) \right] \\
&=  \frac{L}{T} \left[w_0 +  \int_{\mathbb{R}^d}\frac{1}{2} Ai^{*}\left(\frac{2\left|\left|t_{\mu_0}^{\pi}(x)\right|\right|}{A}\right)\mu_0(\mathrm{d}x) + \int_{\mathbb{R}^d} \frac{1}{2}i^{*}\left(2\left|\left|x\right|\right|\right)\mu_0(\mathrm{d}x) \right], \label{eq:53v2}
\end{align}
where in the second line we have substituted $\smash{x \mapsto t_{\mu_0}^{\pi}(x)}$ into both integrals, and used the fact that $\smash{(t_{\mu_0}^{\pi})_{\#}\mu_0 = \pi}$; and, in the final line, we have used the fact that the Fenchel conjugate of $i(||\cdot||)$ is $i^{*}(||\cdot||)$ since $i^{*}$ is an even function \citep[][Example 13.7]{Bauschke2011}. Now, by Lemma 18 in \citet{Orabona2016}, we have the upper bound
\begin{equation}
    i^{*}(u) \leq |u| \sqrt{T\ln \left(1+ \frac{24T^2u^2}{w_0^2}\right)} - \frac{w_0}{K_1\sqrt{T}}.
\end{equation}
Substituting this into our previous bound \eqref{eq:53v2}, we finally arrive at 
\begin{align}
    \mathcal{F}\left(\frac{1}{T}\sum_{t=1}^T \mu_t\right) - \mathcal{F}(\pi)
    &\leq \frac{L}{T} \bigg[ w_0\left(1 - \frac{A}{2K_1\sqrt{T}}\right) +  \int_{\mathbb{R}^d} \frac{1}{2}A\frac{2||t_{\mu_0}^{\pi}(x)||}{A} \sqrt{T \ln \left(1+ \frac{24T^2||2t_{\mu_0}^{\pi}(x)||^2}{A^2w_0^2}\right)} \mu_0(\mathrm{d}x) \nonumber \\
    &\hspace{37mm}+ \int_{\mathbb{R}^d} \frac{1}{2} 2||x|| \sqrt{T \ln \left(1+ \frac{24T^2||2x||^2}{A^2w_0^2}\right)} \mu_0(\mathrm{d}x) \bigg]  \\[3mm]
    &\leq \frac{L}{T} \bigg[w_0 +  \int_{\mathbb{R}^d} ||x|| \sqrt{T \ln \left(1+ \frac{96T^2||x||^2}{A^2 w_0^2}\right)} \pi(\mathrm{d}x) \nonumber \\
    &\hspace{15mm}+ \int_{\mathbb{R}^d} ||x|| \sqrt{T \ln \left(1+ \frac{96T^2||x||^2}{w_0^2}\right)} \mu_0(\mathrm{d}x) \bigg], \label{eq:penultimate_bound_v2} 
\end{align}
where, in the second line, we have substituted $x\mapsto t_{\pi}^{\mu_0}(x)$ in the first integral, and used the fact that $(t_{\pi}^{\mu_0})_{\#}\pi = \mu_0$.
It remains to bound the constant ${A}$ from below, or, equivalently, the constant $A^{-1}$ from above. The required bound will follow directly from our sufficient condition (Assumption \ref{sufficient_condition_2}). Indeed, from Assumption \ref{sufficient_condition_2}, we have that
\begin{equation}
    \frac{1}{2L^2T} \left[||v(x)||^2 - ||\tilde{v}(x)||^2\right]  \leq \ln K \implies \exp\left[ \frac{||v(x)||^2}{2L^2T} \right] \leq K \exp\left[ \frac{||\tilde{v}(x)||^2}{2L^2T}\right].
\end{equation}
Using this bound, and the definition of $A$ in \eqref{eq:constantA}, it then follows straightforwardly that
\begin{align}
A^{-1}
= \frac{\displaystyle{ \int_{\mathbb{R}^d}\frac{w_0}{K_1\sqrt{T}}\exp \left[\frac{||{u}(x)||^2}{2T}\right]}\pi(\mathrm{d}x)}{\displaystyle{ \int_{\mathbb{R}^d}\frac{w_0}{K_1\sqrt{T}}\exp\left[\frac{ ||\tilde{u}(x)||^2}{2T}\right]\pi(\mathrm{d}x)}}  &= \frac{\displaystyle{ \int_{\mathbb{R}^d}\exp \left[\frac{||{v}(x)||^2}{2L^2T}\right]}\pi(\mathrm{d}x)}{\displaystyle{ \int_{\mathbb{R}^d}\exp\left[\frac{ ||\tilde{v}(x)||^2}{2L^2T}\right]\pi(\mathrm{d}x)}}
\leq \frac{\displaystyle{ \int_{\mathbb{R}^d}K\exp \left[\frac{||\tilde{v}(x)||^2}{2L^2T}\right]}\pi(\mathrm{d}x)}{\displaystyle{ \int_{\mathbb{R}^d}\exp\left[\frac{ ||\tilde{v}(x)||^2}{2L^2T}\right]\pi(\mathrm{d}x)}} = K.
\label{eq:constantA_bound}
\end{align}
Finally, substituting the bound in \eqref{eq:constantA_bound} into \eqref{eq:penultimate_bound_v2}, we arrive at
\begin{align}
    \mathcal{F}\left(\frac{1}{T}\sum_{t=1}^T \mu_t\right) - \mathcal{F}(\pi)
    \leq \frac{L}{T} \bigg[w_0 &+  \int_{\mathbb{R}^d}  \left|\left|x\right|\right| \sqrt{T \ln \left(1+ \frac{96K^2T^2\left|\left|x\right|\right|^2}{w_0^2}\right)} \pi(\mathrm{d}x) \nonumber \\
    &+\int_{\mathbb{R}^d}  \left|\left|x\right|\right| \sqrt{T \ln \left(1+ \frac{96T^2\left|\left|x\right|\right|^2}{w_0^2}\right)} \mu_0(\mathrm{d}x) \bigg]. \label{eq:final_bound_v2}
\end{align}
\end{proof}

\subsection{An Alternative to Proposition \ref{theorem:convergence_rate}}

In this section, we outline how to obtain an alternative bound to the one given in Proposition \ref{theorem:convergence_rate}, under a slightly different sufficient condition. 

\begin{customassumption}{B.1'} \label{sufficient_condition_1}
There exists a transport map $T_{\mu_0}^{\pi}:\mathbb{R}^d\rightarrow\mathbb{R}^d$ from $\mu_0$ to $\pi$, which does not depend on $t\in[T]$, such that, for some $K>0$,
\begin{align}
\sum_{t=1}^T \left|\left|t_{\mu_t}^{\pi}\circ \varphi_t - T_{\mu_0}^{\pi}\right|\right|_{L^2(\mu_0)} \leq K\sqrt{T}\ln\left[\mathrm{poly}(T)\right]. \label{K1}
\end{align}
\end{customassumption}

\begin{remark}
It is, at present, unclear how to establish the existence of a transport map $T_{\mu_0}^{\pi}$ which satisfies \eqref{K1}, aside from in one very simple case (see Sec. \ref{sec:gaussian}). In the general case, one possible candidate for $T_{\mu_0}^{\pi}$ is the optimal transport map $t_{\mu_0}^{\pi}$ from $\mu_0$ to $\pi$, so that $T_{\mu_0}^{\pi}(x) = t_{\mu_0}^{\pi}(x)$. In this case, \eqref{K1} reads 
\begin{equation}
     \sum_{t=1}^T || {t}_{\mu_t}^{\pi} \circ \varphi_t - t_{\mu_0}^{\pi}||_{L^2(\mu_0)} \leq K\sqrt{T}\ln\left[\mathrm{poly}(T)\right], 
\end{equation}
which can be interpreted as a bound on the sum of the distances between the optimal transport map $t_{\mu_0}^{\pi}$ from $\mu_0$ to $\pi$, and the maps $(t_{\mu_t}^{\pi}\circ\varphi_t)_{t\in[T]}$ which first transport $\mu_0$ to $(\mu_t)_{t\in[T]}$ according to the transport maps $(\varphi_t)_{t\in[T]}$ defined by Alg. \ref{alg:param_free_grad_descent_v2}, and then map $(\mu_t)_{t\in[T]}$ to $\pi$ via the optimal transport maps $(t_{\mu_t}^{\pi})_{t\in[T]}$. 

Another possible candidate for $T_{\mu_0}^{\pi}$ is the average, over all iterations, of the composition of the optimal transport maps $(t_{\mu_t}^{\pi})_{t\in[T]}$ from $(\mu_t)_{t\in[T]}$ to $\pi$, and the maps $(\varphi_t)_{t\in[T]}$ from $\mu_0$ to $(\mu_t)_{t\in[T]}$ defined by Alg. \ref{alg:param_free_grad_descent_v2}, viz $\smash{T_{\mu_0}^{\pi}(x) = \frac{1}{T}\sum_{t=1}^T t_{\mu_t}^{\pi}(\varphi_t(x))}$. In this case, \eqref{K1} becomes
\begin{equation}
      \sum_{t=1}^T ||t_{\mu_t}^{\pi} \circ \varphi_t - \frac{1}{T}\sum_{s=1}^T t_{\mu_s}^{\pi}\circ \varphi_s||_{L^2(\mu_0)}\leq K\sqrt{T}\ln\left[\mathrm{poly}(T)\right].
\end{equation}
\end{remark}
\vspace{1mm}

\begin{customprop}{3.3'} \label{theorem2}
Let Assumptions \ref{assumption1} - \ref{assumption2} and Assumption \ref{sufficient_condition_1} hold. Then
\begin{align}
    &\mathcal{F}\left(\frac{1}{T}\sum_{t=1}^T \mu_t\right) - \mathcal{F}(\pi) \\
    &\leq \frac{L}{T} \bigg[ w_0 + \int_{\mathbb{R}^d} \left|\left|T_{\mu_0}^{\pi}(x) - x\right|\right| \sqrt{T \ln \left(1+ \frac{24T^2\left|\left|T_{\mu_0}^{\pi}(x) - x\right|\right|^2}{w_0^2}\right)} \mu_0(\mathrm{d}x) + K\sqrt{T}\ln\left[\mathrm{poly}(T)\right] \bigg].  \nonumber
\end{align}
where $K>0$ and $T_{\mu_0}^{\pi}:\mathbb{R}^d\rightarrow\mathbb{R}^d$ are defined in Assumption \ref{sufficient_condition_1}.
\end{customprop}

\begin{proof}[Proof of Proposition \ref{theorem2}] 
Once again, our proof begins in a similar fashion to the proof of Lemma \ref{orabona_lemma}. In this case, we now have
\begin{align}
\mathcal{F}\left(\frac{1}{T}\sum_{t=1}^T \mu_t\right) - \mathcal{F}(\pi) 
&\leq \frac{1}{T}\sum_{t=1}^T \mathcal{F}\left( \mu_t\right) - \mathcal{F}(\pi) \label{eq:38} \\
&\leq \frac{1}{T} \sum_{t=1}^T \int_{\mathbb{R}^d} \langle -\nabla_{W_2} \mathcal{F}(\mu_t)(x), t_{\mu_t}^{\pi}(x) - x\rangle\mu_t (\mathrm{d}x) \label{lin2} \\
&= \frac{L}{T} \sum_{t=1}^T \int_{\mathbb{R}^d} \langle -\nabla_{W_2} \mathcal{\hat{F}}(\mu_t)(\varphi_t(x)), t_{\mu_t}^{\pi}(\varphi_t(x)) - \varphi_t(x)\rangle\mu_0 (\mathrm{d}x) \label{lin3} \\
&= \frac{L}{T} \bigg[ \sum_{t=1}^T \int_{\mathbb{R}^d} \left\langle -\nabla_{W_2} \mathcal{\hat{F}}(\mu_t)(\varphi_t(x)), t_{\mu_t}^{\pi}(\varphi_t(x)) - x\right\rangle \mu_0(\mathrm{d}x)  \nonumber \\
&\hspace{20mm}- \int_{\mathbb{R}^d}  \sum_{t=1}^T \left\langle  -\nabla_{W_2} \mathcal{\hat{F}}(\mu_t)(\varphi_t(x)), \varphi_t(x) -  x\right \rangle \mu_0(\mathrm{d}x)\bigg], \label{eq:42}
\end{align}
where, in \eqref{eq:38} we have used Jensen's inequality, in \eqref{lin2} we have used the definition of geodesic convexity (see App. \ref{sec:wasserstein_details}), in \eqref{lin3} we have substituted $x\mapsto \varphi_{t}(x)$, and used the fact that, by definition, $(\varphi_t)_{\#}\mu_0 = \mu_t$, and in \eqref{eq:42} we have added and subtracted the same integral. Following the same argument as in \eqref{eq:wealth_inequalityv2} - \eqref{eq:32v2}, it then follows that
\begin{align}
&\mathcal{F}\left(\frac{1}{T}\sum_{t=1}^T \mu_t\right) - \mathcal{F}(\pi) \leq \frac{L}{T} \left[ \sum_{t=1}^T  \int_{\mathbb{R}^d} \left \langle -\nabla_{W_2} \mathcal{\hat{F}}(\mu_t)(\varphi_t(x)), t_{\mu_t}^{\pi}(\varphi_t(x)) - x \right\rangle \mu_0 (\mathrm{d}x)\right. \nonumber \\
&\left.\hspace{45mm}- \int_{\mathbb{R}^d} I\left(\sum_{t=1}^T -\nabla_{W_2} \mathcal{\hat{F}}(\mu_t)(\varphi_t(x))\right) \mu_0 (\mathrm{d}x) + w_0\right], \label{eq:32}
\end{align}
where $I:\mathbb{R}^d\rightarrow(-\infty,\infty]$ is the function defined in the proof of Proposition \ref{theorem:convergence_rate}, c.f. \eqref{eq:bigI_def}. To proceed, we will now write
    \begin{align}
        t_{\mu_t}^{\pi}(\varphi_t(x)) &= T_{\mu_0}^{\pi}(x) + \underbrace{t_{\mu_t}^{\pi}(\varphi_t(x)) - T_{\mu_0}^{\pi}(x)}_{S_{t}(x)}. \label{t1}
    \end{align}
Based on this decomposition, we can rewrite the first term in \eqref{eq:32} as
\begin{align}
    &\frac{L}{T}\sum_{t=1}^T  \int_{\mathbb{R}^d} \left\langle -\nabla_{W_2} \mathcal{\hat{F}}(\mu_t)(\varphi_t(x)), t_{\mu_t}^{\pi}(\varphi_t(x)) - x\right\rangle \mu_0 (\mathrm{d}x) \nonumber \\
    &= \frac{L}{T} \bigg[ \int_{\mathbb{R}^d} \bigg\langle \sum_{t=1}^T -\nabla_{W_2} \mathcal{\hat{F}}(\mu_t)(\varphi_t(x)), T_{\mu_0}^{\pi}(x) - x\bigg\rangle \mu_0(\mathrm{d}x) 
    + \sum_{t=1}^T  \int_{\mathbb{R}^d} \underbrace{\left\langle -\nabla_{W_2} \mathcal{\hat{F}}(\mu_t)(\varphi_t(x)), S_t(x)\right\rangle}_{R_t(x)} \mu_0(\mathrm{d}x)\bigg]. \label{eq:80}
\end{align}
By substituting \eqref{eq:80}, the previous inequality \eqref{eq:32} can now be written as 
\begin{align}
&\mathcal{F}\left(\frac{1}{T}\sum_{t=1}^T \mu_t\right) - \mathcal{F}(\pi) \leq \frac{L}{T} \bigg[ \int_{\mathbb{R}^d} \bigg\langle \sum_{t=1}^T -\nabla_{W_2} \mathcal{\hat{F}}(\mu_t)(\varphi_t(x)), T_{\mu_0}^{\pi}(x) - x \bigg\rangle \mu_0 (\mathrm{d}x) + \sum_{t=1}^T  \int_{\mathbb{R}^d} R_t(x)\mu_0(\mathrm{d}x)  \nonumber \\
&\hspace{45mm}- \int_{\mathbb{R}^d} I\bigg( \sum_{t=1}^T -\nabla_{W_2} \mathcal{\hat{F}}(\mu_t)(\varphi_t(x)) \bigg) \mu_0 (\mathrm{d}x) + w_0 \bigg]  \label{eq:46} \\[2mm]
&\hphantom{\mathcal{F}\left(\frac{1}{T}\sum_{t=1}^T \mu_t\right) - \mathcal{F}(\pi)}= \frac{L}{T} \bigg[ w_0 + \int_{\mathbb{R}^d} \left[ \langle z(x), T_{\mu_0}^{\pi}(x) - x\rangle - I(z(x)) \right] \mu_0(\mathrm{d}x) + \sum_{t=1}^T  \int_{\mathbb{R}^d} R_t(x)\mu_0(\mathrm{d}x) \bigg], \label{eq:47}
\end{align}
where in the final line we have introduced the notation $z(x) = \sum_{t=1}^T -\nabla_{W_2} \mathcal{\hat{F}}(\mu_t)(\varphi_t(x))$. Suppose we now fix $x\in\mathbb{R}^d$, and write $\theta=z(x)$ and $x^{*} = T_{\mu_0}^{\pi}(x) - x$. Using this notation, we can now rewrite the first integrand in the previous expression as 
\begin{equation}
\langle z(x), T_{\mu_0}^{\pi}(x) - x\rangle - I(z(x)) := \langle \theta, x^{*} \rangle - I(\theta),
\end{equation}
for each fixed $x\in\mathbb{R}^d$. Taking the supremum over $\theta\in\mathbb{R}^d$ and using the definition of the convex conjugate, we can easily upper bound this expression by
\begin{align}
\langle \theta, x^{*} \rangle - I(\theta) \leq \sup_{\theta\in\mathbb{R}^d} \left(\langle \theta, x^{*} \rangle - I(\theta)\right) \leq I^{*}(x^{*}). 
\end{align}
Returning to our previous notation, and using the fact $x\in\mathbb{R}^d$ was chosen arbitrarily, we thus have
\begin{equation}
    \langle z(x), T_{\mu_0}^{\pi}(x) - x\rangle - I(z(x)) \leq I^{*}(T_{\mu_0}^{\pi}(x) - x)
\end{equation}
for all $x\in\mathbb{R}^d$. Substituting this upper bound into \eqref{eq:46} - \eqref{eq:47}, it then follows straightforwardly that
\begin{align}
\mathcal{F}\left(\frac{1}{T}\sum_{t=1}^T \mu_t\right) - \mathcal{F}(\pi) &\leq  \frac{L}{T} \left[ w_0 + \int_{\mathbb{R}^d} I^{*}(T_{\mu_0}^{\pi}(x) - x)\mu_0(\mathrm{d}x) + \sum_{t=1}^T  \int_{\mathbb{R}^d} R_t(x)\mu_0(\mathrm{d}x)  \right] \\
&=  \frac{L}{T} \left[ w_0 + \int_{\mathbb{R}^d} i^{*}(\left|\left|T_{\mu_0}^{\pi}(x) - x\right|\right|)\mu_0(\mathrm{d}x) + \sum_{t=1}^T  \int_{\mathbb{R}^d} R_t(x)\mu_0(\mathrm{d}x) \right], \label{eq:53}
\end{align}
where, in the second line, we have used the fact that the Fenchel conjugate of $i(||\cdot||)$ is $i^{*}(||\cdot||)$ since $i^{*}$ is an even function \citep[][Example 13.7]{Bauschke2011}. Similar to before, Lemma 18 of \citet{Orabona2016} allows to bound this Fenchel conjugate as
\begin{equation}
    i^{*}(u) \leq |u| \sqrt{T\ln \left(1+ \frac{24T^2u^2}{w_0^2}\right)} - \frac{w_0}{K_1\sqrt{T}}.
\end{equation}
Substituting this into the previous bound in \eqref{eq:53}, we then have that
\begin{align}
    \mathcal{F}\left(\frac{1}{T}\sum_{t=1}^T \mu_t\right) - \mathcal{F}(\pi)
    &\leq \frac{L}{T} \bigg[ w_0\left(1-\frac{1}{K_1\sqrt{T}}\right) + \int_{\mathbb{R}^d} \left|\left|T_{\mu_0}^{\pi}(x) - x\right|\right| \sqrt{T \ln \left(1+ \frac{24T^2\left|\left|T_{\mu_0}^{\pi}(x) - x\right|\right|^2}{w_0^2}\right)} \mu_0(\mathrm{d}x) \nonumber \\
    &\hspace{36mm}+ \sum_{t=1}^T  \int_{\mathbb{R}^d} R_t(x)\mu_0(\mathrm{d}x)  \bigg] \nonumber \\[3mm]
    &\leq \frac{L}{T} \bigg[ w_0 + \int_{\mathbb{R}^d} \left|\left|T_{\mu_0}^{\pi}(x) - x\right|\right| \sqrt{T \ln \left(1+ \frac{24T^2\left|\left|T_{\mu_0}^{\pi}(x) - x\right|\right|^2}{w_0^2}\right)} \mu_0(\mathrm{d}x) \label{eq:penultimate_bound} \\
    &\hspace{15mm}+ \sum_{t=1}^T  \int_{\mathbb{R}^d} R_t(x)\mu_0(\mathrm{d}x)  \bigg]. \nonumber
\end{align}
It remains to deal with the final term. Unsurprisingly, the bound on this term will follow directly from our alternative sufficient condition (Assumption \ref{sufficient_condition_1}). In particular, recalling the definition of $R_t$ from \eqref{eq:80}, we have 
\begin{align}
 \sum_{t=1}^T  \int_{\mathbb{R}^d} R_t(x)\mu_0(\mathrm{d}x) &= \sum_{t=1}^T  \int_{\mathbb{R}^d} \left\langle -\nabla_{W_2} \hat{\mathcal{F}}(\mu_t)(\varphi_t(x)), S_t(x)\right\rangle \mu_0(\mathrm{d}x) \\
 &\leq \sum_{t=1}^T \left[\int_{\mathbb{R}^d} ||\nabla_{W_2}\hat{\mathcal{F}}(\mu_t)(\varphi_t(x)) ||^2 \mu_0(\mathrm{d}x)\right]^{\frac{1}{2}} \left[\int_{\mathbb{R}^d}||S_t(x)||^2 \mu_0(\mathrm{d}x)\right]^{\frac{1}{2}} \label{eq:new_line2} \\
 &\leq\sum_{t=1}^T  \left[\int_{\mathbb{R}^d}||S_t(x)||^2 \mu_0(\mathrm{d}x)\right]^{\frac{1}{2}}  \label{eq:new_line3} \\
 &= \sum_{t=1}^T  \left[\int_{\mathbb{R}^d}||t_{\mu_t}^{\pi}(\varphi_t(x)) - T_{\mu_0}^{\pi}(x)||^2 \mu_0(\mathrm{d}x)\right]^{\frac{1}{2}} \leq K\sqrt{T}\ln\left[\mathrm{poly}(T)\right], \label{eq:o_bound}
\end{align}
where in \eqref{eq:new_line2} we have used the Cauchy-Schwarz inequality, in \eqref{eq:new_line3} we have used the assumed bound on $||\nabla_{W_2}\mathcal{F}(\mu_t)(\varphi_t(x))||$ (Assumption \ref{assumption2}), and in \eqref{eq:o_bound} we have substituted the definition of $S_t$ from \eqref{t1}, and used Assumption \ref{sufficient_condition_1}.
Finally, substituting \eqref{eq:o_bound} into \eqref{eq:penultimate_bound}, we arrive at
\begin{align}
    &\mathcal{F}\left(\frac{1}{T}\sum_{t=1}^T \mu_t\right) - \mathcal{F}(\pi) \\
    &\leq \frac{L}{T} \bigg[ w_0 + \int_{\mathbb{R}^d} \left|\left|T_{\mu_0}^{\pi}(x) - x\right|\right| \sqrt{T \ln \left(1+ \frac{24T^2\left|\left|T_{\mu_0}^{\pi}(x) - x\right|\right|^2}{w_0^2}\right)} \mu_0(\mathrm{d}x) + K\sqrt{T}\ln\left[\mathrm{poly}(T)\right] \bigg]. \nonumber
\end{align}
\end{proof}

\subsection{A Simple Gaussian Case}
\label{sec:gaussian}
In this section, we establish a stronger version of Proposition \ref{theorem2} in a very simple Gaussian setting. In particular, we consider the case in which the initial distribution is Gaussian, the target distribution is Gaussian, and these two distributions have the same covariance. In this case, not only can we tighten the bound in Proposition \ref{theorem2}, but we no longer require any additional technical assumptions (Assumption \ref{sufficient_condition_2} or \ref{sufficient_condition_1}). 

\begin{customprop}{3.4} \label{theorem:gaussian}
Let $\mathcal{F}(\mu) = \mathrm{KL}(\mu|\pi)$. Let $\mu_0\sim \mathcal{N}(m_0,\Sigma)$ and $\pi\sim \mathcal{N}(m_{\pi},\Sigma)$. Suppose that Assumption \ref{assumption2} holds. Then 
    \begin{align}
    &\mathcal{F}\left(\frac{1}{T}\sum_{t=1}^T \mu_t\right) - \mathcal{F}(\pi) \leq \frac{L}{T} \bigg[ w_0 + \int_{\mathbb{R}^d} \left|\left|t_{\mu_0}^{\pi}(x) - x\right|\right| \sqrt{T \ln \left(1+ \frac{24T^2\left|\left|t_{\mu_0}^{\pi}(x) - x\right|\right|^2}{w_0^2}\right)} \mu_0(\mathrm{d}x) \bigg]. \label{eq:gaussian_bound}
\end{align}
\end{customprop}

\begin{proof}[Proof of Proposition \ref{theorem:gaussian}]
We will establish this result as a special case of Proposition \ref{theorem2}. We must therefore begin by verifying the additional assumptions required by Proposition \ref{theorem2}, namely, Assumption \ref{assumption1} and Assumption \ref{sufficient_condition_1}. We begin with Assumption \ref{assumption1}. Given $\pi\sim \mathcal{N}(m_{\pi},\sigma^2)$, it follows from standard results that the functional $\mathcal{F}(\mu) = \mathrm{KL}(\mu|\pi)$ is proper, lower semi-continuous, and geodesically convex \citep[e.g.,][Sec. 9.4]{Ambrosio2008}. Thus, this assumption is indeed satisfied. 

We now turn our attention to Assumption \ref{sufficient_condition_1}. We will show that Assumption \ref{sufficient_condition_1} is satisfied with $T_{\mu_0}^{\pi} = t_{\mu_0}^{\pi}$, and $K = 0$. To prove this, it is clearly sufficient to show that for all $t\in[T]$, 
\begin{equation}
    t_{\mu_t}^{\pi}(\varphi_t(x)) = t_{\mu_0}^{\pi}(x), \label{eq:102}
\end{equation}
since in this case the LHS of the bound in Assumption \ref{sufficient_condition_1} is identically zero. Our proof of \eqref{eq:102} will consist of two steps. We will first establish that $\smash{ \varphi_t(x) = t_{\mu_0}^{\mu_t}(x)}$, i.e., the transport map defined by Alg. \ref{alg:param_free_grad_descent_v2} coincides with the optimal transport map. We will then show that $\smash{t_{\mu_t}^{\pi}(t_{\mu_0}^{\mu_t}(x)) = t_{\mu_0}^{\pi}(x)}$, i.e., the composition of the optimal transport maps $t_{\mu_0}^{\mu_t}$ and $t_{\mu_t}^{\pi}$ coincides with the optimal transport map $t_{\mu_0}^{\pi}$. By substituting the first of these identities into the second, we obtain the required result in \eqref{eq:102}. 

We begin by showing that $\varphi_t(x)=t_{\mu_0}^{\mu_t}(x)$. To prove this, we will first establish that $\mu_t \sim \mathcal{N}(m_t, \Sigma)$ and $\varphi_t(x) = m_t + (x - m_0)$. We proceed by induction. In the base case ($t=1$), it follows from the update equation \eqref{eq:x_transform} in Alg. \ref{alg:param_free_grad_descent_v2} that $\varphi_1(x) = \mathrm{id}(x) = x$. 
We thus have, by definition, that $\mu_1 = (\varphi_1)_{\#}\mu_0 = \mu_0$. Thus, in particular, $\mu_1 \sim \mathcal{N}(m_1,\Sigma)$ and $\varphi_1(x) = m_1 + (x - m_0)$, in this case with $m_1 = m_0$. This proves the base case. 

We now proceed to the inductive step. Assume that, for $s=1,\dots,t-1$, it is indeed the case that $\mu_s \sim \mathcal{N}(m_s, \Sigma)$ and $\varphi_s(x) = m_s + (x - m_0)$. We then have
\begin{align}
    \varphi_t(x) &= x - \frac{\sum_{s=1}^{t-1}\nabla_{W_2}\mathcal{F}(\mu_s)(\varphi_s(x))}{Lt}\left(w_0 - \sum_{s=1}^{t-1} \langle \frac{1}{L}\nabla_{W_2}\mathcal{F}(\mu_s)(\varphi_s(x)), \varphi_s(x) - x) \rangle \right) \\
    &= x +\underbrace{- \frac{\sum_{s=1}^{t-1}\Sigma^{-1}(m_s - m_{\pi})}{Lt}\left(w_0 - \sum_{s=1}^{t-1} \langle \frac{1}{L}\Sigma^{-1}(m_s - m_{\pi}), m_s - m_0 \rangle \right)}_{:=m_t - m_0} = m_t + (x - m_0), 
\end{align}
where in the second line we have used the fact that $\smash{\nabla_{W_2}\mathcal{F}(\mu) = \nabla_{W_2} \mathrm{KL}(\mu|\pi) = \nabla \log\frac{\mu}{\pi} = \nabla \log \mu - \nabla \log \pi}$, which implies in particular that 
\begin{align}
    \nabla_{W_2}\mathcal{F}(\mu_s)(\varphi_s(x)) &= \nabla \log \mu_s(x_s) - \nabla \log \pi(x_s) \\
    &= - \Sigma^{-1}(\varphi_s(x) - m_s) + \Sigma^{-1}(\varphi_s(x) - m_{\pi})= \Sigma^{-1}(m_s - m_\pi).
\end{align}
Thus, we do indeed have $\varphi_t(x) = m_t + (x - m_0)$. Moreover, using standard properties regarding affine transformations of normal random variables, it follows that $\mu_t = (\varphi_t)_{\#}\mu_0 \stackrel{d}{=} \mathcal{N}(m_t,\Sigma)$. This completes the inductive step. We have thus shown that $\mu_t \sim \mathcal{N}(m_t, \Sigma)$ and $\varphi_t(x) = m_t + (x - m_0)$ for all $t\in \mathbb{N}$. The required identity, namely $t_{\mu_0}^{\mu_t}(x) = \varphi_t(x)$, now follows using the fact that the optimal transport map between $\mu_0\sim \mathcal{N}(m_0,\Sigma)$ and $\mu_t\sim\mathcal{N}(m_t,\Sigma)$ is precisely given by $\smash{t_{\mu_0}^{\mu_t}(x) = m_t + (x - m_0)}$ \citep[e.g.,][]{Chen2019}.

It remains to show that $\smash{t_{\mu_t}^{\pi}(t_{\mu_0}^{\mu_t}(x)) = t_{\mu_0}^{\pi}(x)}$. Based on the previous result, and standard results on the optimal transport map between Gaussians \citep[e.g.,][]{Chen2019}, we have $\smash{t_{\mu_t}^{\pi}(x) = m_{\pi} + (x - m_t)}$ and $\smash{t_{\mu_0}^{\mu_t}(x) = m_{t} + (x - m_0)}$.
It follows straightforwardly that 
\begin{equation}
    t_{\mu_t}^{\pi}(t_{\mu_0}^{\mu_t}(x)) = m_{\pi} + \left( [ m_t + (x - m_0) ] - m_t \right) = m_{\pi} + (x - m_0) = t_{\mu_0}^{\pi}(x).
\end{equation}
We have thus proved both of the identities required to establish \eqref{eq:102}. This verifies Assumption \ref{sufficient_condition_1} with $T_{\mu_0}^{\pi}=t_{\mu_0}^{\pi}$, and $K = 0$. Finally, substituting  $T_{\mu_0}^{\pi} = t_{\mu_0}^{\pi}$ and $K = 0$ into the bound in Proposition \ref{theorem2}, we arrive at the required bound in Proposition \ref{theorem:gaussian}. 
\end{proof}

\section{Other Coin ParVI Algorithms}
\label{sec:coin-parvi}

In Sec. \ref{sec:coin-svgd}, we presented Coin SVGD, an algorithm which can be viewed as the coin sampling analogue of SVGD \cite{Liu2016a}. In this section, we provide details of two other algorithms - Coin LAWGD and Coin KSDD - which represent coin sampling analogues of two other ParVI algorithms.

\subsection{Coin Laplacian Adjusted Wasserstein Gradient Descent}
\label{sec:coin-lawgd}
Let $\mathcal{F}(\mu) = \mathrm{KL}(\mu|\pi)$, with $\nabla_{W_2}\mathcal{F}(\mu) = \nabla \ln \frac{\mathrm{d}\mu}{\mathrm{d}\pi}$. Following \citet{Chewi2020}, suppose that we replace $\nabla_{W_2}\mathcal{F}(\mu)$ in Alg. \ref{alg:param_free_grad_descent_v2} by $\nabla P_{\pi,k_{\mathcal{L}}} \frac{\mathrm{d}\mu}{\mathrm{d}\pi}$, the gradient of the image of $\smash{\frac{\mathrm{d}\mu}{\mathrm{d}\pi}}$ under the integral operator $P_{\mu,k_{\mathcal{L}}}$, where $k_{\mathcal{L}}$ is the kernel such that $P_{\pi,k_{\mathcal{L}}} = -\mathcal{L}_{\pi}^{-1}$. Here, $\mathcal{L}_{\pi}$ denotes the infinitesimal generator of the overdamped Langevin diffusion with stationary distribution $\pi$. In this case, one can show that $\smash{\nabla P_{\pi,k_{\mathcal{L}}} \frac{\mathrm{d}\mu}{\mathrm{d}\pi} = \mathbb{E}_{x\sim\mu}[\nabla_{1}k_{\mathcal{L}}(\cdot,x)]}$ \citep[][Sec.  4]{Chewi2020}, and thus 
\begin{equation}
    \nabla P_{\pi,k_{\mathcal{L}}} \frac{\mathrm{d}\mu^N}{\mathrm{d}\pi}(x_t^{i}) = \frac{1}{N}\sum_{j=1}^N \nabla_{1} k_{\mathcal{L}}(x_t^{i},x_t^{j}).
\end{equation}
By using these gradients in Alg. \ref{alg:param_free_grad_descent_v2}, we obtain a learning-rate free analogue of the LAWGD algorithm \citep[][Alg. 1]{Chewi2020}. This algorithm is summarised in Alg. \ref{alg:param_free_lawgd}.

\begin{algorithm}[ht]
   \caption{Coin Laplacian Adjusted Wasserstein Gradient Descent (Coin LAWGD)}
   \label{alg:param_free_lawgd}
\begin{algorithmic}
   \STATE {\bfseries Input:} initial measure $\mu_0\in\mathcal{P}_2(\mathbb{R}_{d})$, initial particles $(x_0^{i})_{i=1}^N \stackrel{\mathrm{i.i.d.}}{\sim}\mu_0$, initial wealth of particles $(w_0^{i})_{i=1}^N\in\mathbb{R}_{+}$, gradient upper bound $L$.  
   \FOR{$t=1$ {\bfseries to} $T$}
   \FOR{$i=1$ {\bfseries to} $N$} 
   \STATE Compute \vspace{-6mm}
\begin{align}
x_{t}^{i} &= x_0^{i} -\frac{{\sum_{s=1}^{t-1} \sum_{j=1}^N \nabla_{1}k_{\mathcal{L}}(x_s^{i},x_s^{j})}}{LNt} \bigg( w_0^{i} -  \sum_{s=1}^{t-1} \bigg\langle \frac{1}{LN}\sum_{j=1}^N \nabla_{1}k_{\mathcal{L}}(x_s^{i},x_s^{j}), x_s^{i} - x_0^{i} \bigg\rangle \bigg).
\vspace{-2mm}
\end{align}
   \STATE Define $\smash{\mu_{t}^N = \frac{1}{N}\sum_{i=1}^N \delta_{x_t^{i}}}$. \vspace{2mm}
   \ENDFOR
   \ENDFOR
   \STATE {\bfseries Output:} $\mu_T^{N}$ or $\frac{1}{T}\sum_{t=1}^T \mu_t^{N}$.
\end{algorithmic}
\end{algorithm}

\subsection{Coin Kernel Stein Discrepancy Descent}
\label{sec:coin-ksd}
Let $\mathcal{F}(\mu) = \frac{1}{2}\mathrm{KSD}^2(\mu|\pi)$, where $\mathrm{KSD}(\mu|\pi)$ is the kernel Stein discrepancy, defined according to \cite{Liu2016b,Chwialkowski2016,Gorham2017}
\begin{equation}
    \mathrm{KSD}(\mu|\pi) = \sqrt{\int_{\mathbb{R}^d}\int_{\mathbb{R}^d} k_{\pi}(x,y) \mu(\mathrm{d}x)\mu(\mathrm{d}y)}, 
\end{equation}
and where $k_{\pi}$ is the Stein kernel, defined in terms of the score $s = \nabla \log \pi$, and a positive semi-definite kernel $k$, as
\begin{align}
    k_{\pi}(x,y) &= s^T(x)s(y)k(x,y) + s^T(x) \nabla_{2}k(x,y)+\nabla_{1}k^T(x,y)s(y) + \nabla_{\cdot 1} \nabla_{2}k(x,y).
\end{align}
In this case, given a discrete measure $\smash{{\mu}^N = \frac{1}{N}\sum_{j=1}^N \delta_{x^{j}}}$, the loss function and its gradient are given by
\begin{equation}
    \mathcal{F}(\mu^N) = \frac{1}{N^2}\sum_{i,j=1}^N k_{\pi}(x^{i},x^{j})~~~,~~~\nabla_{x_i}\mathcal{F}(\mu_t^N) = \frac{1}{N^2}\sum_{j=1}^N \nabla_{2}k_{\pi}(x_t^{j},x_t^{i}).
\end{equation}
By substituting these gradients into Alg. \ref{alg:param_free_grad_descent_v2}, we obtain a learning-rate free analogue of KSDD \cite{Korba2021}.\footnote{In fact, \citet{Korba2021} also propose a learning-rate free version of KSDD based on the quasi-Newton L-BFGS algorithm \cite{Liu1989}. Our method provides an alternative approach based on the `coin-betting' paradigm.} This algorithm is summarised in Alg. \ref{alg:param_free_ksd}. 

\begin{algorithm}[ht]
   \caption{Coin Kernel Stein Discrepancy Descent (Coin KSDD)}
   \label{alg:param_free_ksd}
\begin{algorithmic}
   \STATE {\bfseries Input:} initial measure $\mu_0\in\mathcal{P}_2(\mathbb{R}_{d})$, initial particles $(x_0^{i})_{i=1}^N \stackrel{\mathrm{i.i.d.}}{\sim}\mu_0$, initial wealth of particles $(w_0^{i})_{i=1}^N\in\mathbb{R}_{+}$, kernel $k$, gradient upper bound $L$. 
   \FOR{$t=1$ {\bfseries to} $T$}
   \FOR{$i=1$ {\bfseries to} $N$} 
   \STATE Compute \vspace{-7mm}
\begin{align}
x_{t}^{i} &= x_0^{i} -\frac{{\sum_{s=1}^{t-1} \sum_{j=1}^N \nabla_{2}k_{\pi}(x_s^{j},x_s^{i})}}{LN^2t} \bigg( w_0^{i} -  \sum_{s=1}^{t-1} \bigg\langle \frac{1}{LN^2}\sum_{j=1}^N \nabla_{2}k_{\pi}(x_s^{j},x_s^{i}), x_s^{i} - x_0^{i} \bigg\rangle \bigg).
\vspace{-0.5mm}
\end{align}
   \STATE Define $\smash{\mu_{t}^{N} = \frac{1}{N}\sum_{i=1}^{x,N} \delta_{x_t^{i}}}$. \vspace{2mm}
   \ENDFOR
   \ENDFOR
   \STATE {\bfseries Output:} $\mu_T^{N}$ or $\frac{1}{T}\sum_{t=1}^T \mu_t^{N}$.
\end{algorithmic}
\end{algorithm}

\newpage
\section{Coin Sampling with Adaptive Gradient Bounds}
\label{sec:coinWGD_adaptive}

\subsection{Adaptive Coin Wasserstein Gradient Descent}

In principle, Alg. \ref{alg:param_free_grad_descent_v2} (and Alg. \ref{alg:param_free_svgd}) depend on knowledge of a constant $L>0$ such that, for all $t\in[T]$, $||\nabla_{W_2}\mathcal{F}(\mu_t)(x_t)||\leq L$. In practice, however, such a constant may not be not known in advance. In this case, following \citet{Orabona2017}, we can use a modified version of our algorithm in which the gradient bounds are adaptively estimated. This algorithm is summarised in Alg. \ref{alg:adaptive_param_free_grad_descent}.

\begin{algorithm}[ht]
   \caption{Adaptive Coin Wasserstein Gradient Descent}
   \label{alg:adaptive_param_free_grad_descent}
\begin{algorithmic}
   \STATE {\bfseries Input:} initial measure $\mu_0\in\mathcal{P}_2(\mathbb{R}^d)$, initial parameter $x_0\sim\mu_0$, dissimilarity functional $\mathcal{F}:\mathcal{P}_2(\mathbb{R}^d)\rightarrow(-\infty,\infty]$.
   \STATE{\bfseries Initialise:} for $j=1,\dots,d$, $L_{0,j}=0$, $G_{0,j}=0$, $R_{0,j}=0$. 
   \FOR{$t=1$ {\bfseries to} $T$}
   \STATE Compute the negative Wasserstein gradient: $c_{t-1} = -\nabla_{W_2}\mathcal{F}(\mu_{t-1})(x_{t-1})$.
   \FOR{$j=1$ {\bfseries to} $d$}
    \STATE Update the maximum observed scale $L_{t,j} = \mathrm{max}(L_{t-1,j}, |c_{t-1,j}|)$.
    \STATE Update the sum of the absolute value of the gradients: $G_{t,j} = G_{t-1,j} + |c_{t-1,j}|$.
    \STATE Update the reward: $R_{t,j} = \max(R_{t-1,j} + c_{t-1,j} (x_{t-1,j} - x_{0,j}), 0)$.
    \STATE Update the parameter
\begin{equation}
x_{t,j} = x_{0,j} + \frac{\sum_{s=1}^{t-1} c_{s,j}}{G_{t,j} + L_{t,j}} (1 + \frac{R_{t,j}}{L_{t,j}}).
\end{equation} 
\ENDFOR
    \STATE Define $\mu_{t} = (\varphi_{t})_{\#} \mu_0$, where $\varphi_t:x_0 \mapsto x_t$.
   \ENDFOR
   \STATE {\bfseries Output:} $\mu_T$.
\end{algorithmic}
\end{algorithm}

\subsection{Adaptive Coin Stein Variational Gradient Descent}
\label{sec:adaptive-coin-svgd}
In the same way, one can also obtain an adaptive version of Coin SVGD (Alg. \ref{alg:param_free_svgd}) and, indeed, Coin LAWGD (Alg. \ref{alg:param_free_lawgd}) and Coin KSDD (Alg. \ref{alg:param_free_ksd}).\footnote{In the interest of brevity, we do not provide the adaptive versions of Coin LAWGD and Coin KSDD in full. However, these are easily obtained by substituting the relevant gradients into the adaptive version of Coin SVGD.} The adaptive Coin SVGD algorithm is summarised in Alg. \ref{alg:adaptive_param_free_svgd}. Following \citet{Orabona2017}, we make one further alteration when we use the adaptive version of Coin SVGD to perform inference in Bayesian neural networks (see Sec. \ref{sec:bnn}). In particular, we now modify the denominator of the betting fraction in Alg. \ref{alg:adaptive_param_free_svgd} such that it is at least $\alpha L_{t,j}^{i}$, for some positive constant $\alpha>0$, which as a default we set equal to 100. Thus, the update in \eqref{eq:adaptive_update} now becomes
\begin{equation}
x_{t,j}^{i} = x_{0,j}^{i} + \frac{\sum_{s=1}^{t-1} c_{s,j}^{i}}{\max(G_{t,j}^{i} + L_{t,j}^{i},\alpha L_{t,j}^{i})} (1 + \frac{R_{t,j}^{i}}{L_{t,j}^{i}}).
\end{equation}
In practice, is is these adaptive algorithms that we use in all of our numerical experiments (Sec. \ref{sec:numerics}), and that we recommend for use in future work. 

\begin{algorithm}[ht]
   \caption{Adaptive Coin Stein Variational Gradient Descent}
   \label{alg:adaptive_param_free_svgd}
\begin{algorithmic}
   \STATE {\bfseries Input:} initial measure $\mu_0\in\mathcal{P}_2(\mathbb{R}^d)$; initial particles $(x_0^{i})_{i=1}^N\stackrel{\text{i.i.d.}}{\sim} \mu_0$.
   \STATE{\bfseries Initialise:} for $i=1,\dots,N$, $j=1,\dots,d$, $L^{i}_{0,j}=0$, $G_{0,j}^{i}=0$, $R_{0,j}^{i}=0$. 
   \FOR{$t=1$ {\bfseries to} $T$}
   \FOR{$i=1$ {\bfseries to} $N$}
   \STATE Compute the negative gradient $c_{t-1}^{i} =  -\tfrac{1}{N}\textstyle\sum_{j=1}^N[k(x_{t-1}^{j},x_{t-1}^{i}) \nabla U(x_{t-1}^{j}) -  \nabla_{1}k(x_{t-1}^{j},x_{t-1}^{i})]$.
   \FOR{$j=1$ {\bfseries to} $d$}
        \STATE Update the maximum observed scale: $\smash{L_{t,j}^{i} = \mathrm{max}(L_{t-1,j}^{i}, |c_{t-1,j}^{i}|)}$. 
    \STATE Update the sum of the absolute value of the gradients: $\smash{G_{t,j}^{i} = G_{t-1,j}^{i} + |c_{t-1,j}^{i}|}$. 
        \STATE Update the reward $R_{t,j}^{i} = \max(R_{t-1,j}^{i} + \langle c_{t-1,j}^{i}, x_{t-1,j}^{i} - x_{0,j}^{i} \rangle, 0)$.
    \STATE Update the parameter
\begin{equation}
x_{t,j}^{i} = x_{0,j}^{i} + \frac{\sum_{s=1}^{t-1} c_{s,j}^{i}}{G_{t,j}^{i} + L_{t,j}^{i}} (1 + \frac{R_{t,j}^{i}}{L_{t,j}^{i}}). \label{eq:adaptive_update}
\end{equation}
    \ENDFOR
   \ENDFOR
   \STATE Define $\mu_{t}^N = \frac{1}{N} \sum_{i=1}^N \delta_{x_{t}^{i}}$.
   \ENDFOR
   \STATE {\bfseries Output:} $\mu_T^{N}$.
\end{algorithmic}
\end{algorithm}

\textbf{Computational Complexity}.  In terms of computational cost and memory requirements, the adaptive variant of Coin SVGD is similar to SVGD when the latter is paired, as is common, with a method such as Adagrad \cite{Duchi2011}, RMSProp \cite{Tieleman2012}, or Adam \cite{Kingma2015}. The computational cost per iteration is $O(N^2)$ in the number of particles $N$, due to the kernelised gradient. In terms of memory requirements, it is necessary to keep track of the sum of the previous gradients, the sum of the absolute value of the previous gradients, the maximum observed absolute value of the previous gradients, and the reward, for each of the particles. Thus, in particular, the memory requirement is $O(Nd)$ in the number of particles $N$ and the dimension $d$. This is identical to, e.g., Adagrad, which must keep track of the sum of the squares of the previous gradients, for each of the particles \cite{Duchi2011}.

\section{Additional Experimental Details and Numerical Results}
\label{sec:experimental-details}
We implement our methods using Python 3, PyTorch, Theano, and Jax. For our comparisons, we use existing implementations of  \href{https://github.com/dilinwang820/Stein-Variational-Gradient-Descent}{\texttt{SVGD}} by \citet{Liu2016a},  \href{https://github.com/dilinwang820/Stein-Variational-Gradient-Descent}{\texttt{LAWGD}} by \citet{Chewi2020}, and \href{https://github.com/pierreablin/ksddescent}{\texttt{KSDD}} by \citet{Korba2021}. We perform all experiments using a MacBook Pro 16" (2021) laptop with Apple M1 Pro chip and 16GB of RAM.

\subsection{Toy Examples}
\label{sec:toy-details}

\subsubsection{Coin SVGD}
\textbf{Experimental Details}. In Sec. \ref{sec:intro}, we compare the performance of SVGD \cite{Liu2016a} and Coin SVGD (Alg. \ref{alg:param_free_svgd}) on the following two-dimensional distributions.

\emph{Two-Dimensional Gaussian}. We first consider an anisotropic bivariate Gaussian distribution, $p(x) = \mathcal{N}(x|\mu,\Sigma)$, where we set $\mu = (-1,1)^{\top}$ and $\Sigma^{-1} = \big( \begin{smallmatrix} 3 & -0.5 \\ -0.5 & 1 \end{smallmatrix} \big) $.

\emph{Mixture of Two Two-Dimensional Gaussians}. For the second example, we consider a mixture of two bivariate Gaussian distributions, $p(x) = \alpha_1 \mathcal{N}(x; \mu_1,\Sigma_1) + \alpha_2 \mathcal{N}(x; \mu_2,\Sigma_2)$, with $\alpha_1 = 0.5$, $\mu_1 = (-2,2)^{\top}$, and $\Sigma_1 = \frac{1}{2}\mathbb{1}$; $\alpha_2 = 0.5$, $\mu_2 = (2,-2)^{\top}$, and $\Sigma_2 = \frac{1}{2}\mathbb{1}$.

\emph{Donut Distribution}. We next consider an annulus or `donut' distribution, with density $p(x) \propto \exp(-\frac{(|x|-r_0)^2}{2\sigma^2})$, where we set $r_0=2.5$ and $ \sigma^2 = 0.5$.

\emph{Rosenbrock Distribution}. The next example is a variant on the so-called Rosenbrock or `banana' distribution \cite{Pagani2022}. The target density is given by $p(x) \propto \exp [ -\frac{1}{2} ({x}' - {\mu})^{\top} \Sigma^{-1} ({x'} - \mu) ]$, where $x'_1 = x_1/a$, and $x'_2 = ax_1 + ab(x_1^2 + a^2)$.  In our experiments, we set $a=-1$, $b=1$, $\mu=(0,1)^{\top}$, and $\Sigma = \big( \begin{smallmatrix} 1 & 0.5 \\ 0.5 & 1 \end{smallmatrix} \big)$. 

\emph{Squiggle Distribution}. Our penultimate example is a two-dimensional `squiggle' distribution; see, e.g., App. E in \citet{Hartmann2022}. In this case, the target density again takes the form $p(x) \propto \exp \left[ -\frac{1}{2}({x}' - {\mu})^{\top} \Sigma^{-1} ({x'} - \mu) \right]$, where now $x'_1 = x_1$ and $x'_2 = x_2 + \sin(\omega x_1)$. In our experiments, we set $\mu = (1,1)^{\top} $, $\Sigma = \big(\begin{smallmatrix} 2 & 0.25 \\ 0.25 & 0.5 \end{smallmatrix} \big)$, and the frequency $\omega =2$.

\emph{Funnel Distribution}. Our final example is a two-dimensional `funnel' distribution, with density $p(x_1,x_2)\propto \mathcal{N}(x_1; \mu_1, \exp (x_2)) \mathcal{N}(x_2; \mu_2, \sigma_2^2)$. In our experiments, we set $\mu_1=1$, $\mu_2 =4$, and $\sigma_2 = 3$. This example, in ten-dimensions, was first introduced in \citet{Neal2003} to illustrate the difficulty of sampling from some hierarchical models.

In all cases, we run both algorithms using $N=20$ particles, and for $T=1000$ iterations. We initialise the particles according to $\smash{(\theta_0^{i})_{i=1}^N\stackrel{\mathrm{i.i.d.}}{\sim} \mathcal{N}(0,0.1^2)}$. Finally, we use Adagrad \cite{Duchi2011} to adapt the learning rate for SVGD.

\textbf{Numerical Results}. In Fig. \ref{fig:figure1_KSD} and Fig. \ref{fig:figure1_KSD_vs_t}, we provide a more detailed comparison of SVGD and Coin SVGD. In particular, in Fig. \ref{fig:figure1_KSD}, we plot the the KSD for both algorithms after $T=1000$ iterations as a function of the learning rate. Meanwhile, in Fig. \ref{fig:figure1_KSD_vs_t}, we plot the KSD as a function of the iterations, using the optimal learning rate as determined by the results in Fig. \ref{fig:figure1_KSD}, and two other learning rates. In both cases, following \citet{Gorham2017}, we use the inverse multi-quadratic (IMQ) kernel $k(x,x') = (c^2 + ||x-x'||_2^2)^{\beta}$ to compute the KSD, where $c>0$ and $\beta\in(-1,0)$. In the interest of a fair comparison, we also use Adagrad \cite{Duchi2011} to adapt the learning rate in SVGD. 

In these examples, the performance of Coin SVGD is competitive with the best performance of SVGD, using the optimal but a prior unknown learning rate. Moreover, Coin SVGD clearly outperforms SVGD for sub-optimal choices of the learning rate, attaining significantly lower values of the KSD. In particular, using a step size which is too small is insufficient to guarantee convergence within 1000 iterations (Fig. \ref{fig:figure1_KSD_vs_t}, green lines), while using a step size which is too large leads to non-convergence (Fig. \ref{fig:figure1_KSD_vs_t}, red lines) and, ultimately, numerical instability. It is worth emphasising that it is difficult to determine a good step size, or to implement a line-search method, since SVGD does not minimise a simple function. On the other hand, Coin SVGD achieves performance close to, or even better than, the performance of optimally-tuned SVGD, without any need to tune a step size.

\begin{figure*}[t!]
\centering
\subfigure[Gaussian.]{\includegraphics[trim=0 0 0 0, clip, width=.29\textwidth]{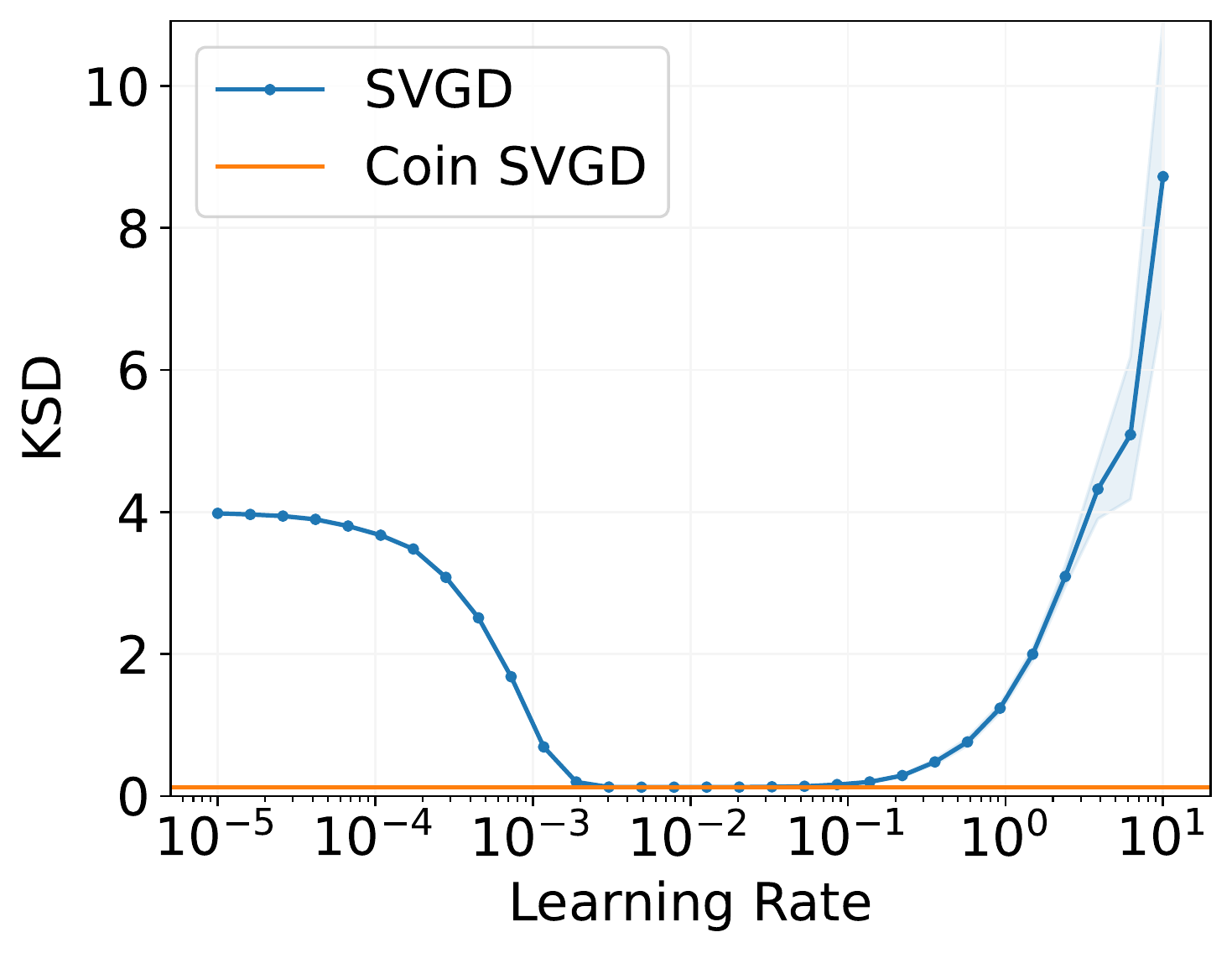}}
\hspace{2mm}
\subfigure[Mixture of Gaussians.]{\includegraphics[trim=0 0 0 0, clip, width=.28\textwidth]{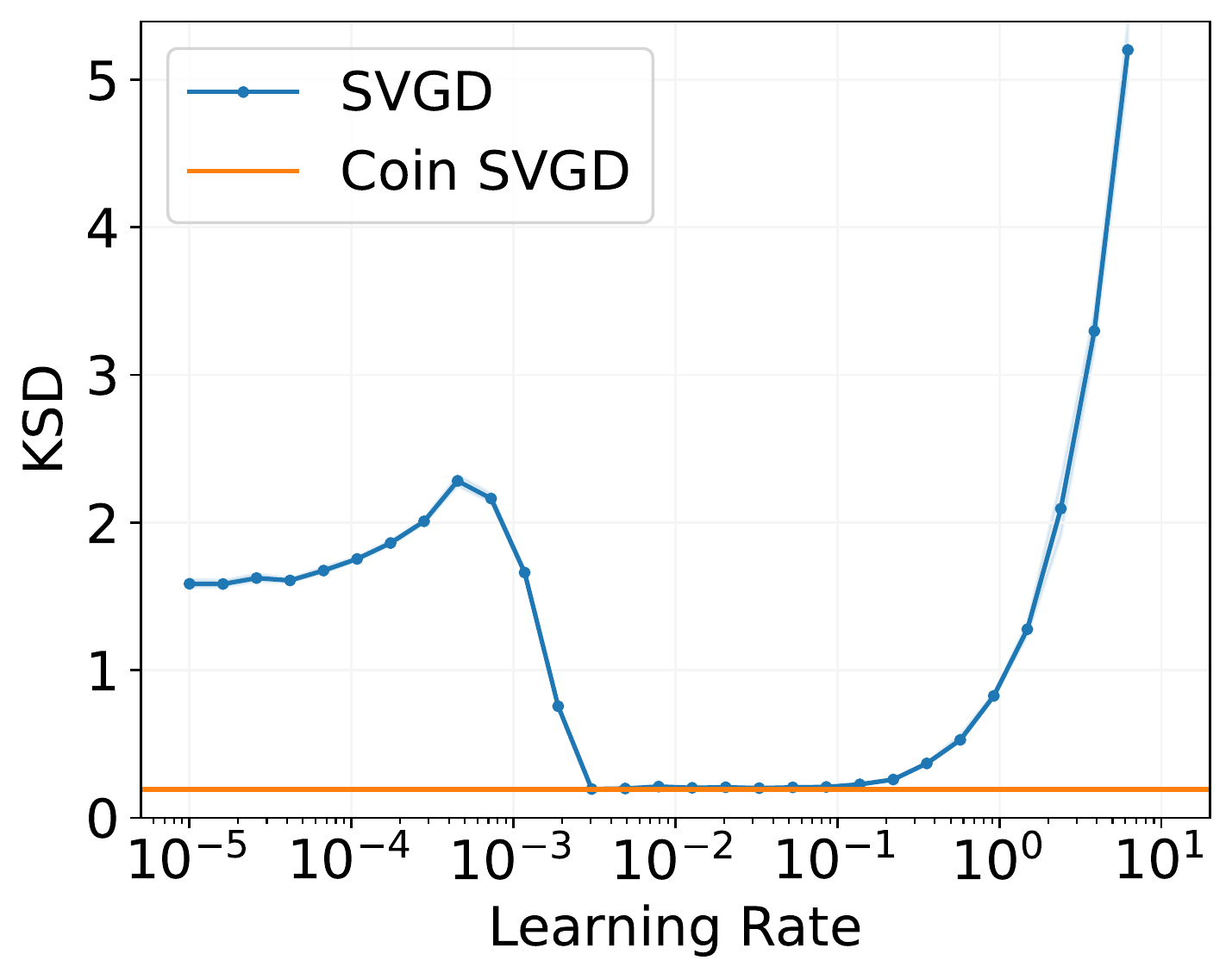}}
\hspace{2mm}
\subfigure[Donut.]{\includegraphics[trim=0 0 0 0, clip, width=.28\textwidth]{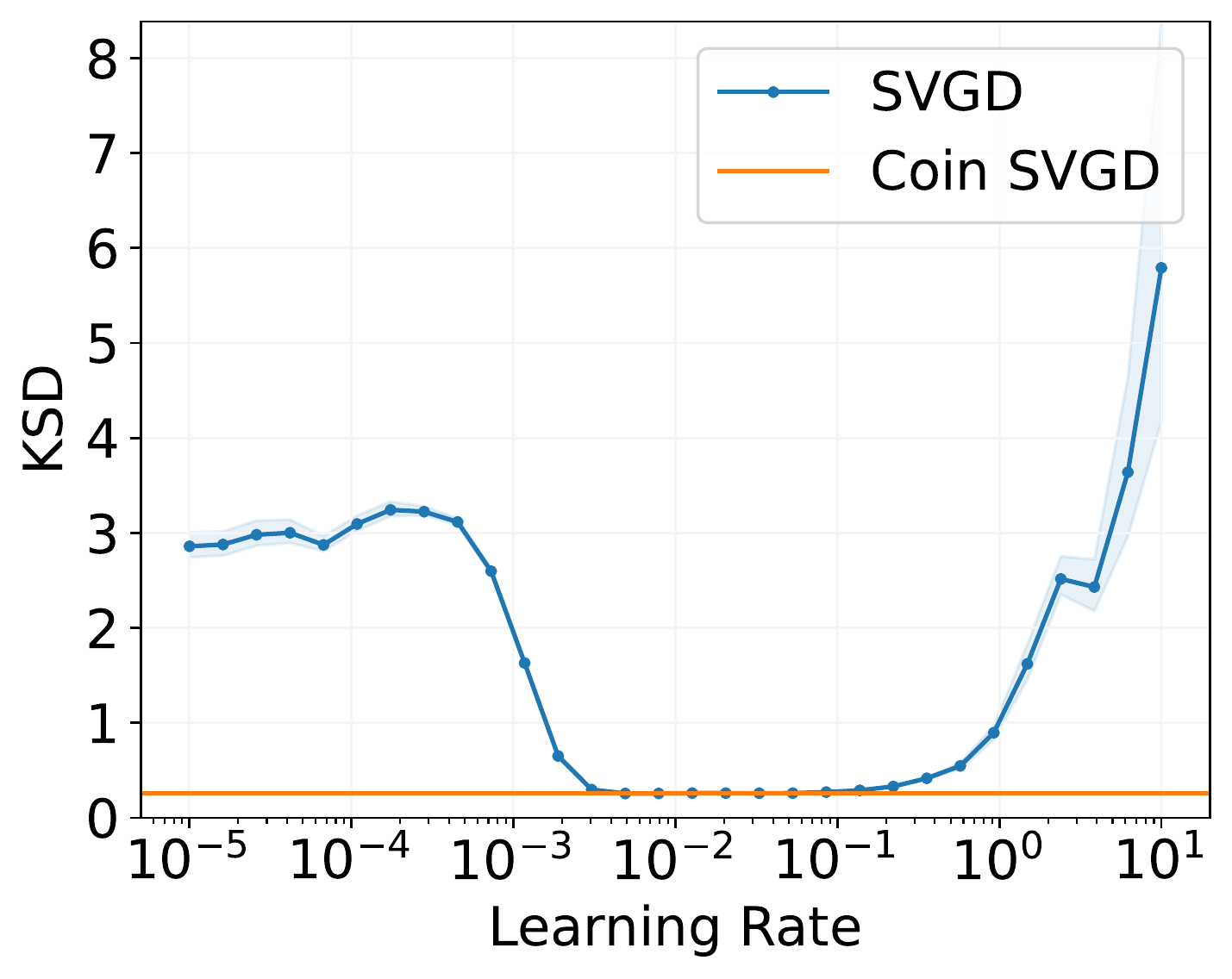}}
\subfigure[Rosenbrock Banana.]{\includegraphics[trim=0 0 0 0, clip, width=.3\textwidth]{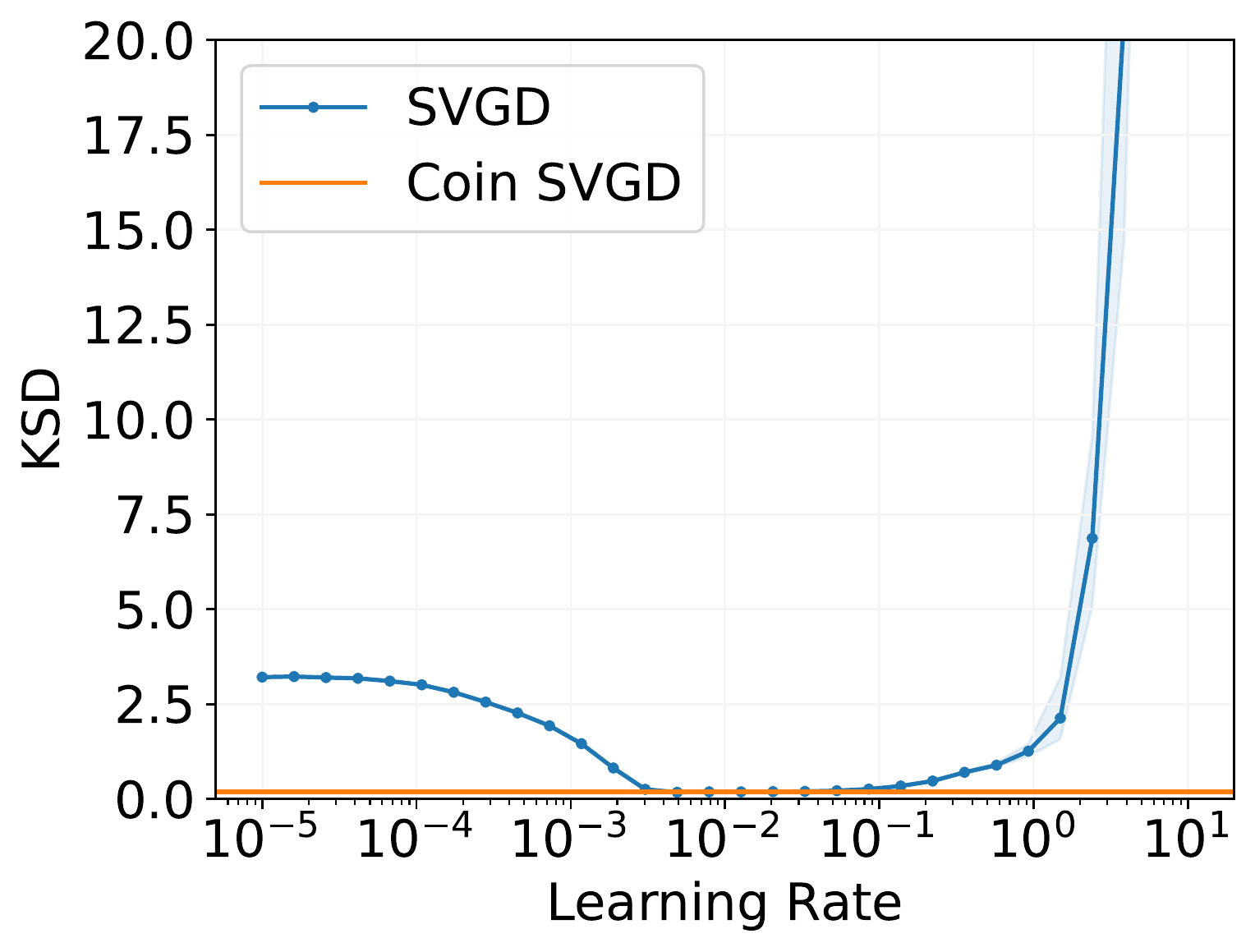}}
\hspace{2mm}
\subfigure[Squiggle.]{\includegraphics[trim=0 0 0 0, clip, width=.28\textwidth]{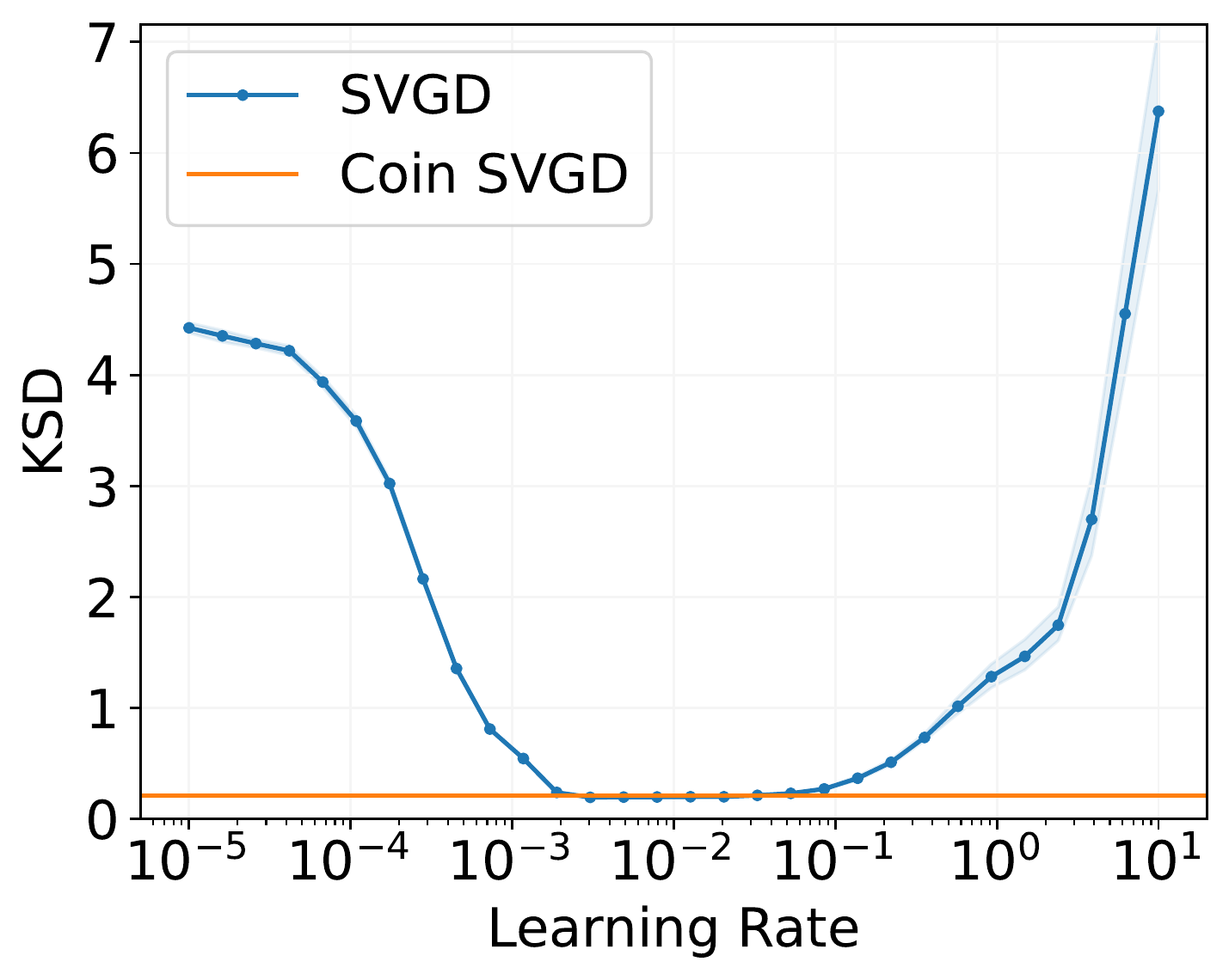}}
\hspace{2mm}
\subfigure[Funnel.]{\includegraphics[trim=0 0 0 0, clip, width=.28\textwidth]{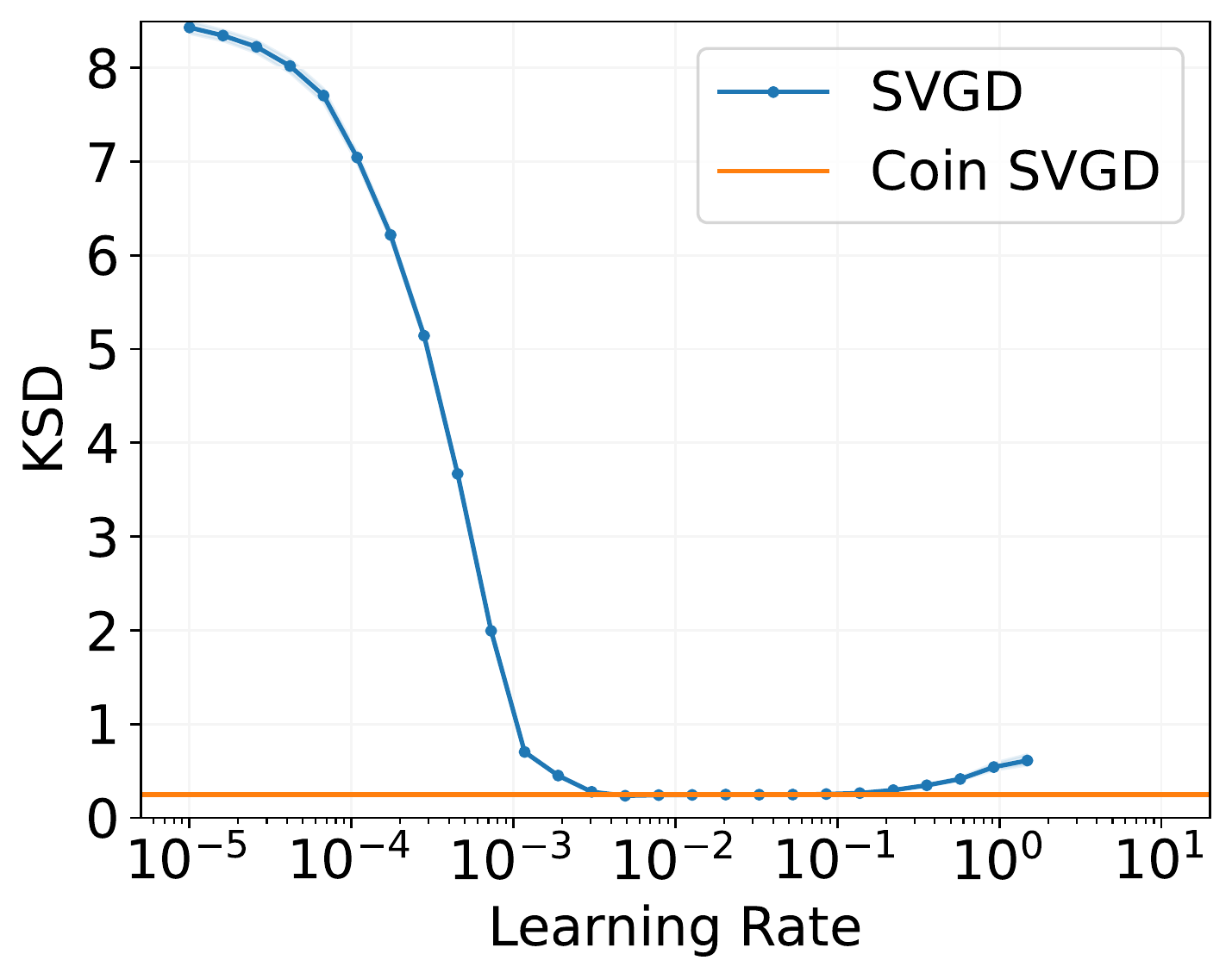}}
\vspace{-1mm}
\caption{\textbf{Additional results for the toy examples in Sec. \ref{sec:toy_examples}.} Plots of the kernel Stein discrepancy (KSD) between each of the target distributions in Fig. \ref{fig:figure1}, and the corresponding approximations generated by Coin SVGD and SVGD. We run SVGD over a grid of 30 logarithmically spaced learning rates $\gamma\in[1\times 10^{-5}, 1\times 10^{1}]$. We run both algorithms for $T=1000$ iterations, and using $N=20$ particles. The results are averaged over 50 random trials.
}
\label{fig:figure1_KSD} 
\vspace{-4mm}
\end{figure*}

\begin{figure*}[t!]
\centering
\subfigure[Gaussian.]{\includegraphics[trim=0 0 0 0, clip, width=.29\textwidth]{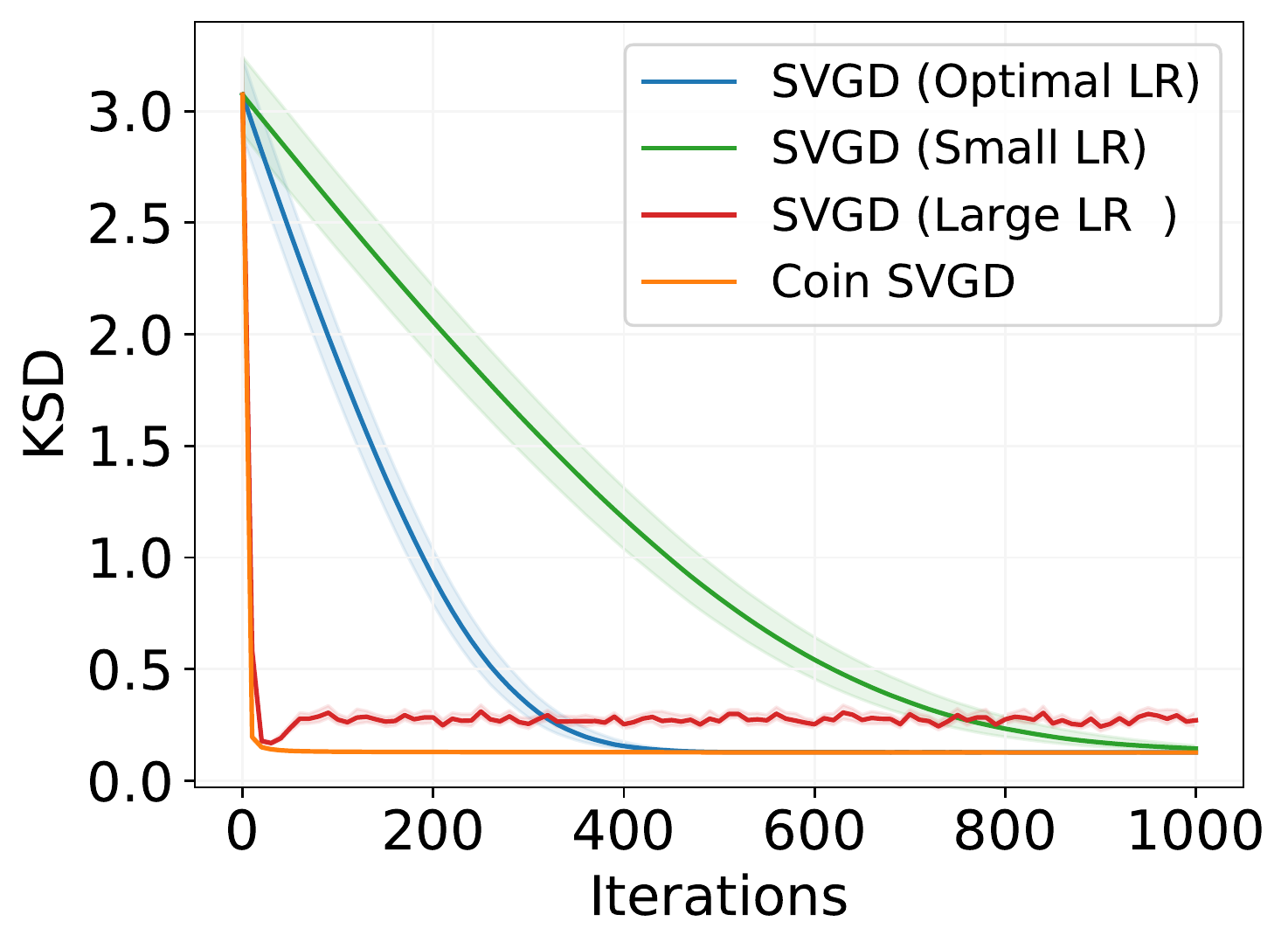}}
\hspace{2mm}
\subfigure[Mixture of Gaussians.]{\includegraphics[trim=0 0 0 0, clip, width=.29\textwidth]{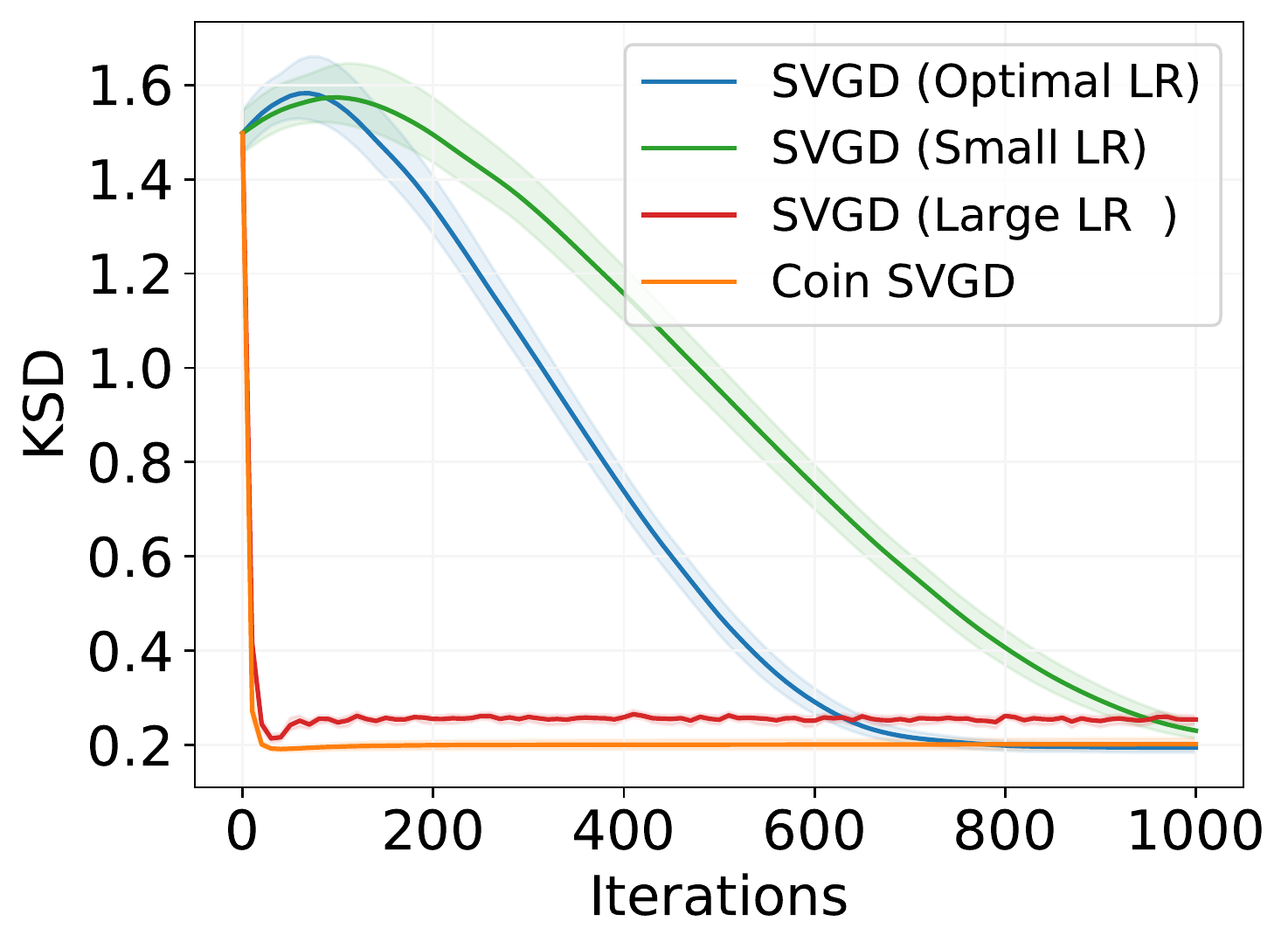}}
\hspace{2mm}
\subfigure[Donut.]{\includegraphics[trim=0 0 0 0, clip, width=.297\textwidth]{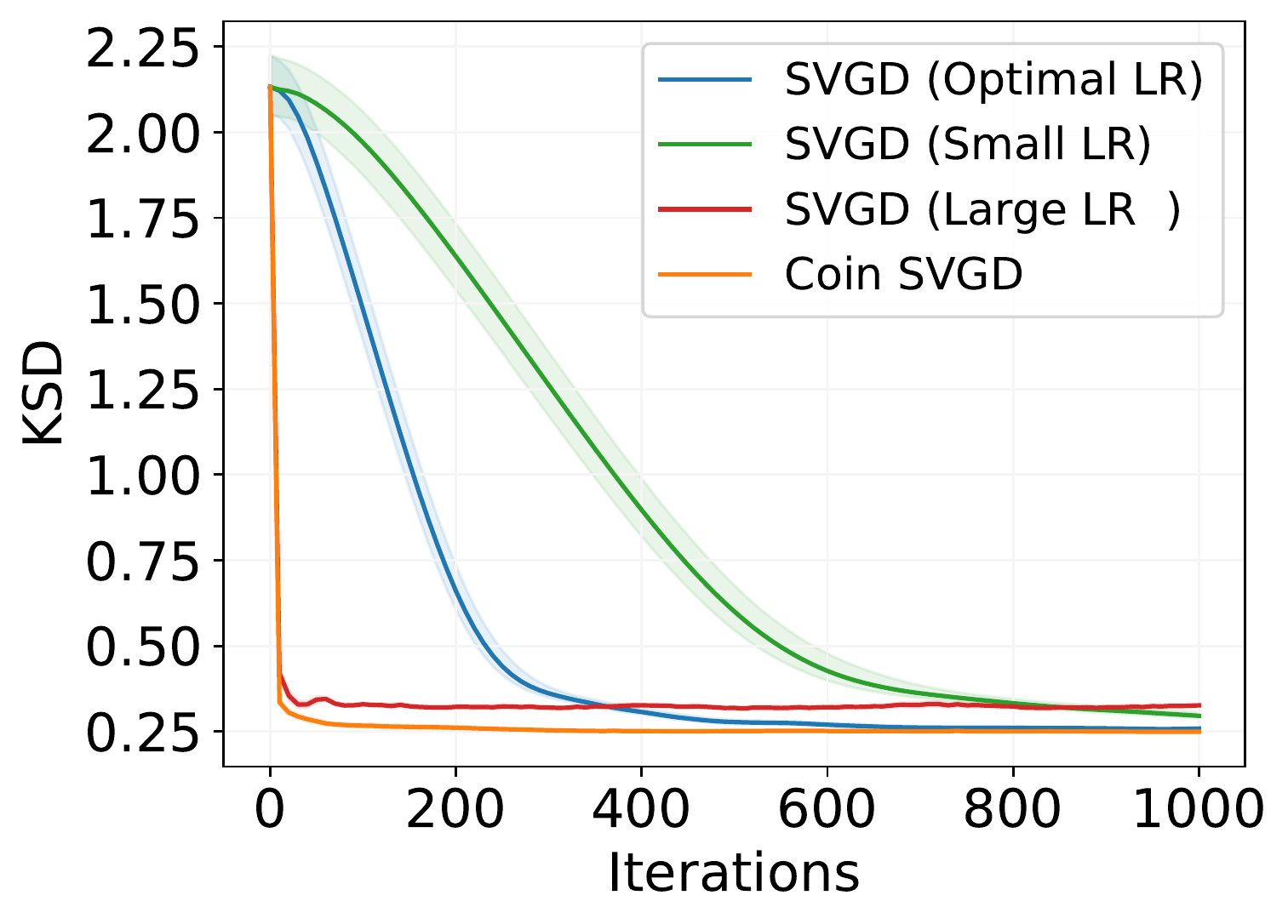}}
\subfigure[Rosenbrock Banana.]{\includegraphics[trim=0 0 0 0, clip, width=.29\textwidth]{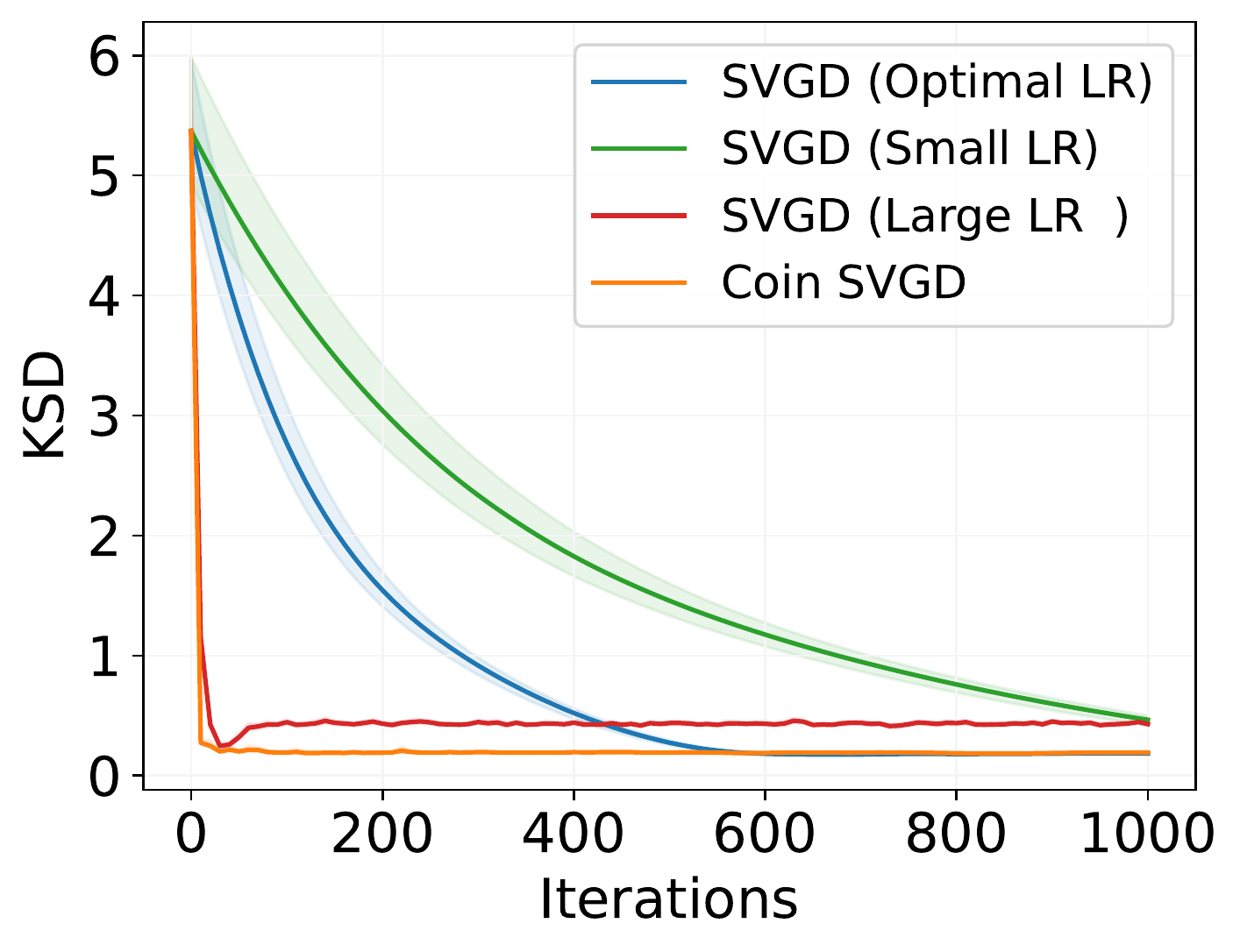}}
\hspace{2mm}
\subfigure[Squiggle.]{\includegraphics[trim=0 0 0 0, clip, width=.302\textwidth]{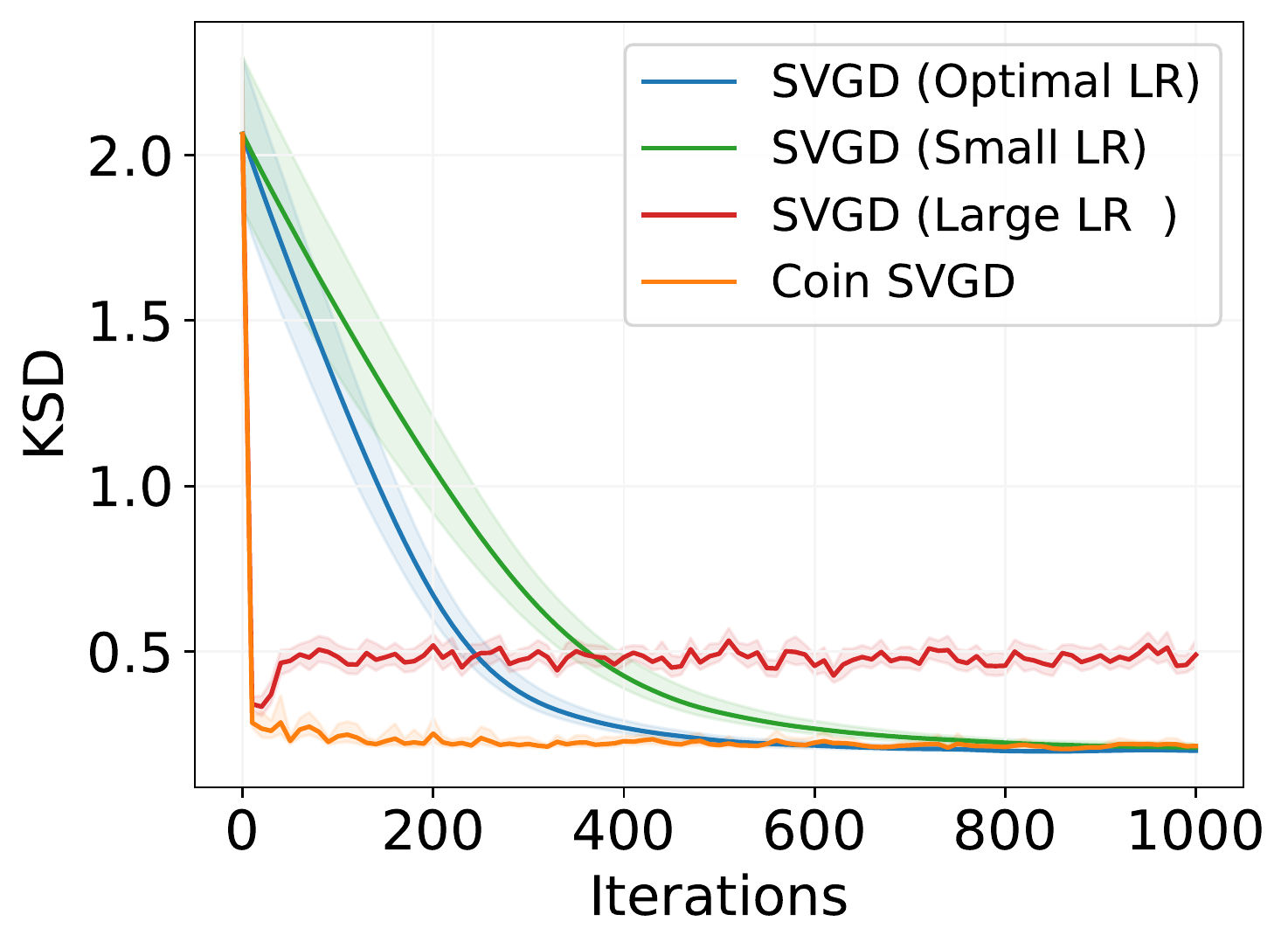}}
\hspace{2mm}
\subfigure[Funnel.]{\includegraphics[trim=0 0 0 0, clip, width=.302\textwidth]{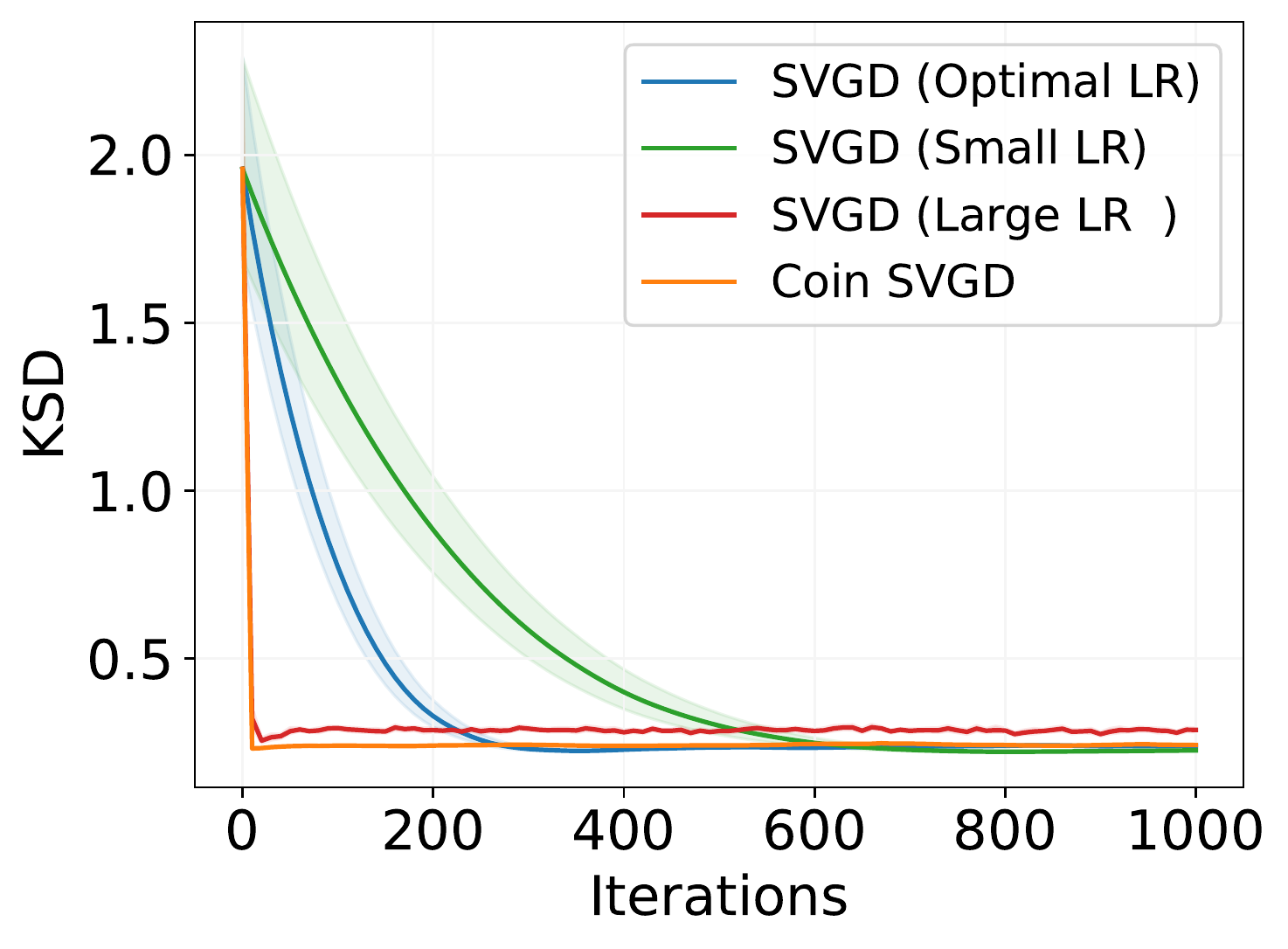}}
\caption{\textbf{Additional results for the toy examples in Sec. \ref{sec:toy_examples}.} Plots of the KSD between each of the target distributions in Fig. \ref{fig:figure1}, and the corresponding approximations generated by Coin SVGD and SVGD, as a function of the number of iterations. We run SVGD use three learning rates: the optimal learning rate as determined by Fig. \ref{fig:figure1_KSD} (blue), a smaller learning rate of $\gamma = 2\times 10^{-3}$ (green), and a larger learning rate of $\gamma = 2\times 10^{-1}$ (red). We run both algorithms for $T=1000$ iterations, and using $N=20$ particles. The results are averaged over 50 random trials.
}
\label{fig:figure1_KSD_vs_t} 
\vspace{-4mm}
\end{figure*}

\begin{figure*}[b!]
\vspace{-4mm}
\centering
\subfigure[Gaussian.]{\includegraphics[trim=0 0 0 0, clip, width=.33\textwidth]{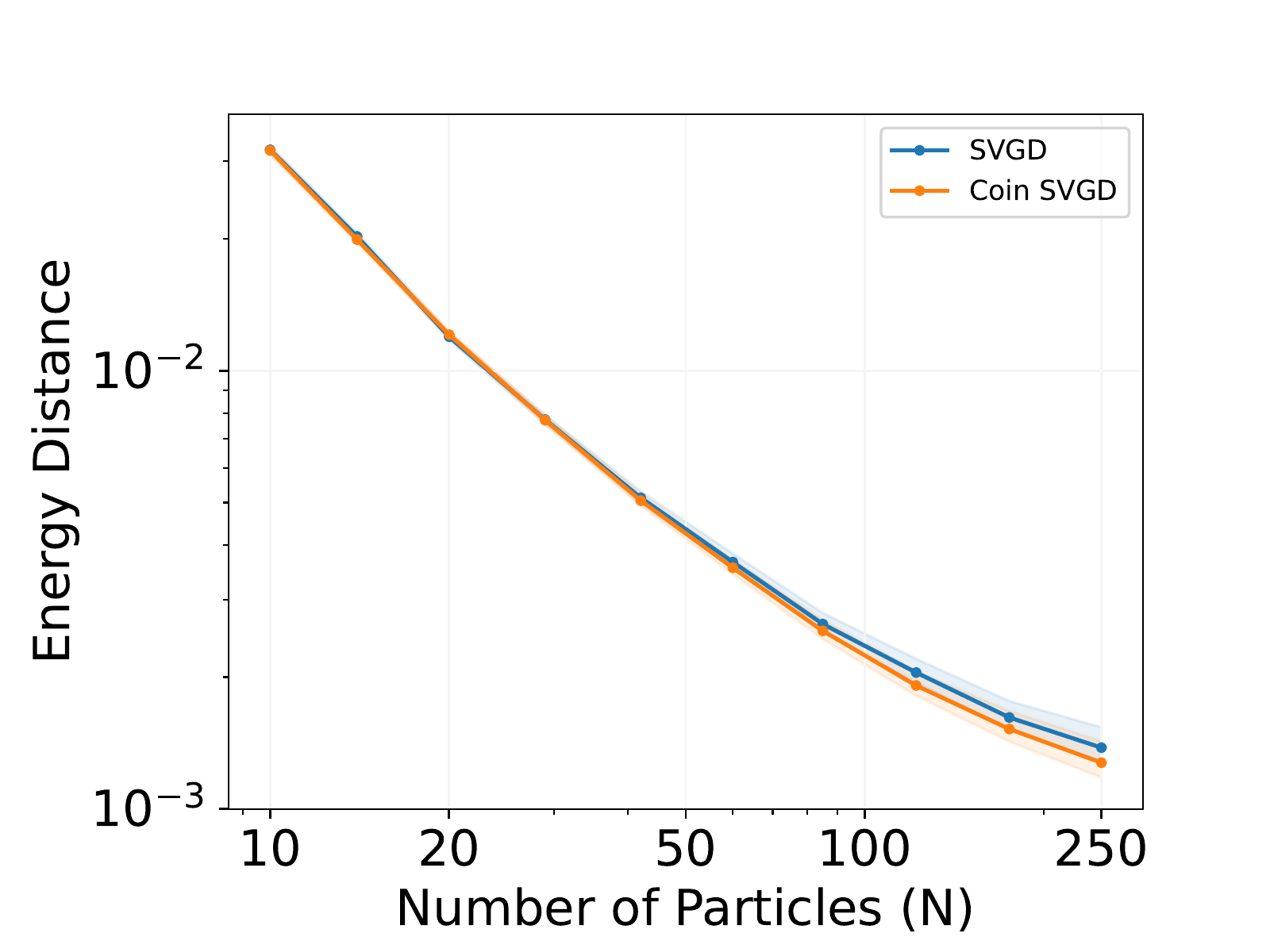}}
\subfigure[Mixture of Gaussians.]{\includegraphics[trim=0 0 0 0, clip, width=.33\textwidth]{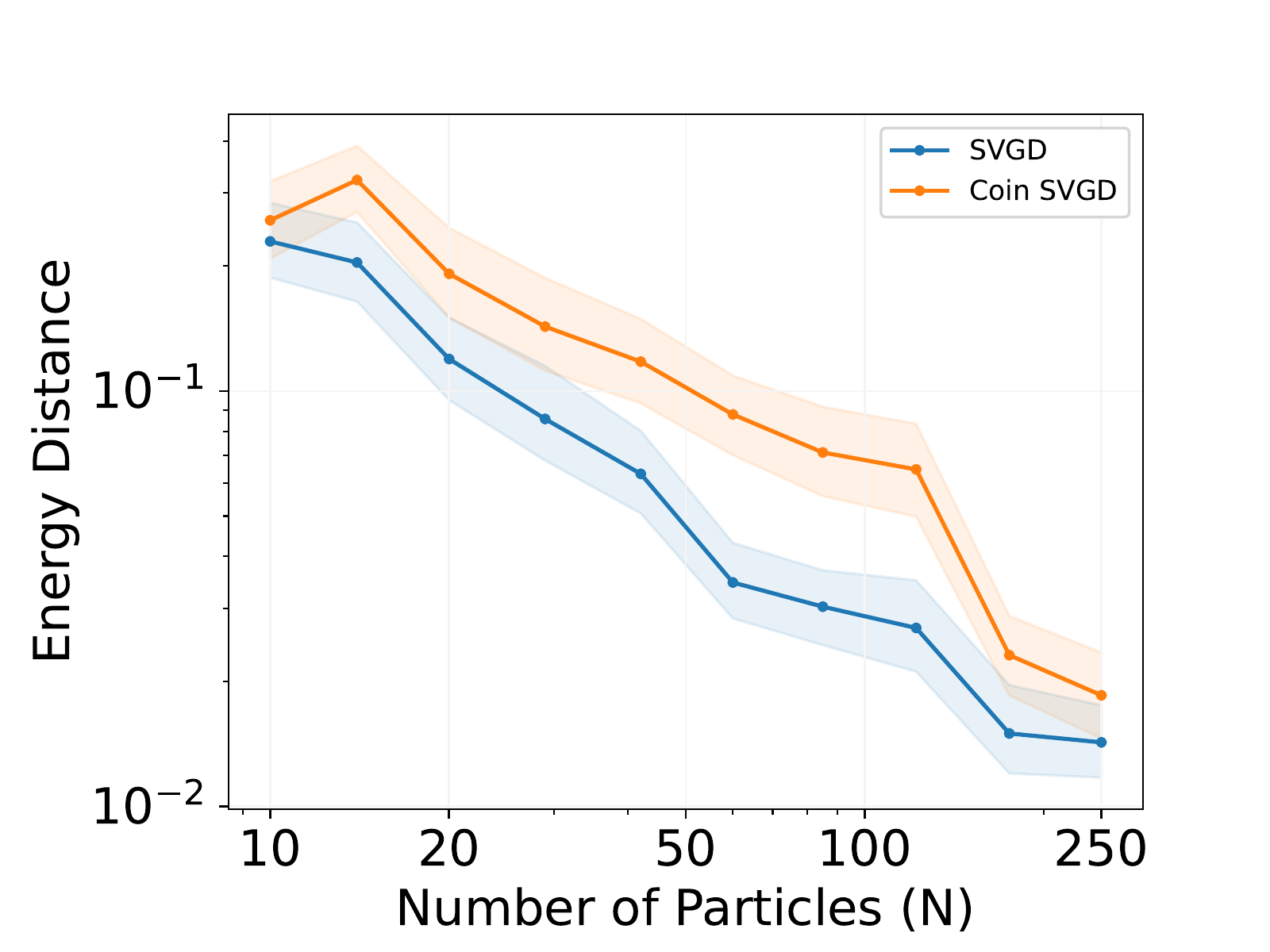}}
\subfigure[Squiggle.]{\includegraphics[trim=0 0 0 0, clip, width=.33\textwidth]{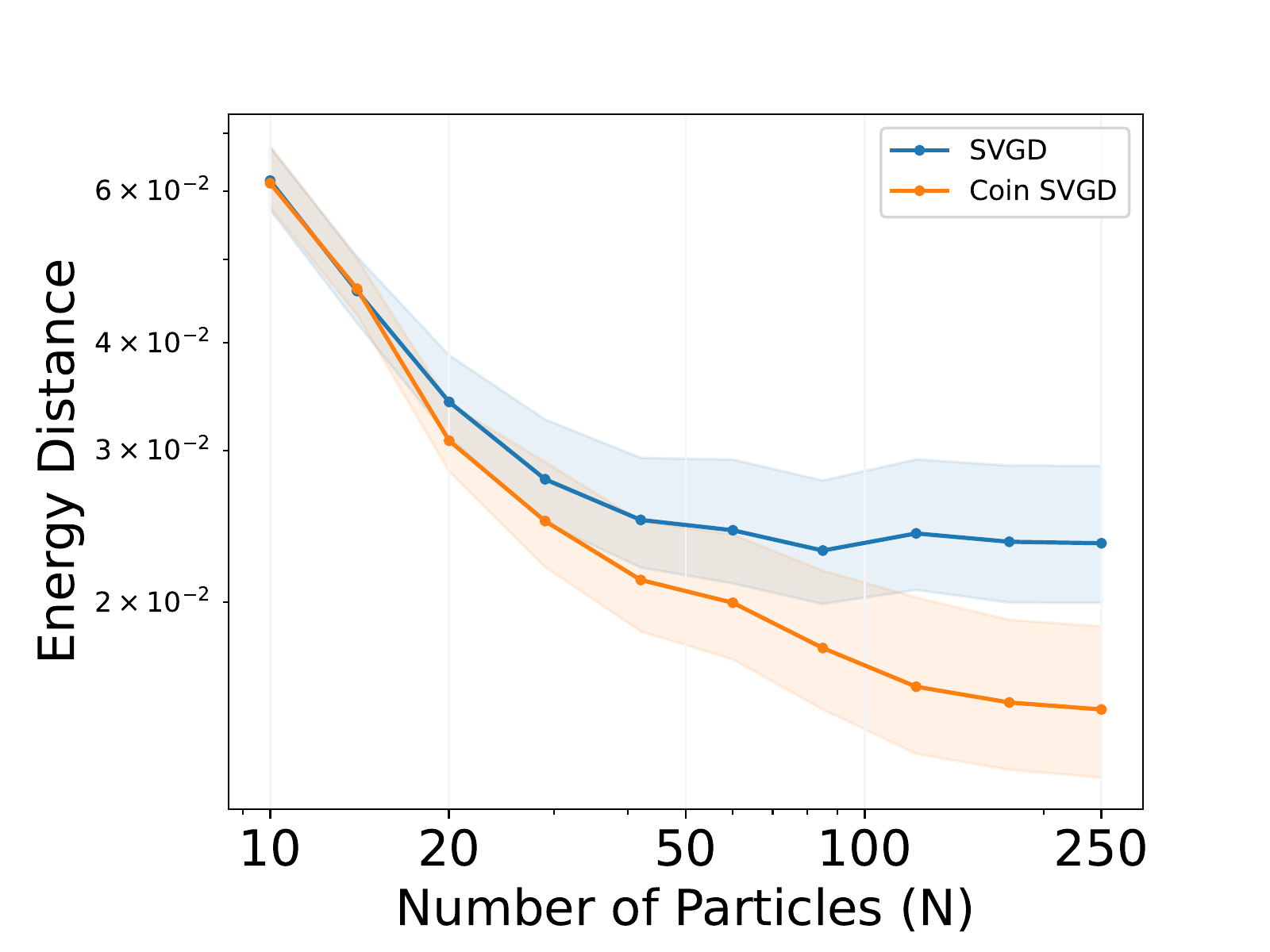}}
\vspace{-2mm}
\caption{\textbf{Additional results for the toy examples in Sec. \ref{sec:toy_examples}.} Plots of the energy distance between a subset of the target distributions in Fig. \ref{fig:figure1}, and the corresponding approximations generated by Coin SVGD and SVGD, as a function of the number of particles. In each case, we run SVGD using the best learning rate as determined by the results in Fig. \ref{fig:figure1_KSD}. We run both algorithms for $T=1000$ iterations, and using between $N=10$ and $N=250$ particles. The results are averaged over 50 random trials.
}
\label{fig:figure1_energy} 
\end{figure*}

In Fig \ref{fig:figure1_energy} and \ref{fig:KSD_v_N}, we further investigate the performance of our algorithm as a function of the number of particles $N$. In particular, we now compute the energy distance (Fig. \ref{fig:figure1_energy}) and the KSD (Fig. \ref{fig:KSD_v_N}) for SVGD and Coin SVGD as a function of $N$. In terms of tuning the learning rate used by SVGD, we consider two possibilities. The first is that the learning rate is only tuned once, using a fixed and pre-determined number of particles. This is the case in Fig. \ref{fig:figure1_energy}. More generally, this may be a realistic scenario in settings where one would like to use a large number of particles (e.g., for complex targets), but it is prohibitively expensive to tune the learning rate using this number (e.g., for high-dimensional targets, or big data settings). In the second case, the learning rate is re-tuned every time we change the number of particles.  In both cases, we determined the optimal learning rate by running SVGD over a grid of 30 learning rates $\gamma\in[1\times 10^{-5}, 1\times 10^{1}]$, and selecting the learning rate which achieves the lowest KSD after 1000 iterations (as in Fig. \ref{fig:KSD}).

For both algorithms, the energy distance and the KSD converge towards zero as the number of particles increases, suggesting that the approximate posterior samples generated provide increasingly accurate approximations of the true posterior \citep[e.g.][Theorem 8]{Gorham2017}. Interestingly, our results also suggest that the best choice of learning rate in SVGD can be somewhat sensitive to the number of particles. In particular, if one tunes the SVGD learning rate using a small number of particles (in Fig. \ref{fig:figure1_KSD}, we use $N=20$), and then runs the full algorithm using a large value of $N$ with the same learning rate, this can lead to sub-optimal performance (e.g., Fig. \ref{fig:KSD_v_Nb}, Fig. \ref{fig:KSD_v_Nd}, or Fig. \ref{fig:KSD_v_Ne}). Coin SVGD suffers no such problems, and provides us with an approach which is robust to the specification of the number of particles.

\begin{figure*}[t!]
\centering
\subfigure[Gaussian.]{\includegraphics[trim=0 0 0 0, clip, width=.3\textwidth]{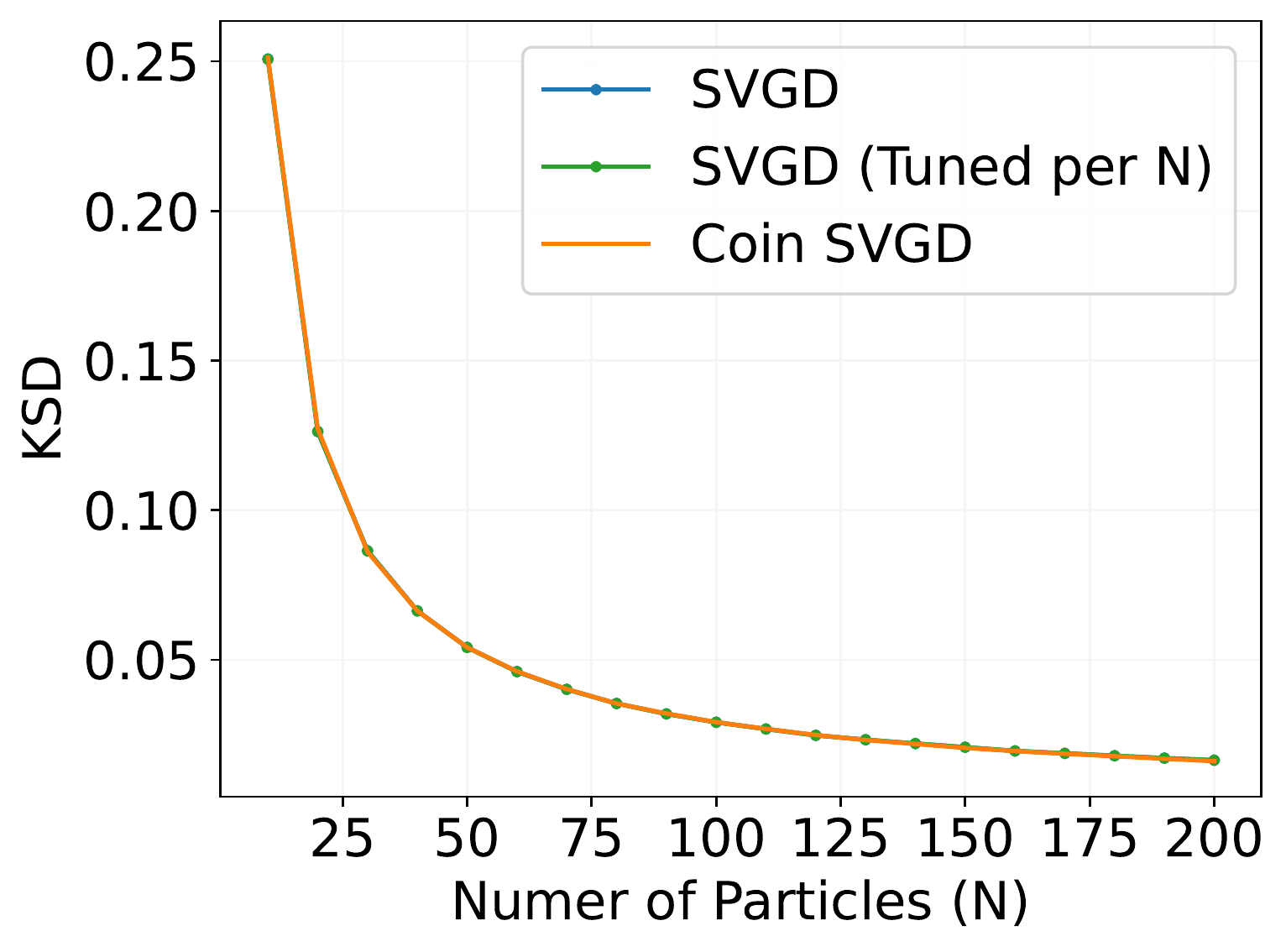}}
\subfigure[Mixture of Gaussians. \label{fig:KSD_v_Nb}]{\includegraphics[trim=0 0 0 0, clip, width=.3\textwidth]{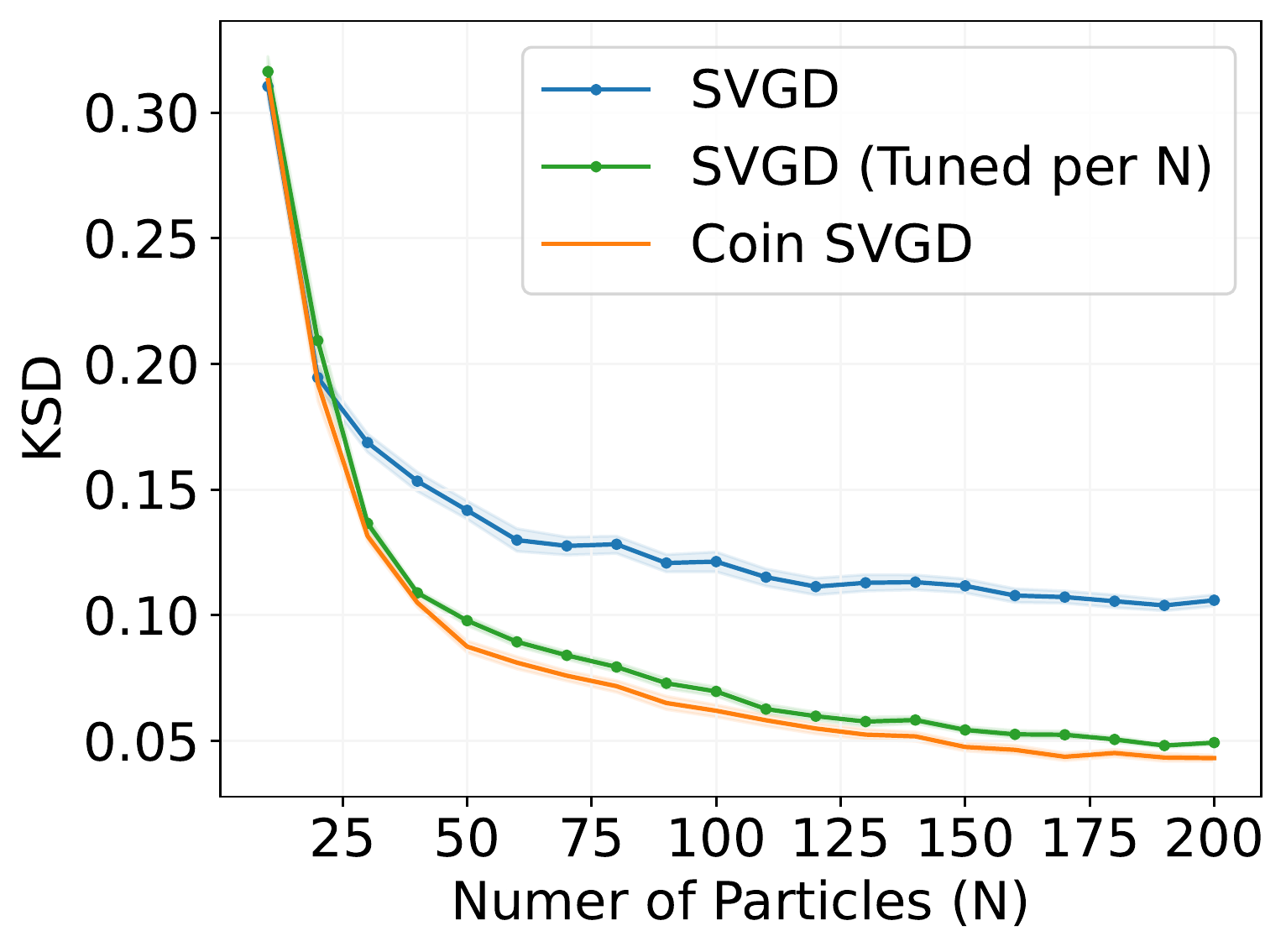}}
\subfigure[Donut.]{\includegraphics[trim=0 0 0 0, clip, width=.3\textwidth]{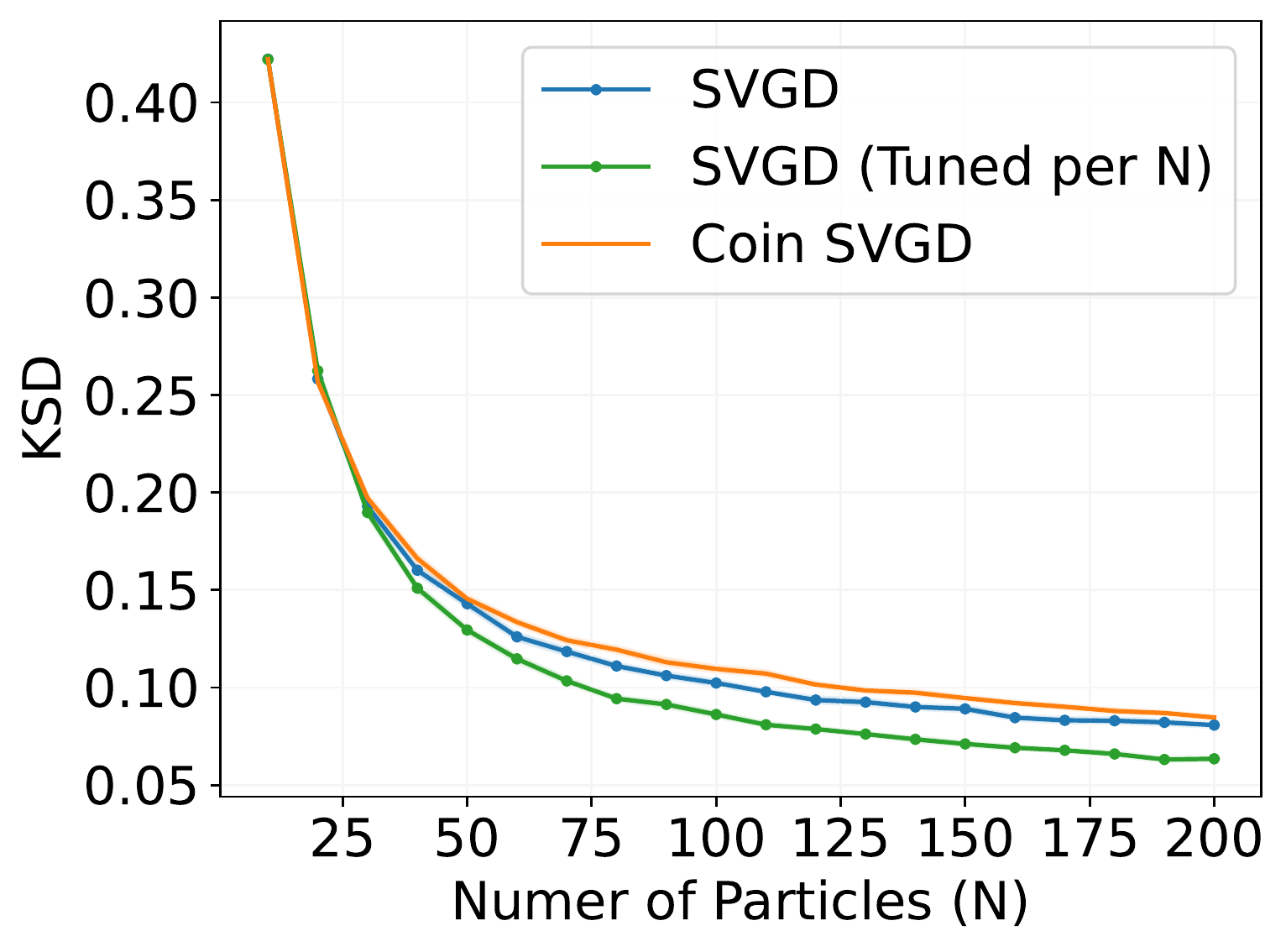}}
\subfigure[Rosenbrock Banana. \label{fig:KSD_v_Nd}]{\includegraphics[trim=0 0 0 0, clip, width=.3\textwidth]{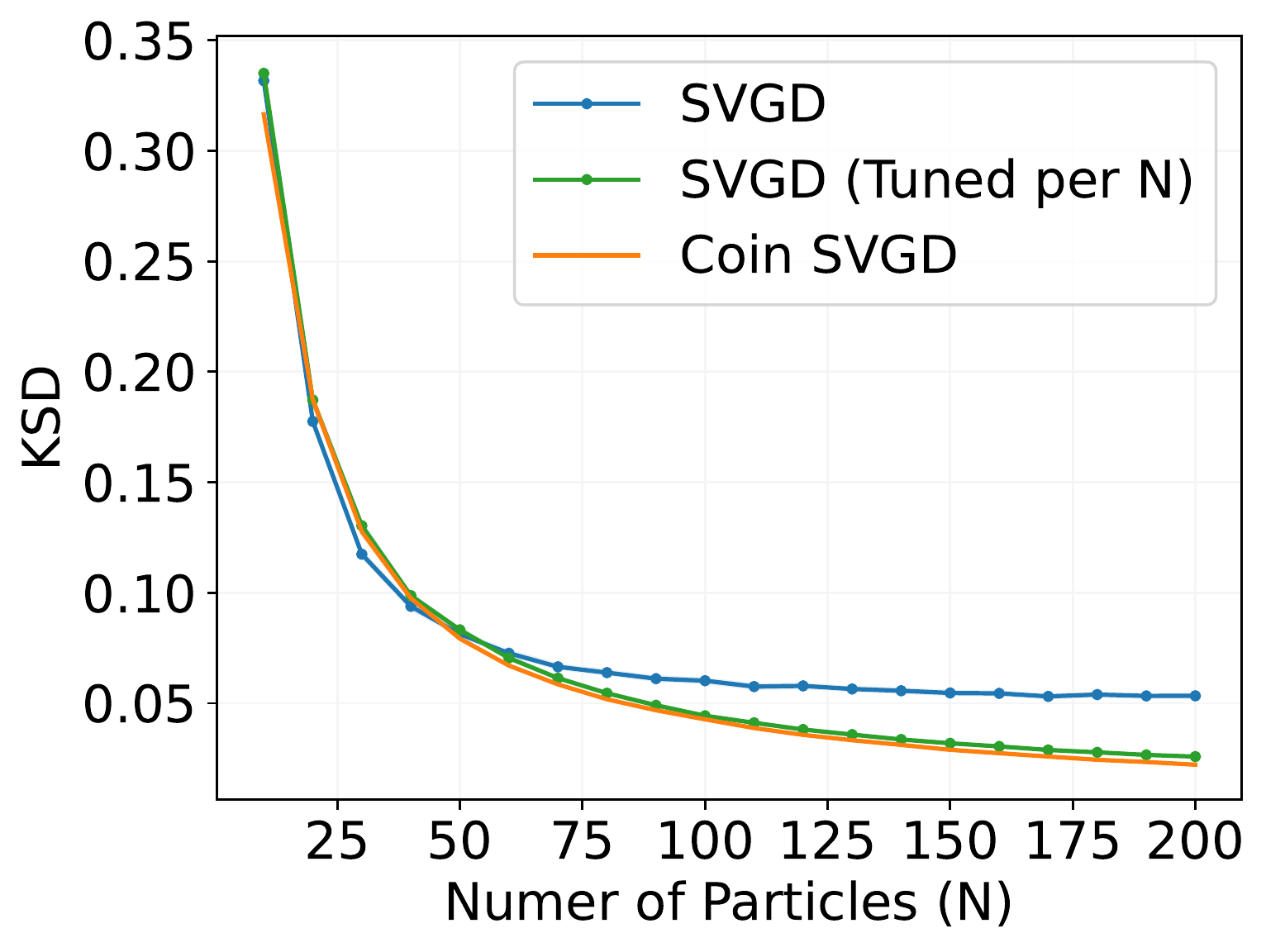}}
\subfigure[Squiggle. \label{fig:KSD_v_Ne}]{\includegraphics[trim=0 0 0 0, clip, width=.3\textwidth]{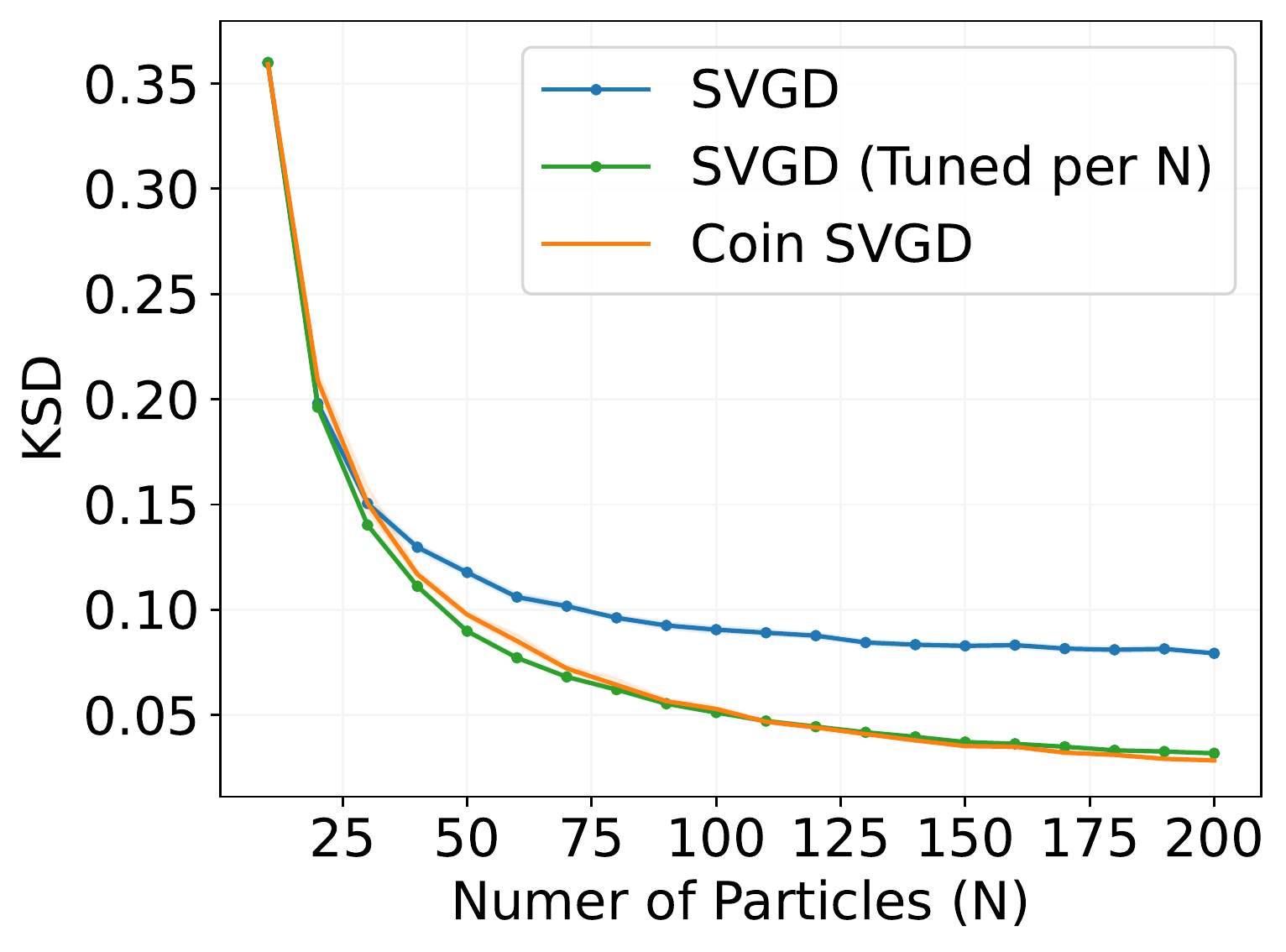}}
\subfigure[Funnel.]{\includegraphics[trim=0 0 0 0, clip, width=.3\textwidth]{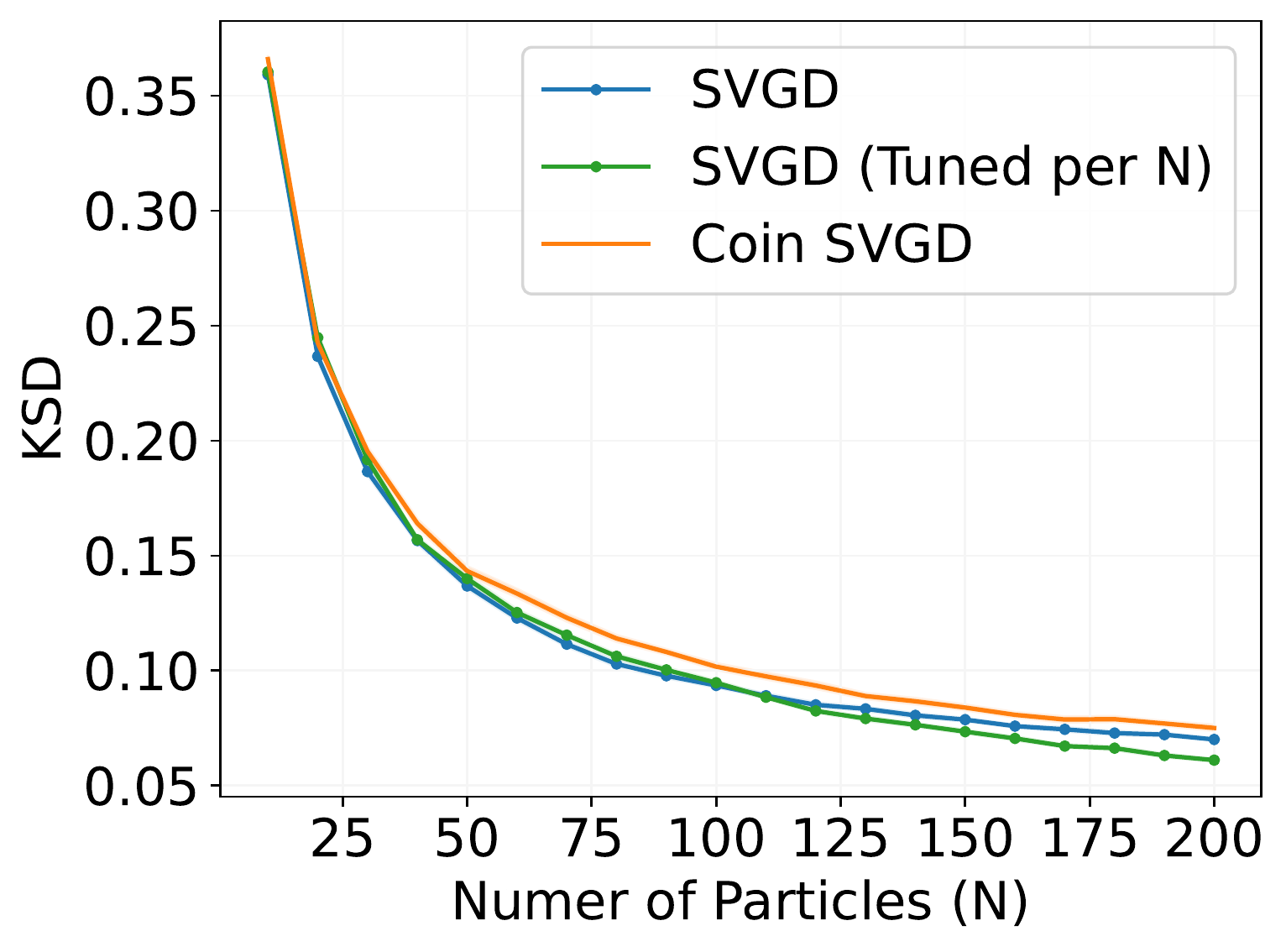}}
\caption{\textbf{Additional results for the toy examples in Sec. \ref{sec:toy_examples}.} Plots of the KSD between each of the target distributions in Fig. \ref{fig:figure1}, and the corresponding approximations generated by Coin SVGD and SVGD, as a function of the number of particles $N$. We run both algorithms for $T=1000$ iterations. For SVGD, we either tune the learning rate once, using a fixed and pre-determined $N=20$ particles (blue); or re-tune the learning rate every time we change the number of particles (green). The results are averaged over 50 random trials.
}
\label{fig:KSD_v_N} 
\vspace{-8mm}
\end{figure*}

\subsubsection{Coin LAWGD}
\label{sec:numerics_LAWGD}
\textbf{Experimental Details}. We next compare the performance of LAWGD \cite{Chewi2020} and Coin LAWGD (Alg. \ref{alg:param_free_lawgd}) on the following examples.

\emph{One-Dimensional Gaussian}. We begin by considering a simple one-dimensional Gaussian, with density $p(x) = \mathcal{N}(x;\mu,\sigma^2)$, where $\mu=3$ and $\sigma^2=1.5$. 

\emph{Mixture of Three One-Dimensional Gaussians}. We also consider a mixture of three one-dimensional Gaussians, with $p(x) = \sum_{i=1}^3 \alpha_i \mathcal{N}(x|\mu_i, \sigma_i^2)$, where we set $\alpha_1 =\frac{1}{3}$, $\mu_1=6$, $\sigma_1^2=2$; $\alpha_2 =\frac{1}{2}$, $\mu_2=-3$, and $\sigma_2^2 =1$; and $\alpha_3 =\frac{1}{6}$, $\mu_3=2$, and $\sigma_3^2 =1$. 

We use either $N=100$ particles (Fig. \ref{fig:LAWGDa}) or $N=25$ particles (Fig. \ref{fig:LAWGDb}). The initial particles are i.i.d. $\mathcal{U}(-1,1)$, or i.i.d. $\mathcal{U}(-2,2)$.  In both cases, we run the algorithms for $T=2500$ iterations.  

To compute the kernel $k_{\mathcal{L}}$ required to implement LAWGD and Coin LAWGD, it is necessary to approximate the eigenfunctions and eigenvalues of the operator $\mathcal{L}$ \citep[][Sec. 5]{Chewi2020}. Following \citet{Chewi2020}, we approximate these quantities using a basic finite difference scheme. For the target in Fig. \ref{fig:LAWGDa}, we use a finite difference scheme based on 1000 equally spaced grid points between -8 and 8, and use the first 150 eigenvalues and eigenfunctions. For the target in Fig. \ref{fig:LAWGDb}, we use 500 equally spaced grid points between -16 and 16, and again use the first 150 eigenvalues and eigenfunctions.

\textbf{Numerical Results}. In Fig. \ref{fig:LAWGD}, we plot an illustrative set of samples obtained using Coin LAWGD and LAWGD for these two examples. Similar to our other coin sampling algorithms, we see that Coin LAWGD converges to the target distribution for both of our test cases, and enjoys a similar performance to the standard LAWGD algorithm.

\begin{figure}[t!]
\centering
\subfigure[Gaussian. \label{fig:LAWGDa}]{\includegraphics[trim=0 16mm 0 16mm, clip, width=.485\textwidth]{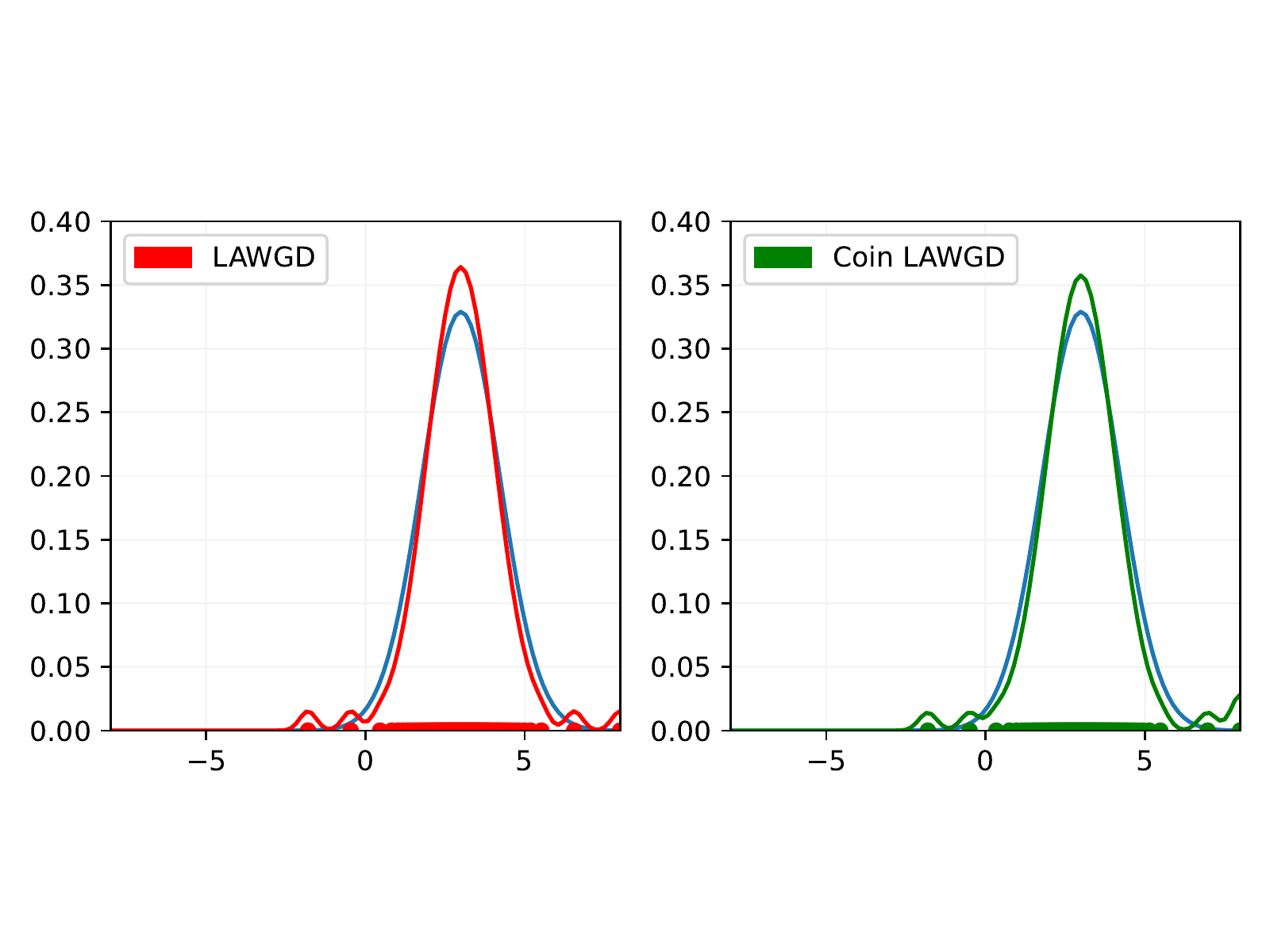}} \hfill
\subfigure[Mixture of Three Gaussians. \label{fig:LAWGDb}]{\includegraphics[trim=0 16mm 0 16mm, clip, width=.485\textwidth]{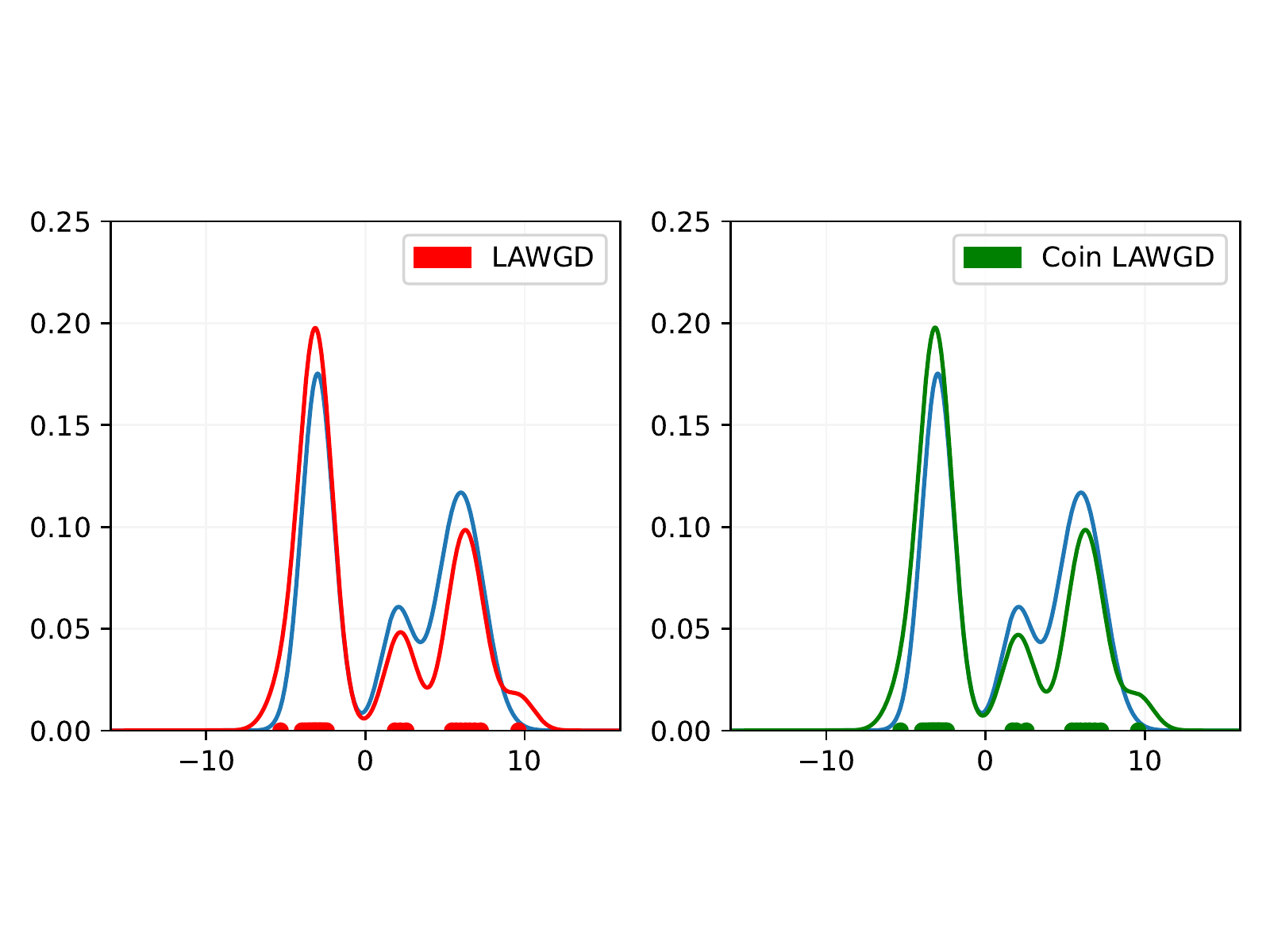}}
\caption{\textbf{A comparison between LAWGD \citep{Chewi2020} and its learning-rate free analogue, Coin LAWGD (Alg. \ref{alg:param_free_lawgd})}. We plot the samples generated by both methods for the two target distributions detailed in App. \ref{sec:numerics_LAWGD}.}
\label{fig:LAWGD}
\end{figure}

\subsubsection{Coin KSDD} 
\label{sec:numerics_KSD}
\textbf{Experimental Details} We next compare the performance of KSDD \cite{Korba2021} and Coin KSDD (Alg. \ref{alg:param_free_ksd}). We consider the following examples.

\emph{Anistropic Two-Dimensional Gaussian}. We first consider a single bivariate Gaussian, $p(x) = \mathcal{N}(x; \mu, \Sigma)$, where $\mu=(-3,3)^{\top}$ and $\Sigma^{-1} = \big( \begin{smallmatrix} 0.2 & -0.05 \\ -0.05 & 0.1 \end{smallmatrix} \big)$.

\emph{Symmetric Mixture of Two Two-Dimensional Gaussians}. For our second and third examples, we consider a symmetric mixture of two, two-dimensional, isotropic Gaussians with different covariances. In particular, $p(x) = \frac{1}{2}\mathcal{N}(x; \mu, \sigma_1^2\mathbb{1}) + \frac{1}{2}\mathcal{N}(x; -\mu,\sigma_2^2\mathbb{1})$, where $\mu=(6,0)^{\top}$, $\sigma_1^2 =2$, $\sigma_2^2 = 1$ in Fig. \ref{fig:KSDb}; and $\mu = (5,5)^{\top}$, $\sigma_1^2=2$, $\sigma_2^2=2$ in Fig. \ref{fig:KSD_anneal}. 

We use $N=20$ particles, and run both methods for $T=5000$ iterations. We initialise the particles according from $\mathcal{N}(0,0.5^2)$ in Fig. \ref{fig:KSDa}, or $\mathcal{N}(0,2^2)$ in Fig. \ref{fig:KSDb}.

\textbf{Numerical Results}. In Fig. \ref{fig:KSD}, we plot the samples obtained using KSDD and Coin KSDD after 5000 iterations. Similar to before, the samples generated by our coin sampling method are very similar to those generated by the original algorithm. In fact, even the dynamics of the two algorithms share many of the same properties. For example, the Coin KSDD particles seem initially to be guided by the final repulsive term in the update, which determine their global arrangement. They are then transported towards the mode(s), driven by the remaining score-based terms. This is in contrast to the Coin SVGD particles, which are first driven by the score term, before being dispersed around the mode by the repulsive term. These dynamics were first observed in \citet{Korba2021} for the standard SVGD and KSDD algorithms, and also to be present for their step-size free analogues. 

\begin{figure}[b!]
\centering
\subfigure[Anisotropic Gaussian. \label{fig:KSDa}]{\includegraphics[trim=0 16mm 0 16mm, clip, width=.46\textwidth]{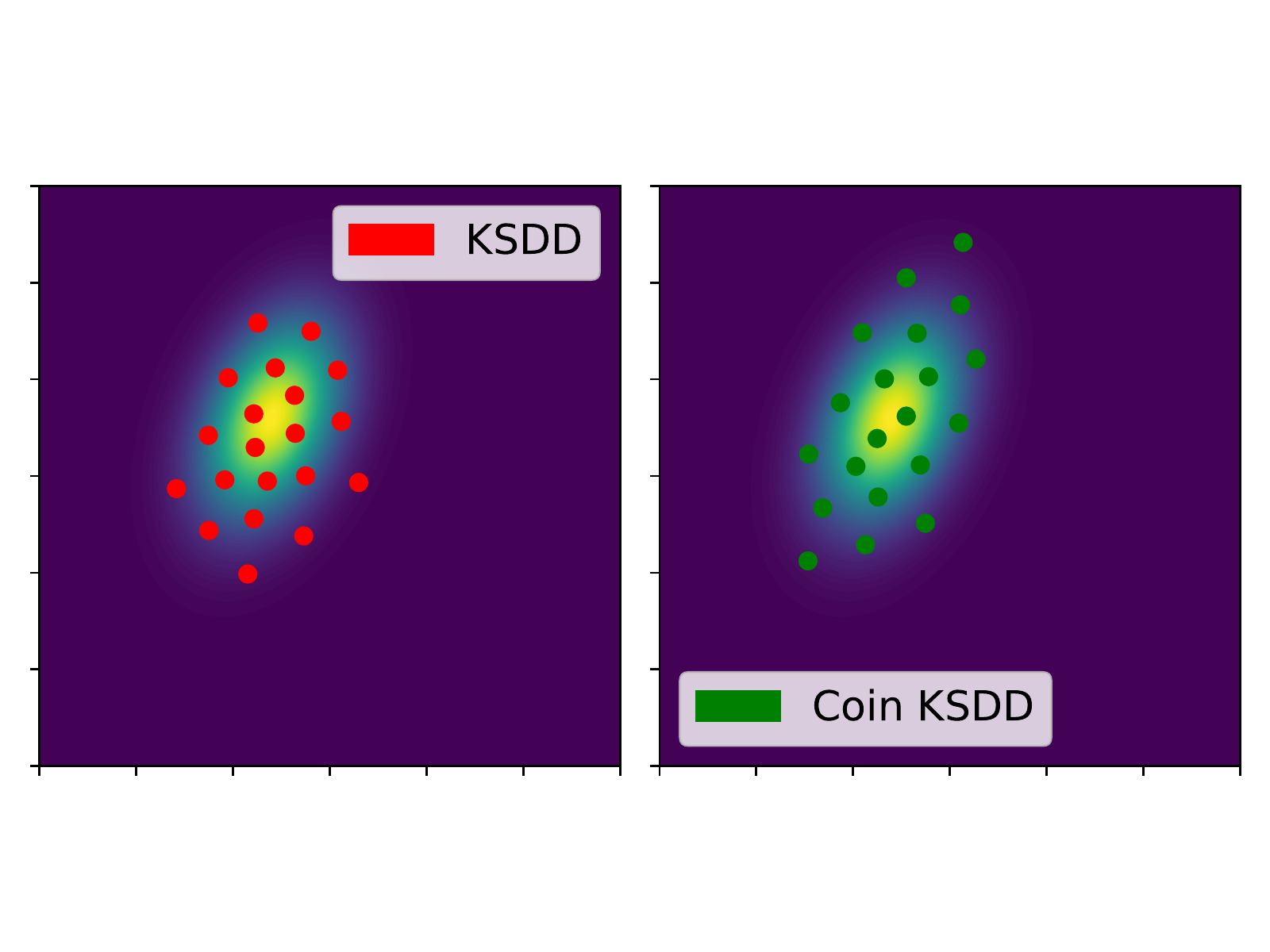}} \hspace{4mm}
\subfigure[Symmetric Mixture of Gaussians. \label{fig:KSDb}]{\includegraphics[trim=0 16mm 0 16mm, clip, width=.46\textwidth]{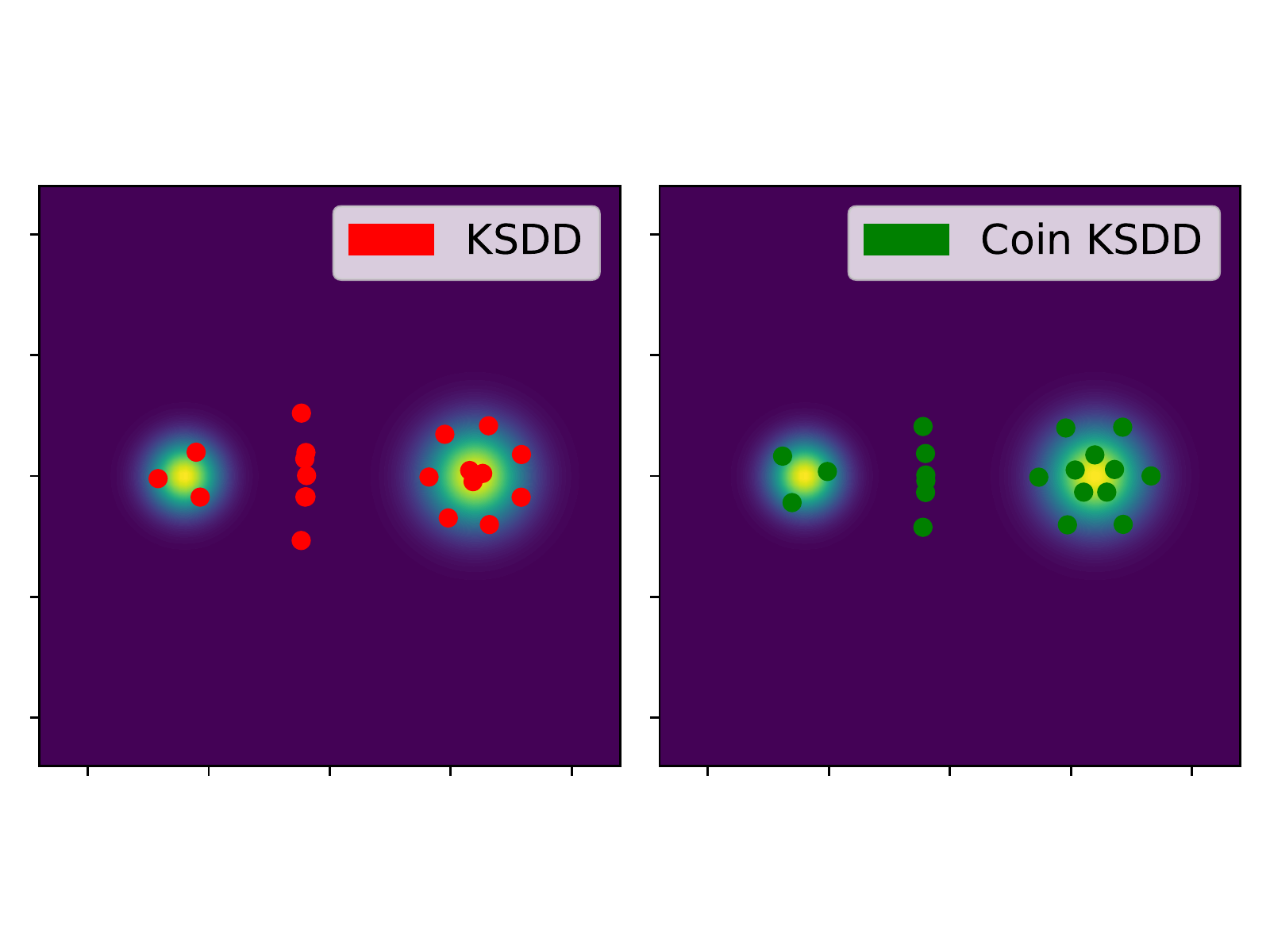}}
\caption{\textbf{A comparison between KSDD \cite{Korba2021} and its learning-rate free analogue, Coin KSDD (Alg. \ref{alg:param_free_ksd})}. We plots the samples generated by both methods for the two target distributions detailed in App. \ref{sec:numerics_KSD}.}
\label{fig:KSD}
\end{figure}

Unsurprisingly, Coin KSDD also inherits some of the shortcomings of KSDD. Given a symmetric target, and a radial kernel, it is known that any plane of symmetry is invariant under the KSD gradient flow \citep[Lemma 11]{Korba2021}. Thus, if KSDD is initialised close to a plane of symmetry, it can become stuck there indefinitely. In practice, this also appears to holds true for Coin KSDD (see Fig. \ref{fig:KSD}). \citet{Korba2021} propose an annealing strategy can be used to resolve this behaviour; see also \citet{Wenliang2021}. One first runs KSDD to obtain samples from the target $\pi_{\beta}(x)\propto \exp(-\beta U(x))$, where the inverse temperature $\beta\sim 0$. One then runs the algorithm a second time, initialised at these samples, on the true target $\pi(x)\propto \exp(- U(x))$. A similar strategy can also be used for Coin KSDD (see Fig. \ref{fig:KSD_anneal}).

\begin{figure}[t]
\centering
\subfigure[$\beta=1$]{\includegraphics[trim=0 16mm 0 16mm, clip, width=.33\textwidth]{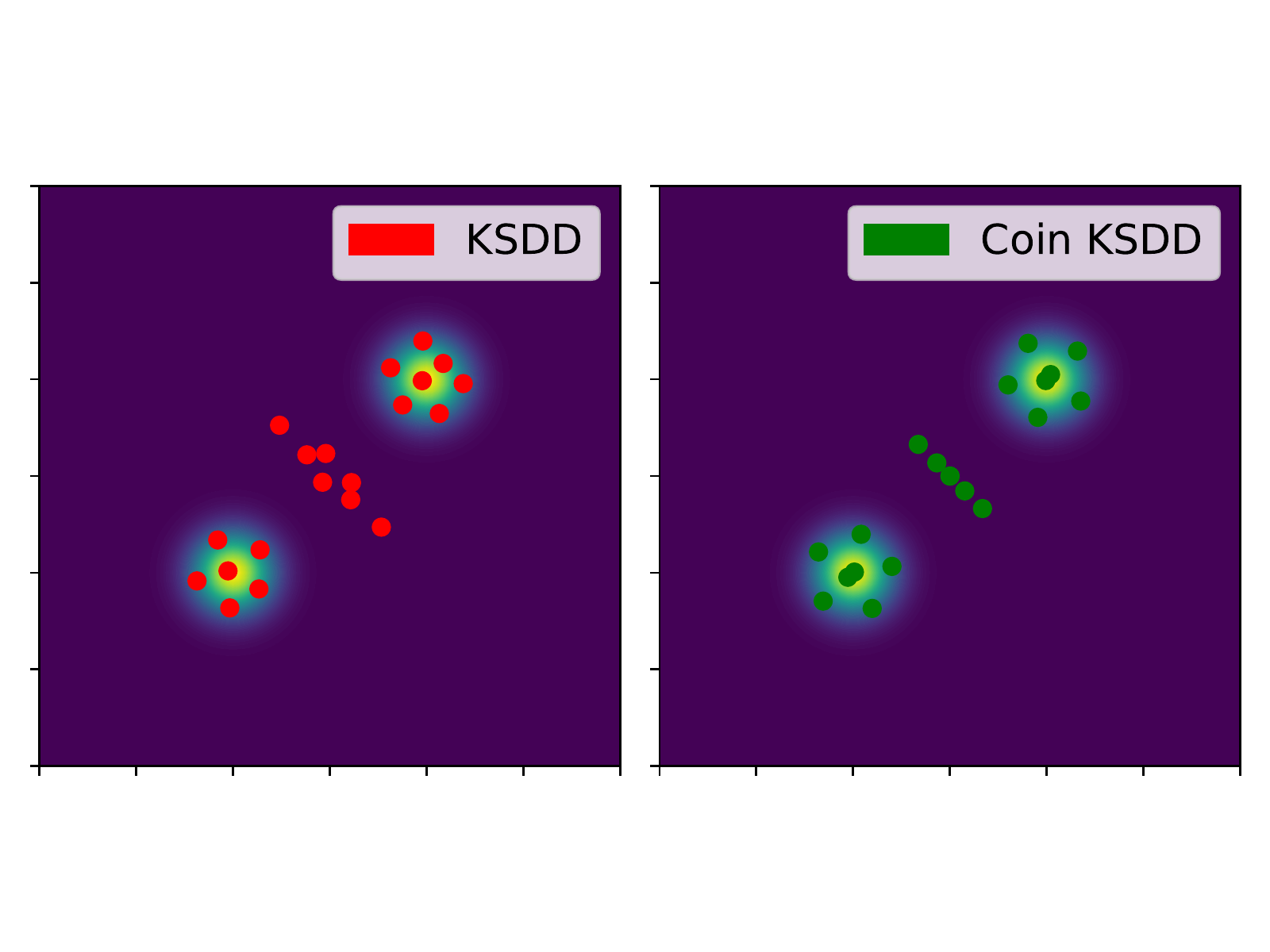}}
\subfigure[$\beta=0.02$]{\includegraphics[trim=0 16mm 0 16mm, clip, width=.33\textwidth]{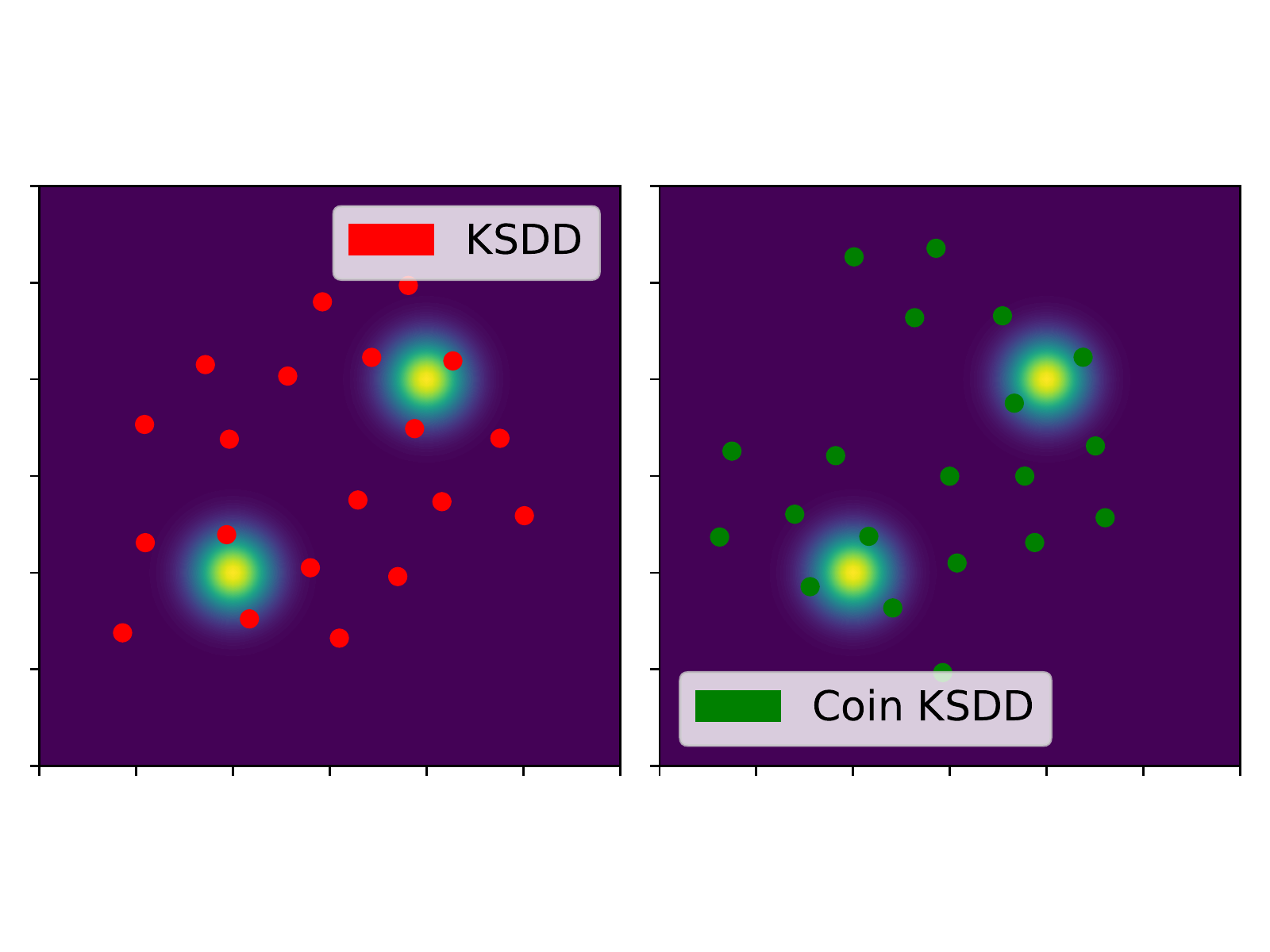}}
\subfigure[$\beta=0.02 \rightarrow 1$.]{\includegraphics[trim=0 16mm 0 16mm, clip, width=.33\textwidth]{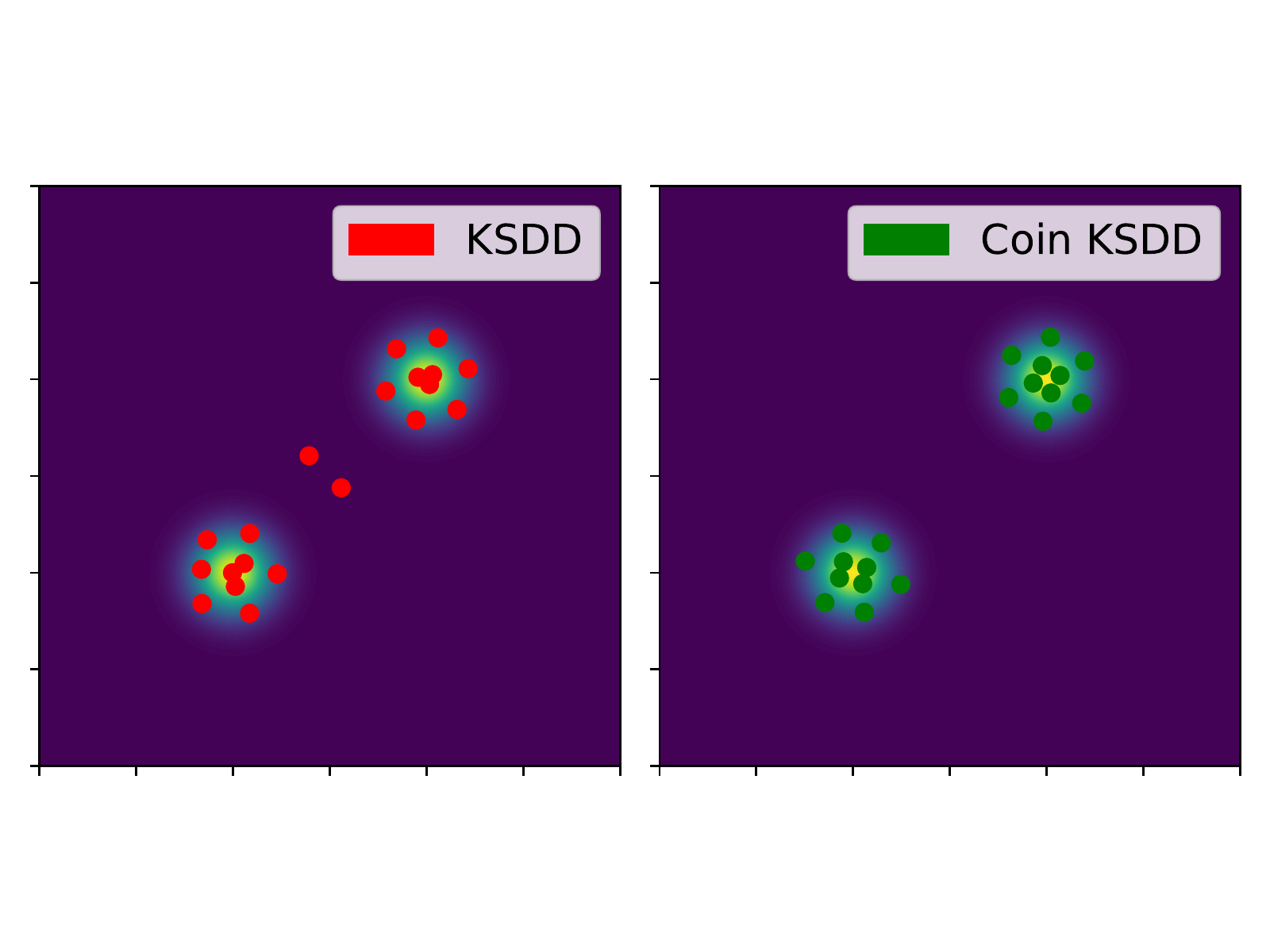}}
\caption{\textbf{A comparison between annealing KSDD \cite{Korba2021} and annealing Coin KSDD (Alg. \ref{alg:param_free_ksd})}. Samples generated by both methods using no annealing ($\beta=1$), after the first step of the annealing method ($\beta=0.02$), and after the full annealing method ($\beta=0.02\rightarrow\beta=1$).}
\label{fig:KSD_anneal}
\vspace{-4mm}
\end{figure}

\subsection{Bayesian Independent Component Analysis}
\label{sec:bayes-ica-add}

\textbf{Additional Experimental Details}. For the results in Fig. \ref{fig:BayesICA} (Sec. \ref{sec:bayes_ica}) and Fig. \ref{fig:bayes-ica-add} (below), we tune the SVGD learning rate by running SVGD for $T=1000$ iterations, using learning rates $\gamma\in\{1\times 10^{-5}, 1\times 10^{-4}, 1\times 10^{-3}, 1\times 10^{-2}, 1\times 10^{-1}, 1\times 10^{0}\}$. We then define the optimal learning rate as the one for which SVGD outputs approximate unmixing matrices $(\bar{\boldsymbol{W}}_i)_{i=1}^{10}$ with the lowest Amari distance to the true unmixing matrix $\boldsymbol{W}$, averaged over 10 random trials. The small and large learning rates are then chosen to be one order of magnitude smaller or greater than the optimal learning rate, respectively.

\textbf{Additional Numerical Results}. 
In Fig. \ref{fig:bayes-ica-add}, we provide a further comparison between Coin SVGD and SVGD, plotting the KSD (Fig. \ref{fig:bayes-ica-add-a}) and the clock time (Fig. \ref{fig:bayes-ica-add-b}) as a function of the number of particles $N$, in the case $p=4$. Similar to elsewhere, the performance of Coin SVGD is similar to the best performance of SVGD (out of the learning rates considered), and both algorithms provide increasingly accurate approximations of the posterior as the number of particles increases. The computational cost of both algorithms is essentially identical.

\begin{figure}[b!]
\centering
\subfigure[KSD vs Number of Particles. \label{fig:bayes-ica-add-a}]{\includegraphics[trim=0 0mm 0 0mm, clip, width=.4\columnwidth]{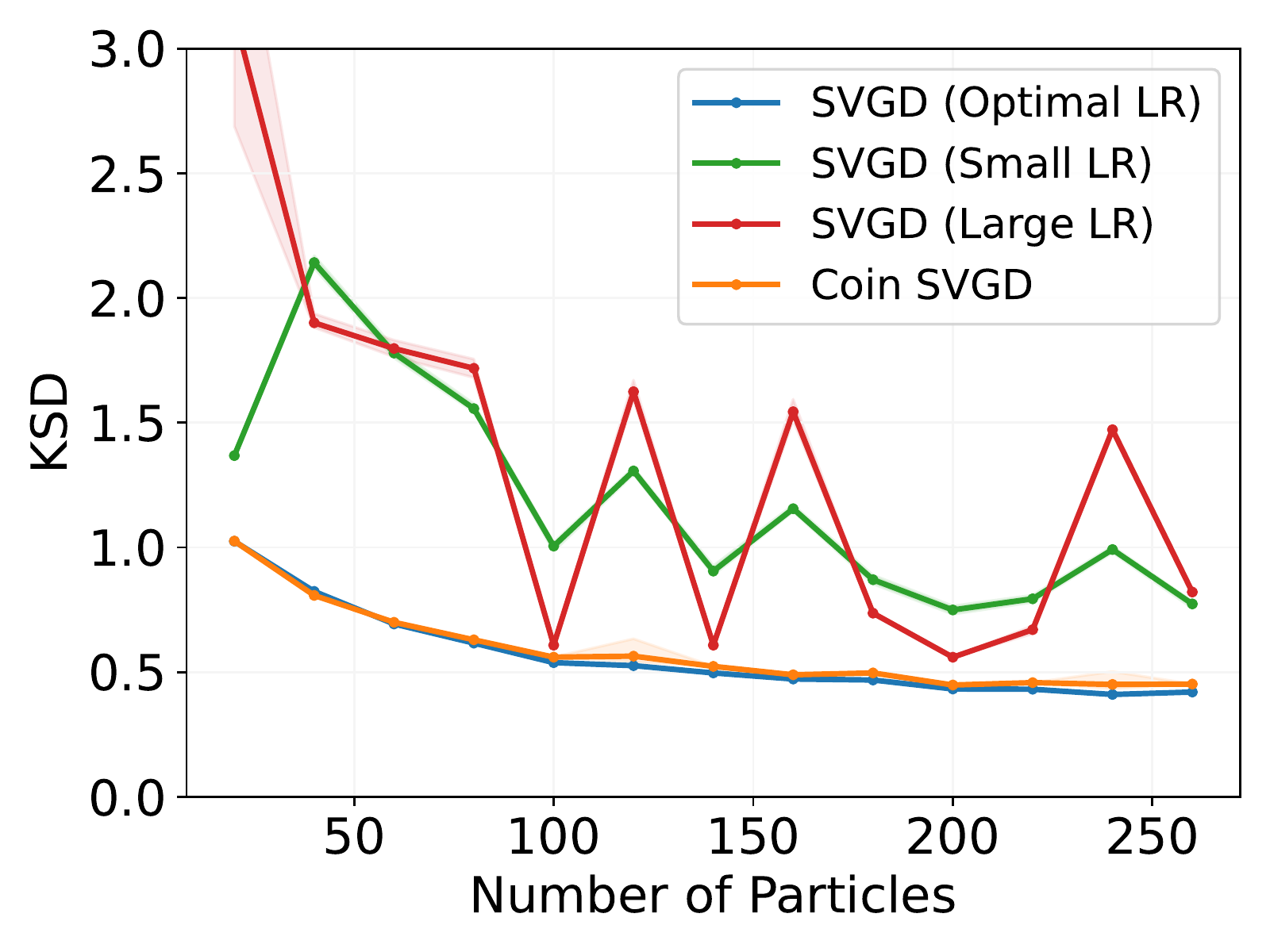}} \hspace{8mm}
\subfigure[Time vs Number of Particles. \label{fig:bayes-ica-add-b}]{\includegraphics[trim=0 0mm 0mm 0mm, clip, width=.4\columnwidth]{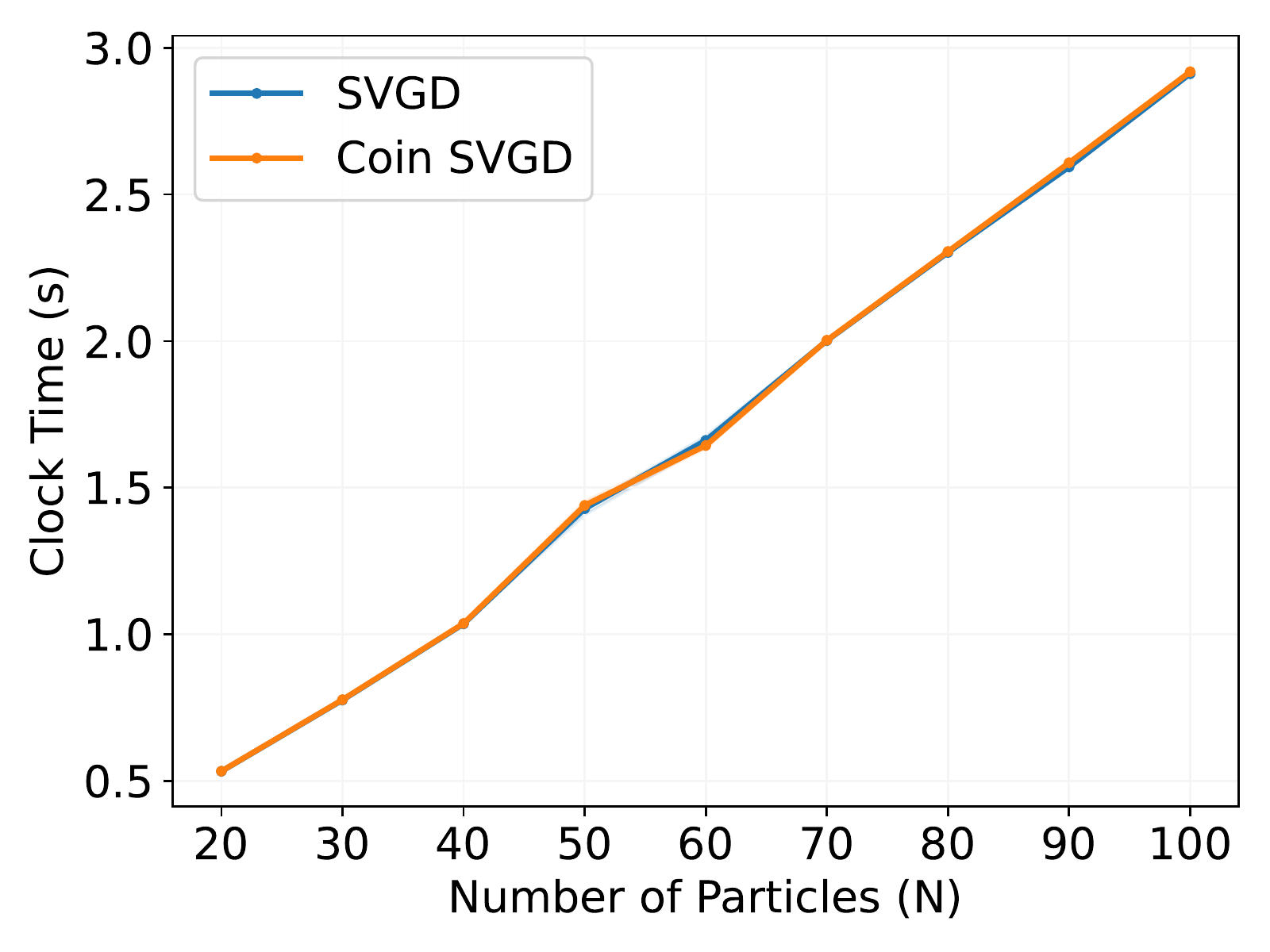}}
\caption{\textbf{Additional results for the Bayesian ICA model}. (a) KSD between the target posterior and the posterior approximations generated by Coin SVGD and SVGD after $T=1000$ iterations, as a function of the number of particles $N$. (b) Time (s) to run Coin SVGD and SVGD for $T=1000$ iterations as a function of the number of particles. For both sets of results, we average the results over 10 random trials.}
\label{fig:bayes-ica-add}
\vspace{-4mm}
\end{figure}

\subsection{Bayesian Logistic Regression}
\label{sec:bayes-lr-extra}

\textbf{Additional Experimental Details}. For the results in Fig. \ref{fig:BayesLRa} and Fig. \ref{fig:BayesLRb}, we run SVGD using a grid of 10 logarithmically spaced learning rates $\gamma\in[1\times 10^{-5},1\times 10^{0}]$. For the results in Fig. \ref{fig:BayesLRc} and Fig. \ref{fig:BayesLRd}, we then define the `optimal' learning rate to be the one which obtained the highest test accuracy after $T=5000$ iterations. Meanwhile, the `small' and `big' learning rates correspond to those with indices two smaller or two larger than the index of the optimal learning rate in the original grid of learning rates. As in \citet{Liu2016a}, we use Adagrad \cite{Duchi2011} to adapt the learning rate for SVGD on the fly.

\textbf{Additional Numerical Results}. In Fig. \ref{fig:bayes-lr-add}, we provide additional numerical results for the Bayesian logistic regression considered in Sec. \ref{sec:bayes-lr}. In Fig. \ref{fig:bayes-lr-add-a}, we plot the KSD between the posterior approximations generated by SVGD and Coin SVGD, and the true target posterior, as a function of the number of particles. For SVGD, we consider several fixed learning rates, namely $\gamma \in\{2\times 10^{-2}, 5\times 10^{-2}, 1\times 10^{-1}\}$ across all values of $N$, which are determined based on the results in Fig. \ref{fig:BayesLRa} and Fig. \ref{fig:BayesLRb}. Our results suggest, as we would expect, that the KSD achieved by both methods decreases as a function of the number of particles; i.e., both methods generate increasingly accurate approximations of the target posterior as the number of particles increases. Meanwhile, the performance of Coin SVGD is broadly comparable to the best performance of SVGD (among the learning rates considered), as measured by the KSD. 

In Fig. \ref{fig:bayes-lr-add-b}, we plot the time taken by Coin SVGD and SVGD to complete $T=2500$ iterations, again as a function of the number of particles. Here, we see no meaningful difference between the two algorithms. This is unsurprising on the basis of our earlier discussion on computational cost in Sec. \ref{sec:adaptive-coin-svgd}.

\begin{figure}[t!]
\centering
\subfigure[KSD vs Number of Particles. \label{fig:bayes-lr-add-a}]{\includegraphics[trim=0 0mm 0 0mm, clip, width=.38\columnwidth]{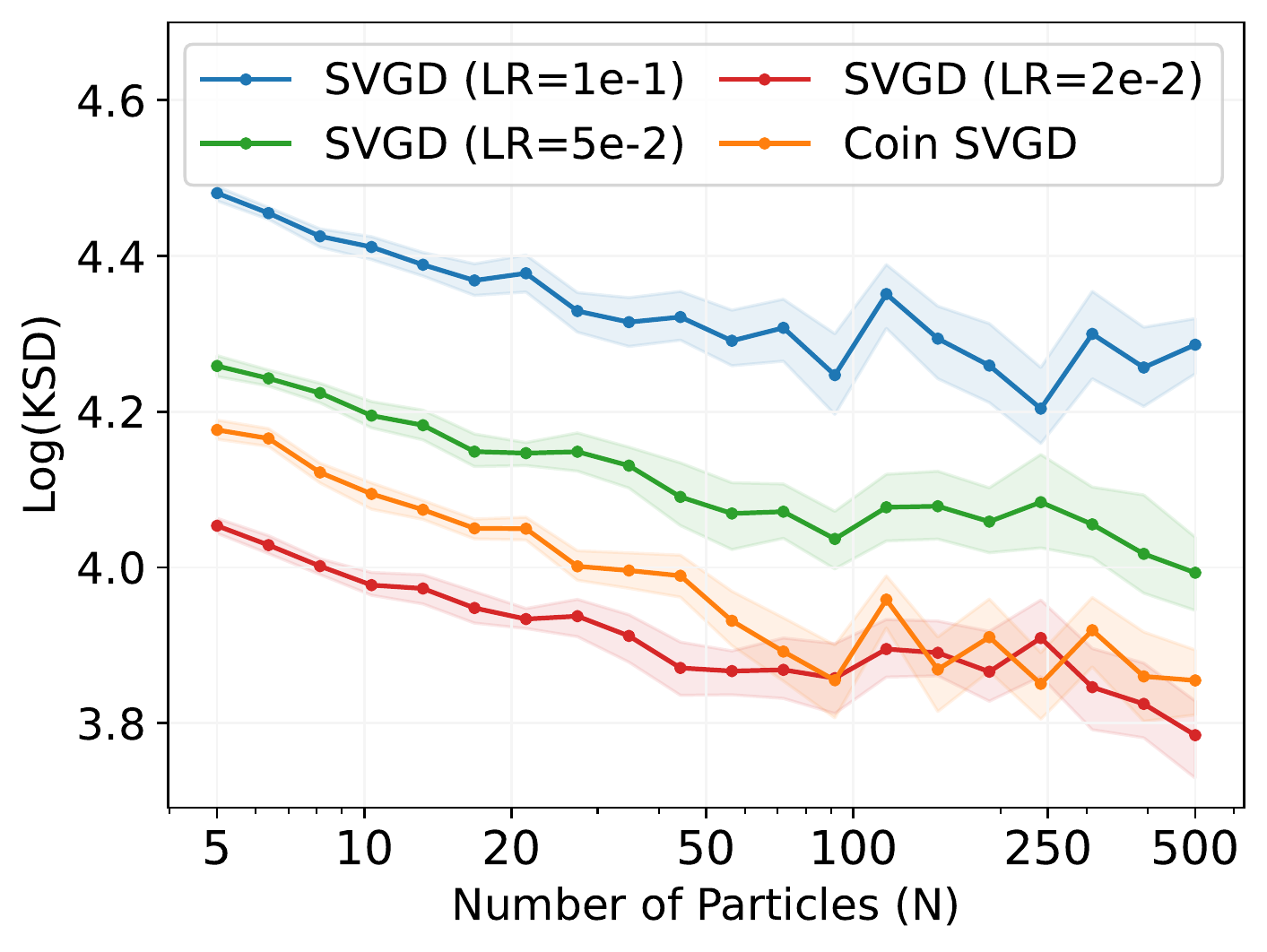}} \hspace{8mm}
\subfigure[Time vs Number of Particles. \label{fig:bayes-lr-add-b}]{\includegraphics[trim=0 0mm 0mm 0mm, clip, width=.38\columnwidth]{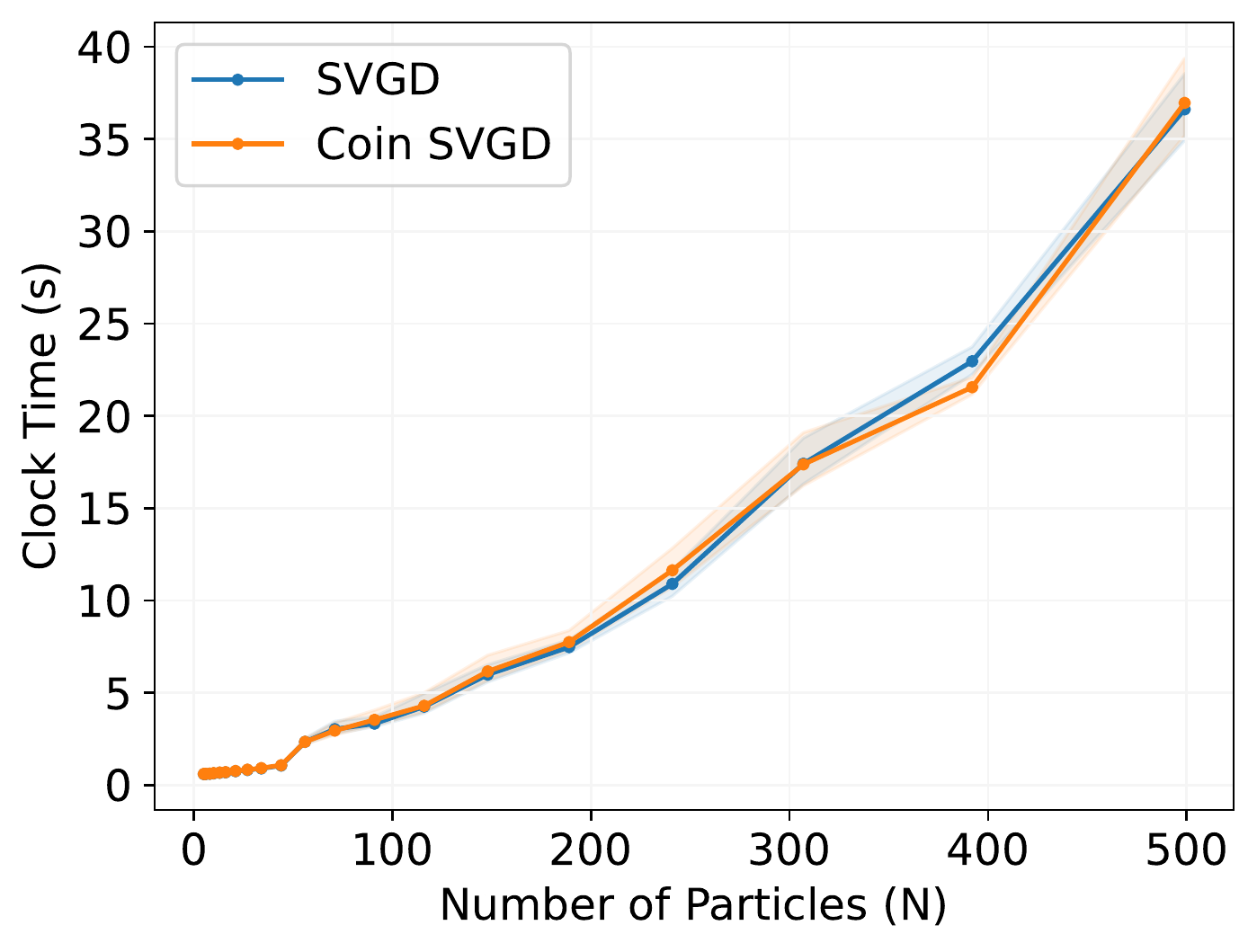}}
\caption{\textbf{Additional results for the Bayesian logistic regression model}. (a) KSD between the target posterior and the posterior approximations generated by Coin SVGD and SVGD after $T=2500$ iterations, as a function of the number of particles $N$. (b) Time (s) to run Coin SVGD and SVGD for $T=2500$ iterations as a function of the number of particles. For both sets of results, we average the results over 20 random train-test splits.}
\label{fig:bayes-lr-add}
\vspace{-4mm}
\end{figure}

\subsection{Bayesian Neural Network}
\label{sec:bnn-extra}

\textbf{Additional Experimental Details}. For the results in Fig. \ref{fig:BayesNN} (Sec. \ref{sec:bnn}) and Fig. \ref{fig:BayesNN-additional} (see below), we run SVGD using a grid of 20 logarithmically spaced learning rates $\gamma\in[1\times 10^{-10},1\times 10^{-0.5}]$. Following \citet{Liu2016a}, for the Protein and Year datasets in Fig. \ref{fig:BayesNN-additional}, we use 100 hidden units (rather than 50). In addition, we use a mini-batch size of 1000 for Year (rather than 100). We report results averaged over 20 random train-test splits for all datasets other than Year, for which we just report a single run. Once again, we use Adagrad \cite{Duchi2011} to adapt the learning rate for SVGD.

\textbf{Additional Numerical Results}. In Fig. \ref{fig:BayesNN-additional}, we plot the average test RMSE achieved by Coin SVGD and SVGD after $T=2000$ iterations, for several additional UCI datasets. As noted in the main text, for certain datasets there remains a considerable gap between the optimal performance of SVGD, and the performance of Coin SVGD (see, e.g., Fig. \ref{fig:BayesNN-wine}). We expect, however, that this performance gap could be significantly reduced by appropriately extending recent advances in parameter-free stochastic optimisation to our setting \citep[e.g.][]{Chen2022,Chen2022a}. 

Finally, in Fig. \ref{fig:BayesNN-time}, we plot the average time (s) for both methods to complete $T=2000$ iterations. As in our previous experiments, there is no meaningful difference between the two methods in terms of clock time. 

\begin{figure*}[t!]
\centering
\subfigure[Protein.]{\includegraphics[trim=0 0mm 15mm 13mm, clip, width=.23\linewidth]{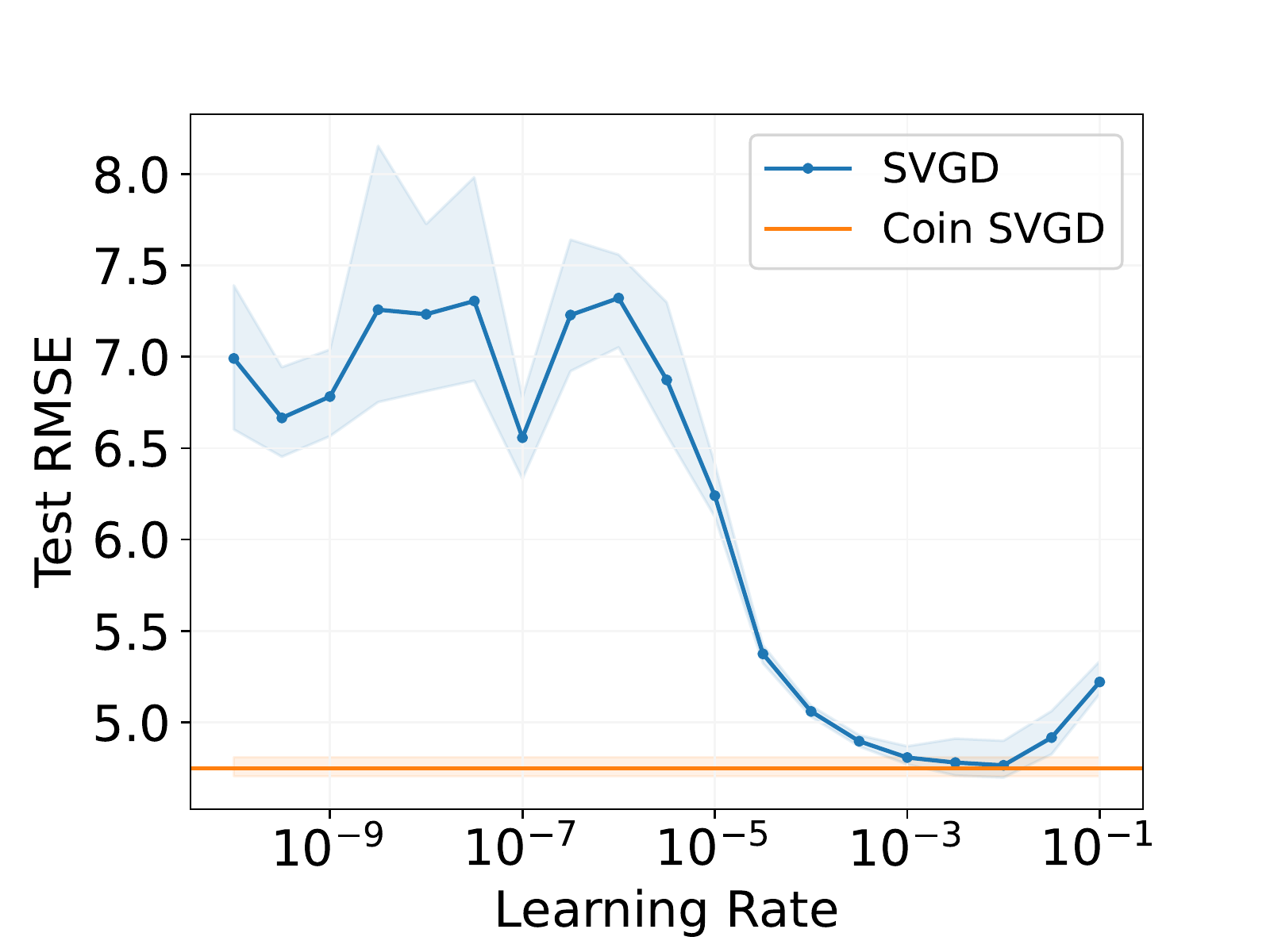}} \hfill
\subfigure[Wine. \label{fig:BayesNN-wine}]{\includegraphics[trim=0 0mm 15mm 13mm, clip, width=.23\linewidth]{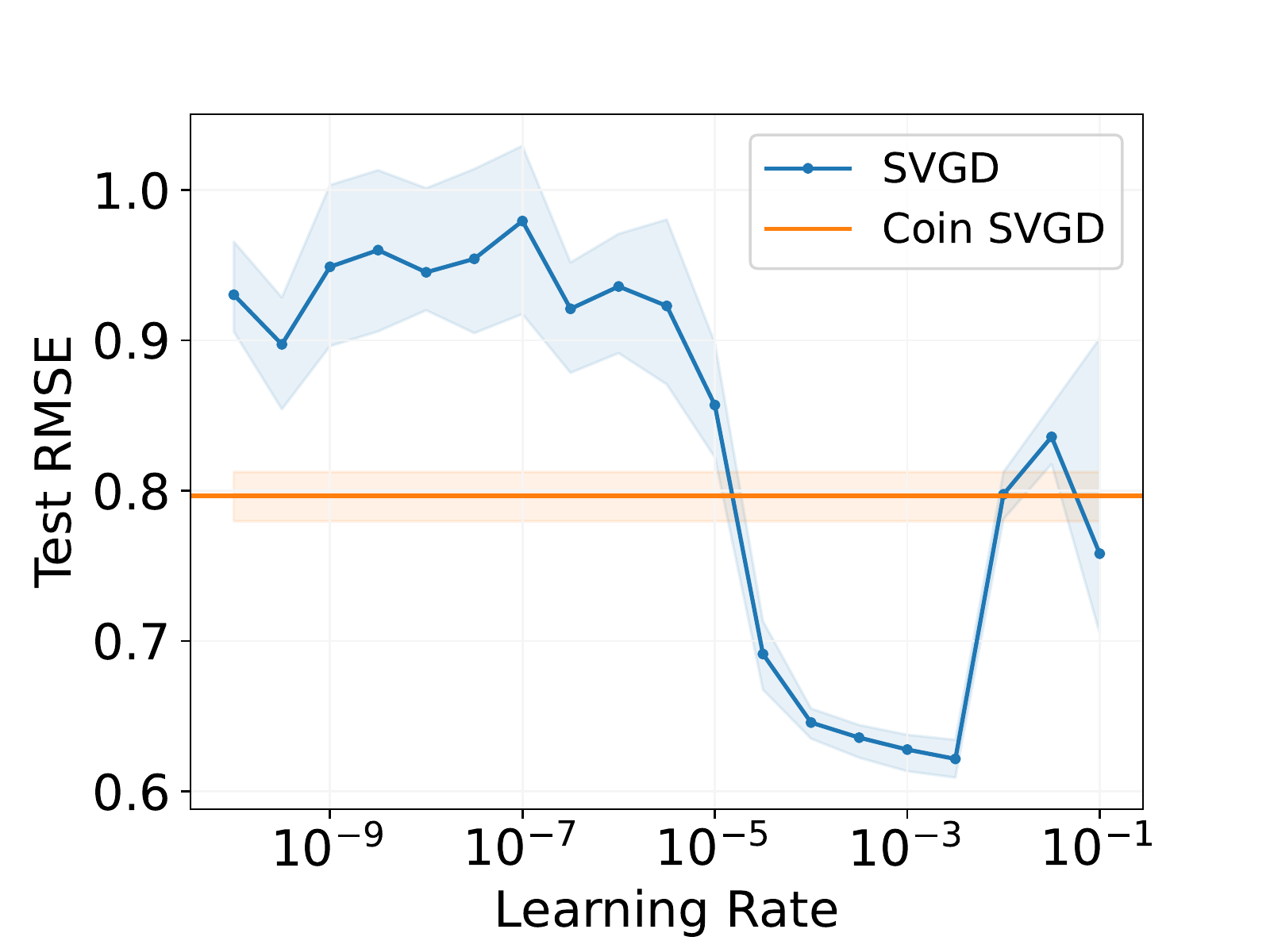}} \hfill
\subfigure[Yacht.]{\includegraphics[trim=0 0mm 15mm 13mm, clip, width=.23\linewidth]{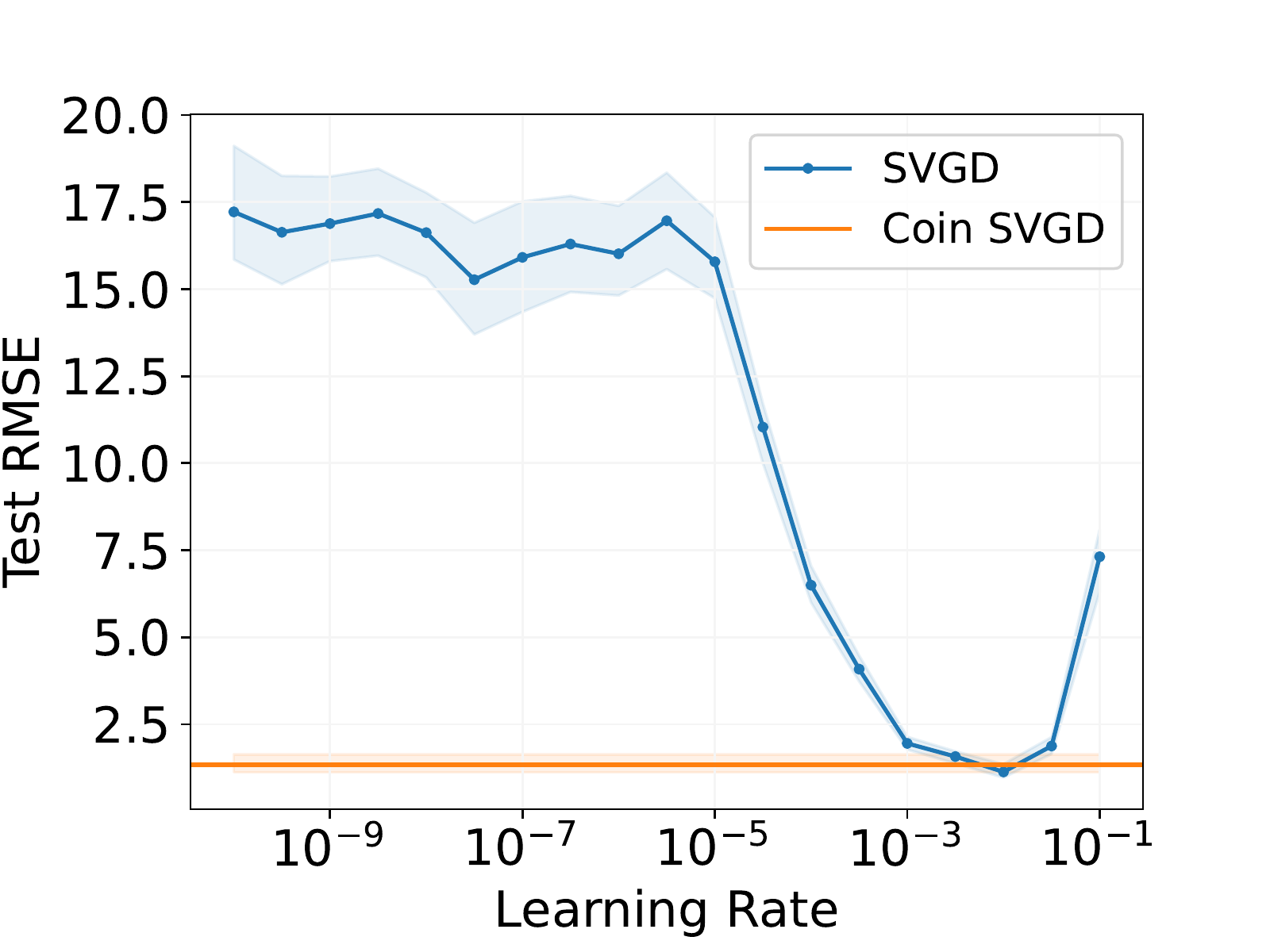}} \hfill
\subfigure[Year.]{\includegraphics[trim=0 0mm 15mm 13mm, clip, width=.23\linewidth]{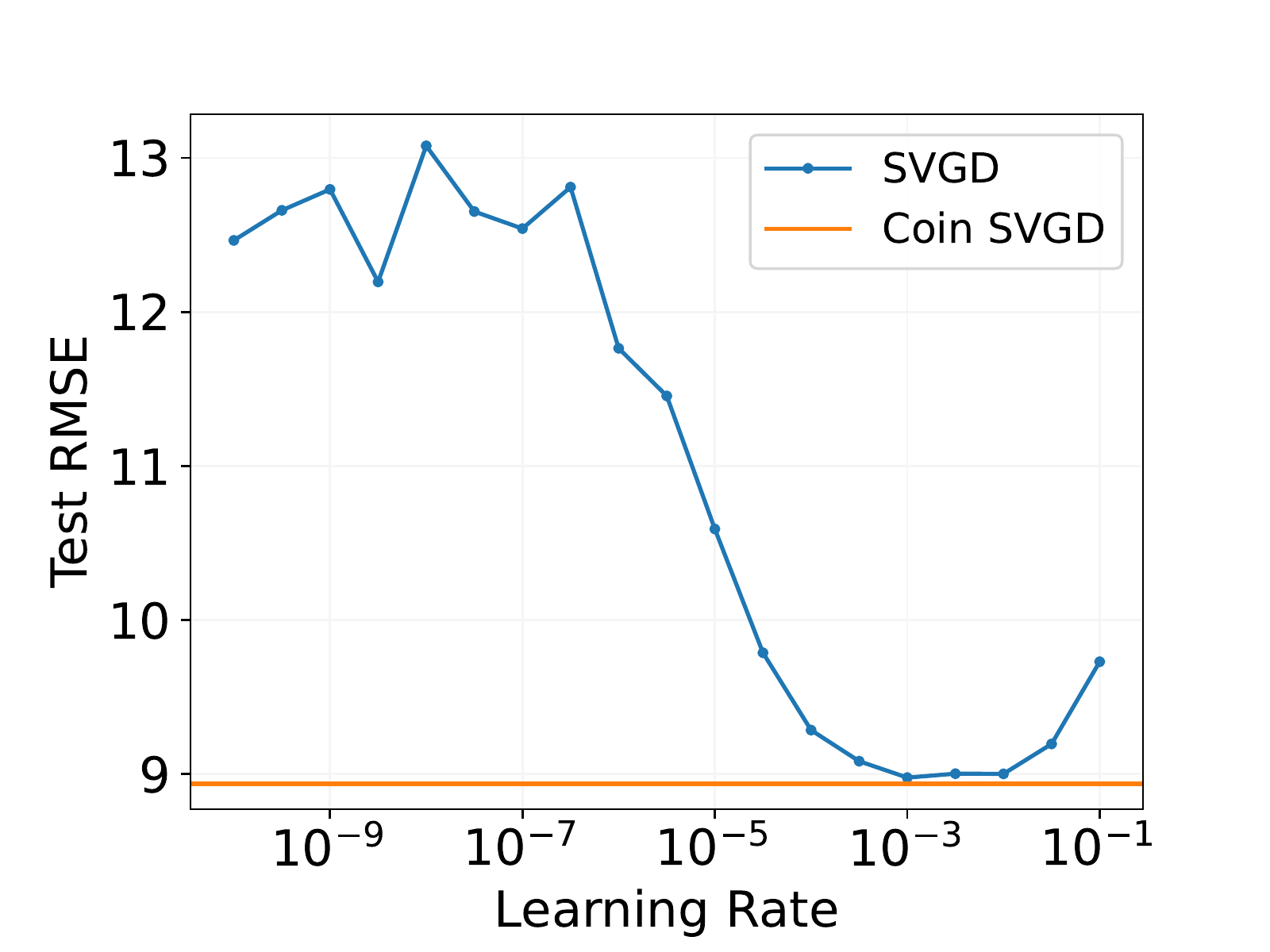}} \hfill
\caption{\textbf{Additional results for the Bayesian neural network}. Average test RMSE for Coin SVGD and SVGD, as a function of the learning rate, after $T=2000$ iterations, for several additional UCI datasets.}
\label{fig:BayesNN-additional}
\vspace{2mm}
\end{figure*}

\begin{figure*}[t!]
\centering
\includegraphics[trim=0 0mm 0mm 0mm, clip, width=.6\linewidth]{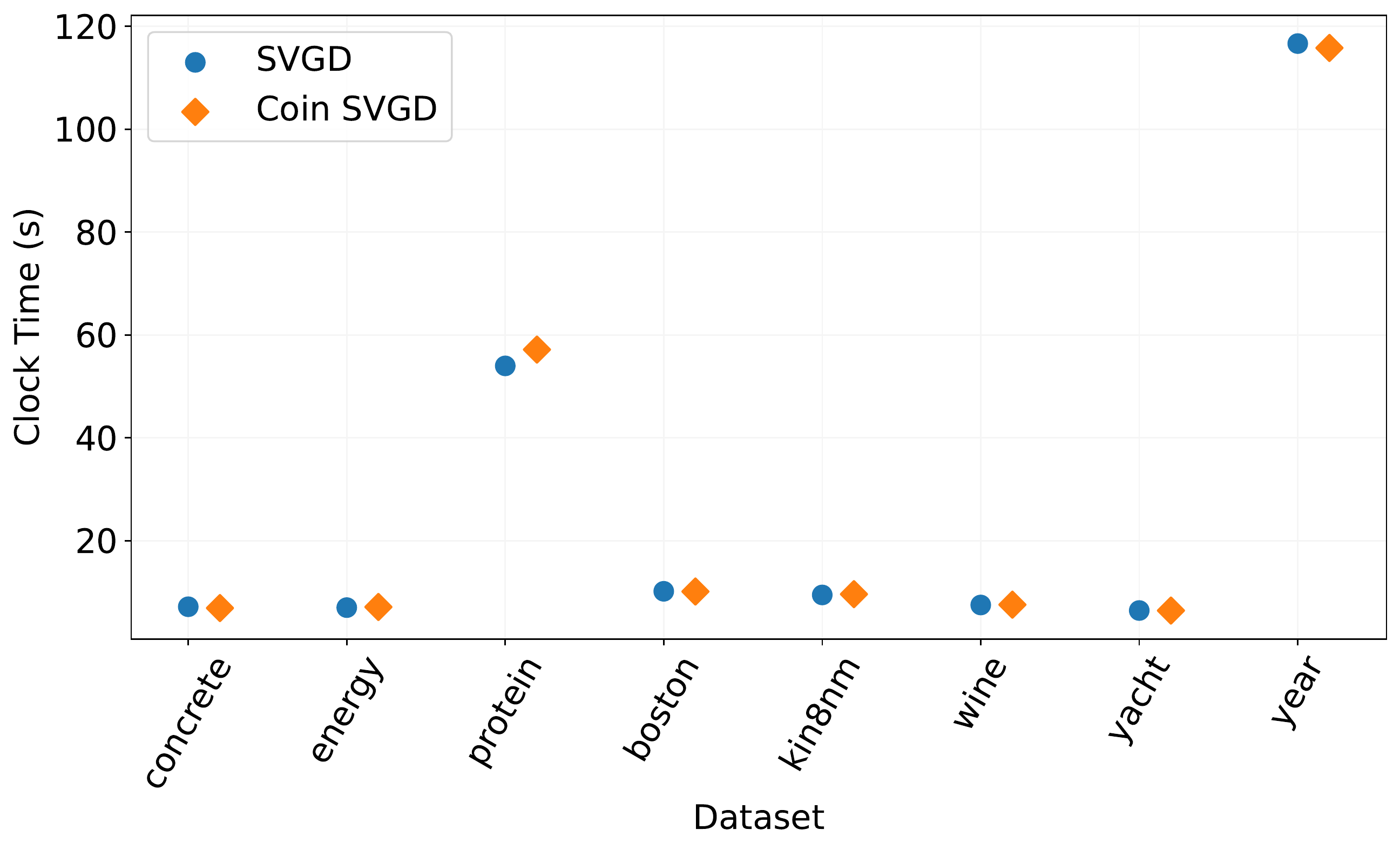}
\caption{\textbf{Additional results for the Bayesian neural network}. Time (s) to run Coin SVGD and SVGD for $T=2000$ iterations, for each of the UCI datasets considered in Fig. \ref{fig:BayesNN} and Fig. \ref{fig:BayesNN-additional}.}
\label{fig:BayesNN-time}
\end{figure*}
\hfill


\end{document}